\documentclass[final]{cvpr}

\usepackage{times}
\usepackage{epsfig}

\usepackage{amsmath,amsthm,amssymb,graphicx}
\newtheorem{thm}{Theorem}

\usepackage[pagebackref=true,breaklinks=true,colorlinks,bookmarks=false]{hyperref}

\def\cvprPaperID{10260} 
\def\confYear{CVPR 2021}

\begin{document}

\title{Adversarial Poisoning Attacks and Defense for General Multi-Class Models Based On Synthetic Reduced Nearest Neighbors}

\author{Pooya Tavallali\\
University of California, Merced\\
{\tt\small ptavallali@ucmerced.edu}
\and
Vahid Behzadan\\
Tagliatela College of Engineering\\
University of New Haven\\
{\tt\small vbehzadan@newhaven.edu}
\and
Peyman Tavallali\\
Independent Researcher\\
{\tt\small tavallali@gmail.com}
\and
Mukesh Singhal\\
University of California, Merced\\
{\tt\small msinghal@ucmerced.edu}
}

\maketitle
\begin{abstract}
State-of-the-art machine learning models are vulnerable to data poisoning attacks whose purpose is to undermine the integrity of the model. However, the current literature on data poisoning attacks is mainly focused on ad hoc techniques that are only applicable to specific machine learning models. Additionally, the existing data poisoning attacks in the literature are limited to either binary classifiers or to gradient-based algorithms. To address these limitations, this paper first proposes a novel model-free label-flipping attack based on the multi-modality of the data, in which the adversary targets the clusters of classes while constrained by a label-flipping budget. The complexity of our proposed attack algorithm is linear in time over the size of the dataset. Also, the proposed attack can increase the error up to two times for the same attack budget. Second, a novel defense technique based on the Synthetic Reduced Nearest Neighbor (SRNN) model is proposed. The defense technique can detect and exclude flipped samples on the fly during the training procedure. Through extensive experimental analysis, we demonstrate that (i) the proposed attack technique can deteriorate the accuracy of several models drastically, and (ii) under the proposed attack, the proposed defense technique significantly outperforms other conventional machine learning models in recovering the accuracy of the targeted model. 

\end{abstract}






\section{Introduction}
\label{sec:intro}
Machine learning models are known to be vulnerable to data poisoning attacks \cite{papernot2018sok} at training-time. In such attacks, the adversary intentionally manipulates the training data by perturbing, adding, or removing training samples with the goal of deteriorating the integrity of the model, thus, resulting in under-performing models with low accuracy \cite{barreno2006can}.

Alternatively, this scenario can be seen as learning under noisy data \cite{manwani2013noise}. Such attacks have the intention of altering the decision boundaries of targeted model, thus, threatening the integrity of the model \cite{kearns1993learning}. The modification to data is done by manipulating samples' information such as features and labels, or by inserting and removing samples. Generally, it is assumed that attacks are constrained by an attack budget to account for realistic consitions of such attacks \cite{papernot2018sok, nelson2006bounding}. Altering the training samples can be seen as modifying the modalities of the data that generated the training samples, thus, deteriorating the consistency of the trainset and the trained model. Much of the literature in this field is focused on the poisoning attacks against specific models and their robustness against ad hoc attacks \cite{papernot2018sok, silver2016mastering}. Authors of \cite{kearns1993learning} theoretically investigated the accuracy of classifiers under an adversarial attack that can modify a specific portion of the trainset.

Label manipulation is a common attack surface for adversaries\cite{papernot2018sok}. The adversarial label manipulation attack tends to perturb the minimum number of labels (constrained by an attack budget), such that the resulting error is maximized. Finding the optimum of such attack is shown to be at least NP-hard (\cite{biggio2011support, mozaffari2014systematic}). A baseline strategy is to perturb the labels at random. The study in \cite{biggio2011support} established that random perturbation of labels can degrade the accuracy of SVM classifiers by flipping about $40\%$ of the labels. However, this conclusion is limited to only binary classification problems in SVM models. Authors of \cite{biggio2011support} further showed that heuristics can improve the success of adversary in degrading the performance of the SVM. \cite{mozaffari2014systematic} proposed a similar approach, in which the attack consists of training a new ML model after poisoning each new sample to measure that sample's effect on the trained model's accuracy. However, this approach is computationally expensive. This is mainly due to the lack of knowledge about the relationship between the training set and test set \cite{papernot2018sok}. In case this relationship is known, then it is feasible to find near optimal samples for label manipulation \cite{xiao2012adversarial}.

While the body of work on label noise is vast, to the extent of our knowledge, no study has yet focused on the label noise that exists in an underlying manifold of the data such as clusters. It also has never been explored how such noise or intentional adversarial attack can be dealt with. This is a very significant type of attack since the adversary can focus on attacking minority groups/clusters hidden inside the dataset, thus, deteriorating the integrity and diversity of the model, while remaining undetected.

Much of the current research is focused on ad hoc data poisoning attacks that are designed for specific machine learning models. Furthermore, the majority of existing attacks in the literature are limited only to binary class problems and gradient-based algorithms. 


Therefore, to tackle the aforementioned issues, this paper proposes the first novel adversarial label-flipping algorithm that is not restricted to a specific machine learning model. The proposed attack technique is the first that is based on the multi-modality of the data. 
Afterwards, we propose a regularized SRNN model that can effectively find up to $80\%$ of malicious samples through novel regularization and pruning techniques. 
The optimization algorithm used here is based on the EM algorithm of \cite{tavallali2020interpretable} and is composed of three steps. The first step is the assignment step and consists of optimizing the regularization term and assigning the optimal label to each centroid. The update step consists of optimizing the centroids. We will show that in the update step by using a proper surrogate objective function, the malicious samples are automatically excluded. 
The contributions of this paper are as follows:
\begin{enumerate}
    \item We propose a novel label poisoning attack mechanism based on SRNN which is not limited to specific machine learning models and classification tasks. 
    \item We propose a defense technique comprised of a novel regularization method and pruning strategy for eliminating maliciously perturbed training samples.
    \item We experimentally demonstrate that the performance of the proposed attack mechanism is superior to similar attack/noisy techniques against a variety of well-known machine learning models and classification tasks.
    \item We experimentally establish the feasibility of our proposed defense technique, and demonstrate that it can detect up to 80\% of the malicious samples in a variety of known resilient machine learning models and classification tasks. 
\end{enumerate}
\section{Related Work}
\subsection{Data Poisoning Attacks}
Many papers have focused on manipulations in the feature space as the mode of the attack. In such settings, the adversary can corrupt both the labels and input features of the trainset. The literature in this area is generally focused on online learning setups and clustering as the model, where the adversarial strategy is to slowly displace the center of the cluster to induce misclassifications. Authors in \cite{kloft2010online} used such an approach in a trainset used for anomaly detection and demonstrated how the approach can gradually shift the decision boundaries of a centroid model. Another similar approach is introduced in \cite{biggio2014poisoning}.

Various other studies in the literature focus on gradient-based methods for selecting samples to poison in the context of SVM models \cite{biggio2012poisoning, mei2015using, xiao2012adversarial, xiao2015feature}. Manipulations of features in samples selected in this manner have been shown to incurr a devastating effect on reinforcement learning agents \cite{behzadan2017vulnerability}.

A data poisoning attack that uses labels as surface of attack can be modeled as noise in labels. Authors in \cite{frenay2014comprehensive} present a comprehensive survey of the label noise. In \cite{frenay2013classification}, inspired by \cite{schafer2002missing}, authors distinguish three types of label noises. First type is label Noise Completely at Random (NCAR). This type of noise happens at random regardless of the features or the true class. The second type is Noise at Random (NAR). This type of noise happen when some classes are more likely to be noisy. However, this type of noise is easy to detect by a human since a specific class is being constantly targeted. The third type of noise is Noise Not at Random (NNAR). In such settings, mislabeling of the samples is related to the features of the samples and the mislabeling can happen at the decision boundaries.

In general, there are three approaches to tackle the noisy label issue. First approach is to use label noise-robust machine learning models. It has been shown that most models are intrinsically robust to noise \cite{manwani2013noise}. However, some of the models are more robust than the others. For example, ensemble models are more robust than other machine learning models \cite{dietterich2000experimental, ratsch1999regularizing, khoshgoftaar2010supervised, ratsch2000robust, ratsch2001soft}. Also, in decision trees, the split criterion can improve robustness of the tree \cite{abellan2003building}. The second approach is to remove noisy samples using methods such as outlier detection \cite{beckman1983outlier, hodge2004survey} or anomaly detection \cite{chandola2009anomaly}. Some studies in the literature use heuristics to remove noisy samples. Examples of such heuristics can be found in reduced nearest neighbor \cite{gates1972reduced}. The work in\cite{gates1972reduced} proposes to remove samples whose removal does not affect misclassification of other samples. AdaBoost-based methods can also be used such as \cite{verbaeten2003ensemble, karmaker2006boosting}. The third approach consists of learning a model that is noise-tolerant such as \cite{gaba1992implications, joseph1995bayesian, swartz2004bayesian}. However, these methods are based on assumptions made for the probability distribution of the noise. Other practical examples of the third approach consist of applying clustering algorithms for detecting mislabled samples through a nearest neighbor approach \cite{younes2010evidential, bouveyron2009robust} or confidence of the model prediction \cite{denoeux2008k, denoeux2000neural}. Further, cross-validation itself can improve the robustness of the model against label noise \cite{hastie2009elements}.

As explored in this section, the state of the art in poisoning attacks is mostly comprised of attacks that focus either on specific machine learning models or are at most gradient based approaches. Therefore, they might not be applicable to general machine learning models, such as decision trees and forest models.

\subsection{Nearest Neighbor Methods}
One of the well-known conventional and the oldest models for classification is the Nearest Neighbor and Reduced Nearest Neighbor \cite{cover1967nearest, gates1972reduced}. Main reasons for their popularity is their simplicity and accuracy when combined with various distance metrics \cite{tran2008human, goldberger2005neighbourhood}. However, a major drawback of Nearest Neighbor models is their high inference complexity. Several studies have approached this issue by breaking down the inference complexity using data structures such as trees or cover/ball trees \cite{de2008orthogonal, beygelzimer2006cover, mathy2015boundary}. Such approaches yield high speed ups but they are vulnerable to curse of dimensionality, and suffer from high storage overhead and lack of interpretability.

Studies reported in \cite{cover1967nearest, gates1972reduced, angiulli2005fast} aim to solve the mentioned issues by reducing the trainset size. Also, alternative approaches include learning prototypes/centroids (synthetic samples) as the nearest neighbor model \cite{bermejo1999adaptive, frosst2019analyzing, kusner2014stochastic, gupta2017protonn, tavallali2020interpretable}. These models are interpretable in the sense that each prototype and its class label represent their corresponding data cluster and label, respectively. Therefore, each centroid represents a modality of the data. Similar to ensemble and decision trees, these models tend to partition the space into various disjoint regions.
\section{Proposed Method}
\subsection{Preliminaries}
\textbf{Synthetic Reduced Nearest Neighbor(SRNN):} A synthetic reduced nearest neighbor (also known as prototype nearest neighbor(PNN)) is a set of synthetic samples (centroids) that infer an input similar to a nearest neighbor model \cite{tavallali2020interpretable}. Assume a dataset of samples consisting of $N$ tuples of observation, $\{(x_i,y_i)\}_{i=1}^N$. $x_i \in \mathbb{R}^D$ is the $i^{th}$ sample's features and its target response is $y_i\in \{1,2,3,...,M\}$. The SRNN model consists of a set of centroids/prototypes $C=\{(c_j,\hat{y_j})\}_{j=1}^K$. At the test time, prediction of an input is the label of closest centroid to that input. The problem of learning a SRNN model is similar to that of a k-means problem except that each centroid is associated with a label that represents the prediction of the centroid. Mathematically speaking, the optimization of the SRNN model using 0-1 loss is:
\begin{equation}
	\label{eq:SRNN_obj}
	\begin{split}
	 \underset{\{(c_j,\hat{y_j})\}_1^K}{\text{min}}  \quad  &\sum_{i=1}^{N} L(y_i,NN(x_i))\\
	 \text{s.t.}  \quad & NN(x_i)=
	\hat{y_{j^*_i}} \\ & j^*_i= \underset{\{j\}_1^K}{\text{argmin}} \quad d(x_i-c_j)
	\end{split}
\end{equation}
where, $NN(.)$ represents the nearest neighbor function. $d(.)$ is a distance metric (chosen to be the Euclidean distance in this study). For simplicity, in the rest of this paper, we use $r_{ij}$ for $d(x_i-c_j)$. $j^*_i$ represents the index of closest centroid to $i^{th}$ sample. $L$ is a 0-1 loss function that outputs $0$ if both of its input arguments are equal, otherwise, it outputs $1$. The problem of \eqref{eq:SRNN_obj} is in fact the problem of finding a set of $K$ synthetic samples that achieve minimum error as a nearest neighbor model with $K$ samples. Intuitively, each centroid represents a modality of the data that the samples are generated from.
\subsection{Modality-based adversarial label flipping}

One common objective of adversarial data poisoning is to undermine the integrity of the trained model. This can be cast as weakening the performance of the trained model at the test time \cite{papernot2018sok}. In other words, the goal of the attacker is to increase the error of the trained model. In this paper, based on SRNN model, we propose the modality-based (or cluster-based) perturbation of the training labels. The problem of selecting optimal samples for proposed attack technique is as follows:
\begin{equation}
	\label{eq:Attack_prob}
	\begin{split}
	 \underset{\{I_i,y_i^p\}_1^N}{\text{max}}  \quad & \sum_{(x_i,y_i) \in S^{train}} L(y_i,NN^*(x_i))\\
	 s.t \quad & NN^*=\underset{\{(c_j,\hat{y_j})\}_1^K}{\text{argmin}} \sum_{x_i \in S^p} L(y_i,NN(x_i))\\
	 & S^p=\{(x_i,(y_i(I_i-1)+y_i^p(I_i))) \}\\
	 & \sum_{i=1}^N I_i \leq Cost
	\end{split}
\end{equation}
Where, $NN^*(.)$ represents the optimal model trained over the poisoned dataset $S^p$. The goal is to increase the error over the trainset with true labels $S^{train}$. $y_i^p$ represents the perturbed label and $I_i$ is an indicator variable that is either $0$ or $1$. $I_i$ is used for representing selected samples. The second constraint represents the poisoned dataset with perturbed labels. The third constraint shows the maximum allowed number of perturbations, given an attack budget ($Cost$).

The problem in \eqref{eq:Attack_prob} is a selection problem. The problem consists of selecting samples and changing their labels to another class such that the error of $NN^*$ is maximized over $S^{train}$. Intuitively, the trained model should perform poorly on the trainset with true labels in the hope that it performs poorly at the test time. Finding optimal solution of \eqref{eq:Attack_prob} is NP-hard \cite{biggio2011support, mozaffari2014systematic}. As a consequence, finding global optimum of \eqref{eq:Attack_prob} is computationally intractable and not practical. Therefore, we propose an efficient greedy algorithm that approximates the solution for problem \eqref{eq:Attack_prob} and also satisfies the constraints. Initially, there are no samples selected for poisoning. Therefore, the adversary has to train $NN^*$ over the trainset. Next step consists of selecting samples. This is done by fixing centroids of $NN^*$ while continuing optimization only over $\{I_i,y_i^p,\hat{y_j} \}$. At this step, the assignment of the train set samples are fixed and cannot be changed since the centroids are fixed. Therefore, the problem consists of changing labels of samples in $S^p$ such that the prediction labels of some of the centroids in $NN^*$ are changed, thus, increasing the error over $S^{train}$. This problem can be stated as follows:
\begin{equation}
	\label{eq:Attack_prob_clust}
	\begin{split}
	 \underset{\{I_i,y_i^p\}_1^N,\{y_j\}_{j=1}^{K}}{\text{max}}  \quad & \sum_{j=1}^{K}\sum_{(x_i,y_i) \in S^{train}_j} L(y_i,\hat{y_j})\\
	 s.t \quad & \hat{y_j}=mode(\{y_i\}_{ S^p_j}) \forall j=1...K\\
	 & S^p=\{(x_i,(y_i(I_i-1)+y_i^p(I_i))) \}\\
	 & \sum_{i=1}^N I_i \leq Cost
	\end{split}
\end{equation}
In \eqref{eq:Attack_prob_clust}, $S_j^p$ and $S_j^{train}$ represent the trainset samples assigned to $j^{th}$ centroid with perturbed and true labels, respectively. Note that the difference between \eqref{eq:Attack_prob_clust} and \eqref{eq:Attack_prob} is that the centroids of clusters are fixed, thus, assignments are fixed. $\hat{y_j}$ represents the $j^{th}$ centroid's optimal label. Therefore, in order to increase the objective function, $\hat{y_j}$ has to be changed through selecting samples for poisoning. $\hat{y_j}$ can only get changed if the majority label in $S_j^p$ gets changed. Many heuristics can be applied here such as changing some of the labels to the second most frequent label in the $S_j^p$. One intuitive approach is to randomly change the labels of half plus one of the samples in $S_j^{train}$ to the minority (i.e, least frequent) labels. This causes the information of the cluster (modality of the data) to become obscure and misleading. In other words, the attacker has to spend a specific cost to turn the $\hat{y_j}$ into a false label (minority label of the cluster). However, the attack has a limited budget and has to select clusters based on its budget. From a practical point of view, it can be seen as recognizing vulnerable groups in a data. All that remains is selecting clusters to attack. Assuming the cost of changing $j^{th}$ label is $Cost_j$, then the problem can be simplified as follows:
\begin{equation}
	\label{eq:Attack_prob_clust_sel}
	\begin{split}
	 \underset{\{I_j\}_{j=1}^K}{\text{max}}  \quad & \sum_{j=1}^K Cost_j I_j\\
	 s.t & \sum_{j=1}^K Cost_j I_j < Cost \\
	 & I_j=\{0,1\} \quad \forall j=1...K
	\end{split}
\end{equation}
Problem \eqref{eq:Attack_prob_clust_sel} can be solved greedily by selecting from clusters with lower cost until the first constraint is violated. Other algorithms such as dynamic programming or other greedy approaches can also solve the problem in \eqref{eq:Attack_prob_clust_sel}.

The proposed attack technique can have the capacity to increase the error over the trainset by two times the attack budget. Further details and proof are provided in the supplementary materials.

\textbf{Computational complexity:} The proposed attack technique first trains a SRNN that takes ${\cal O}(NDK)$ \cite{tavallali2020interpretable}. Then the attack technique perturbs labels. The perturbation step takes ${\cal O}(N)$ which is embarrassingly fast. Compared to other label flipping attacks \cite{mei2015using}, this attack is computationally very cheap and feasible.

Intuitively, the proposed attack targets the minority groups (smaller clusters). In our experiments, the proposed attack technique achieved significantly higher test error for various machine learning models compared to other label-flipping attacks and no attack. In other words, for an adversary, it would be more efficient in terms of budget to target vulnerable smaller clusters of a dataset and undermine the integrity of any machine learning model significantly. 

\subsection{Defense via Regularized Synthetic Reduced Nearest Neighbor (RSRNN)}

We introduce a new parameter named \emph{confidence range} ${r_j}_{j=1}^K$, as well as two new regularization terms for SRNN as the defense technique. Any sample beyond the confidence range, $r_{ij^*}>r_{j_i^*}$, is considered to be malicious. Note that $j_i^*$ represents the index of closest centroid to sample $i$. First term consists of regularizing the confidence range for each centroid. Second regularization term consists of adding cost complexity function over the SRNN structure. The cost function facilitates the pruning of centroids and further recognizes the attacked modalities of the data. The optimization problem of training RSRNN is given in \eqref{eq:RSRNN_obj}:
\begin{equation}
	\label{eq:RSRNN_obj}
	\begin{split}
	& \underset{\{(c_j,\hat{y_j},r_j)\}_1^K}{\text{min}}  \quad \sum_{i=1}^{N} L(y_i,NN(x_i)) + \lambda \sum_{j=1}^K r_j +\alpha \sum_{j=1}^K cost(S_j))\\
	& \text{s.t.} \quad NN(x_i)=\begin{cases}
	\hat{y_{j^*_i}} \quad r_{ij^*}<r_{j^*} \\ Malicious \quad \text{otherwise}
	\end{cases}
	\end{split}
\end{equation}

where, $\lambda$ is the penalty coefficient of $r_j$, $\alpha$ is the cost complexity coefficient, and $cost(.)$ represents the cost function of $j^{th}$ centroid. $S_j$ consists of the samples whose closest centroid from $C$ is $j^{th}$ centroid ($S_j=\{x_i | j=j^*_i \quad \forall i=1,2,...,N\}$). To solve the optimization problem in \eqref{eq:RSRNN_obj}, we follow the same EM algorithm as in \cite{tavallali2020interpretable} which was inspired by K-means algorithm \cite{lloyd1982least}. The optimization approach consists of three steps: the assignment step, the update/centroid step, and the pruning step based on a validation set. 

\subsubsection{Assignment Step:} This step has two parts. First part consists of assigning the train samples to their closest centroid. This is essentially calculating $S_j$ for $j=1...K$. Second part is finding optimal values to $\{\hat{y}_j\}_{j=1}^K$. The problem for $j^{th}$ centroid can be written as
\begin{equation}
	\label{eq:RSRNN_assign}
	\begin{split}
	\underset{\hat{y_j}}{\text{min}}  \quad \sum_{x_i \in S_j} L(y_i,\hat{y_j}) U(r_j-r_{ij})
	\end{split}
\end{equation}
Where $U(.)$ is a step function and is used to impose the constraint in \eqref{eq:RSRNN_obj} for $Malicious$ samples. The problem in \eqref{eq:RSRNN_assign} is that of finding the best constant predictor over the set of $S_j$. Its optimum is the most frequent label of samples in $S_j$ ($\hat{y}_j^{*}=\text{mode}(\{y_i|\forall x_i \in S_j\})$). 

\subsubsection{Update step:} This step consists of optimizing each centroid while the centroid labels are kept constant. In this step, first, $\{r_j\}_{j=1}^{K}$ are fixed and $\{c_j\}_{j=1}^{K}$ are optimized, and then $\{c_j\}_{j=1}^{K}$ are fixed and ${r_j}_{j=1}^{K}$ are optimized. It is shown that the problem for optimizing each centroid is a binary classification task. Further, the centroid problem is NP-hard, hence, it will be approximated using a novel surrogate loss function. The optimization problem of this step for $j^{th}$ centroid over $c_j$ is as follows:
\begin{equation}
	\label{eq:RSRNN_cent}
	\begin{split}
	\underset{c_j}{\text{min}}  \quad  \sum_{x_i \in S_j} & (L(y_i,\hat{y}_j) U(r_j-r_{ij})+U(r_{ij}-r_j)) +\\ & \sum_{x_i \in S_j^c}L(y_i,NN_{C^\prime}(x_i))
	\end{split}
\end{equation}
where $S_j^c$ represents the complement set of $S_j$ (rest of samples that are not assigned to $S_j$). $NN_{C^\prime}$ represents nearest neighbor function over the set of centroids without $j^{th}$ centroid ($C^\prime=C-(c_j,\hat{y}_j)$). Here, the assignment of samples are not fixed and a sample might eventually get assigned to $S_j$ or $S_j^c$ depending on the position of $c_j$ in the feature space. This is in fact a binary classification problem because each sample has to get assigned to $S_j$ or $S_j^c$. However, prior to the optimization, the optimal assignment of each sample that contributes to decreasing the objective function in \eqref{eq:RSRNN_cent} is not known. From the perspective of EM algorithm, this is a latent variable because it is not clear whether the sample is generated by the distribution of $j^{th}$ centroid or the rest of centroids. The optimal assignment of each sample can be extracted by evaluating the correctness of prediction if the sample assigned to $S_j$ or $S_j^c$. For example a sample might get classified correctly only if the sample is assigned to $S_j$ because the label of the sample matches the label of $c_j$. On the other hand, the same sample might be classified incorrectly, or categorized as malicious if assigned to $S_j^c$. Therefore, this sample has to be assigned to $S_j$ during the optimization to contribute to decreasing the objective function of \eqref{eq:RSRNN_cent} (potentially, $c_j$ has to be closer to the sample than the rest of centroids). From the EM algorithm point of view, the calculation of outcome of assigning a sample to $S_j$ or $S_j^c$ is in fact calculating the posterior probability of the sample being generated by the distribution of $P(y_i|x_i \in S_j)$ or $P(y_i| x_i \in S_j^c)$. In general $8$ scenarios can happen for each sample and the optimal assignment of a sample can be found based on these scenarios. 

All $8$ scenarios are shown in table \ref{tab:assign}. The scenarios are based on three factors. The factors are correctness of classification if the sample is assigned to $S_j$, correctness of classification if the sample is assigned to $S_j^c$, and sample categorized as malicious if assigned to $S_j^c$. Therefore, in total there are $8$ scenarios. Note that the case of a sample being recognized as malicious by $S_j$ is not considered here as a factor. This is because during the optimization, the $c_j$ can move and a sample might fall in the $r_j$ ball or may remain outside of it. However, this factor is incorporated inside the surrogate objective function in the following part of this section. In table
\ref{tab:assign}, $j^{\prime *}_i$ represents the index of closest centroid to sample $i$ from the set $C^{\prime}$. From the table \ref{tab:assign}, the samples that belong to scenarios 1-3 cannot affect the objective function of \eqref{eq:RSRNN_cent}, since the samples will be classified incorrectly or are considered as malicious regardless of any set they are assigned to. The samples that belong to scenario 4 are the samples that have to be assigned to $S_j^c$ because if they are assigned to $S_j$ then, they are classified incorrectly, hence, they increase the objective function of \eqref{eq:RSRNN_cent}. We call such set as $S_j^{c*}$. Samples of scenarios 5-7 have to be assigned to $S_j$ to decrease the objective function of \eqref{eq:RSRNN_cent}. We denote such set of samples as $S_j^*$. Samples of scenario 8 are always classified correctly; thus, they are ineffective in the objective function of \eqref{eq:RSRNN_cent}. Intuitively, $c_j$ has to be replaced in the space such that $c_j$ is closest centroid to samples of $S_j^*$. Also, preferably, $c_j$ should be at a distance of $r_j$ from all samples of $S_j^*$. At the same time $c_j$ should remain further away from samples of $S_j^{c*}$ with at least a distance of $r_{ij_i^{\prime *}}$. 
\begin{table}
\caption{Update step scenarios. All scenarios for assigning a sample to $S_j$ or $S_j^{c}$.}
\begin{center}
\begin{small}
\begin{sc}
\begin{tabular}{@{}ccccccc@{}}   
  \# & $L(y_i,\hat{y}_j)$ & $L(y_i,\hat{y_{j^{\prime *}_i}})$ & $NN_{C^\prime}==Mal.$ & assign\\
    \hline
1 & $1$ & $1$ & $1$ & x\\
\hline
2 & $1$ & $1$ & $0$ & x\\
\hline
3 & $1$ & $0$ & $1$ & x\\
\hline
4 & $1$ & $0$ & $0$ & $S_j^{c*}$\\
\hline
5 & $0$ & $1$ & $1$ & $S_j^*$\\
\hline
6 & $0$ & $1$ & $0$ & $S_j^*$\\
\hline
7 & $0$ & $0$ & $1$ & $S_j^*$\\
\hline
8 & $0$ & $0$ & $0$ & x\\
\hline
\end{tabular}
\end{sc}
\end{small}
\end{center}
\vskip -0.15in
\label{tab:assign}
\end{table}
Using sets of $S_j^*$ and $S_j^{c*}$ and the given intuition in previous paragraph, the problem in \eqref{eq:RSRNN_cent} can be rewritten as follows
\begin{equation}
	\label{eq:RSRNN_cent1}
	\begin{split}
	\underset{c_j}{\text{min}}  \quad  \sum_{x_i \in S_j^*} & (U(r_{ij}-r_{ij_i^{\prime *}}) \lor U(r_{ij}-r_j)) +\\ & \sum_{x_i \in S_j^{c*}} U(r_{ij_i^{\prime *}}-r_{ij})
	\end{split}
\end{equation}
where $\lor$ is logical $OR$ operator. Finding the global optimum of \eqref{eq:RSRNN_cent1} is NP-hard and cannot be solved directly \cite{biggio2011support, mozaffari2014systematic}. Therefore, a surrogate objective function will be used to approximate the solution to \eqref{eq:RSRNN_cent1}. Intuitively, the interest of problem \eqref{eq:RSRNN_cent1} is to replace the $c_j$ close to the samples of $S_j^*$. At the same, time we are interested in nullifying the effect of possible $Malicious$ samples of $S_j^*$ that are outside of the $r_j$ automatically. Further, $c_j$ has to stay outside the ball of $r_{ij_i^{\prime *}}$ for the samples in $S_j^{c*}$. Based on the given intuition, the solution to problem \eqref{eq:RSRNN_cent1} is approximated by solving
\begin{equation}
	\label{eq:RSRNN_surrogate}
	\begin{split}
	 c_j^*(\mu)=\underset{c_j}{\text{argmin}} \quad & \sum_{x_i \in S_j^*}min(r_{ij},r_j)+\\ &\sum_{x_i \in S_j^{c*}}relu(\mu r_{ij_i^{\prime *}}-r_{ij}) 
	\end{split}
\end{equation}
where $min(.)$ returns the minimum of its input arguments. $relu(.)$ is a rectified linear unit. $\mu$ is a coefficient and acts like a slack variable \cite{tavallali2020interpretable} that is increased from $0$ to $1$. For every value of $\mu$ along the path, the problem of \eqref{eq:RSRNN_surrogate} can be solved using stochastic gradient descent. Along this path, the $c_j^*$ that returns the smallest loss for \eqref{eq:RSRNN_cent1} is selected. This is similar to solving an SVM for a linear binary classifier \cite{tavallali2020interpretable} while changing the slack variable. 

Finally, $\{r_j\}_{j=1}^K$ have to be optimized while fixing the rest of parameters. The problem of optimizing $r_j$ for a centroid is
\begin{equation}
	\label{eq:RSRNN_r_j}
	\begin{split}
	 \underset{r_j}{\text{min}}  \quad \sum_{x_i \in S_j} L(y_i,\hat{y}_j)U(r_j-r_{ij})+U(r_{ij}-r_j) + \lambda r_j\\
	\end{split}
\end{equation}
From problem \eqref{eq:RSRNN_r_j}, it can be observed that objective function is piecewise-constant over $r_j$ and objective function has a jump at every $r_{ij}$. This problem can be solved efficiently in ${\cal O}(|S_j|log(|S_j|))$. This is done by sorting ${r_{ij}}$ for every $x_i \in S_j$ and evaluating the objective function of \eqref{eq:RSRNN_r_j} for every $r_j=r_{ij}$ through an incremental algorithm. In total, finding optimum $r_j$ for all centroids is ${\cal O}(Nlog(N))$.

\textbf{Excluding malicious samples on the fly:} Note that the objective function in \eqref{eq:RSRNN_surrogate} has two terms that contain $c_j$. The first term is $min(r_{ij},r_j)$ that encourages samples of set $S_j^*$ to be close to $c_j$ but any sample that falls outside of the ball of $r_j$ will become ineffective in optimization of \eqref{eq:RSRNN_surrogate} because the gradient with respect to that sample becomes $0$. This is interesting since the optimization algorithm is in fact excluding samples that are suspicious to be malicious on-the-fly. Further, other malicious samples such as scenarios 1 and 3 are automatically not considered in the optimization.

In total, computational complexity of this step is ${\cal O} (NDK+NlogN)$ \cite{tavallali2020interpretable}.

Finally, by iterating over update step and assignment step, the first two terms of the objective function in \eqref{eq:RSRNN_obj} decrease over the trainset until no further improvement can happen over the parameters (proof in the supplementary material). It is noteworthy that the surrogate objective function for update step also decreases the $\{cost(S_j)\}_{j=1}^{K}$ since it tends to create pure sets for each $S_j$. 

\subsubsection{Pruning step:}
After optimizing the first two terms of \eqref{eq:RSRNN_obj}, the third term needs further attention. Similar to pruning for decision trees, the cost complexity function encourages to remove (prune) the impure centroids based on their assignment set $S_j$ and the coefficient of $\alpha$. In this paper, Gini-index is used as the cost function. However, the dataset's integrity can be under question, meaning that some of the samples/modes are malicious. Here, the intuition and aim of this step is not only to prune the centroids but also remove the malicious modalities of the trainset. Also note that any malicious sample is considered as a loss, thus, only centroids and modes are considered as malicious if removal of them decreases the error over the validation set. Therefore, for this step, all other parameters including assignment of samples are kept fixed. Further, a clean validation set is used to prune centroids and samples of malicious modes. The clean validation set is a set whose integrity is assured and in practice a tiny set for validation set such as $5-10$ percent of the whole trainset would be sufficient. Further details of pruning is presented in the supplementary material.

Finally, after removing the malicious samples and centroids, it is possible to either restart training using original SRNN over the cleaned data, or continue training with the remaining centroids and select the final model based on the error over the validation set.

\section{Experimental Results}
In this section, experimental results of the proposed attack and defense techniques are presented and compared with similar techniques. The attack techniques which are applicable to various machine learning models are used for comparison in this study. In terms of defense technique, models that are known to be resilient against the attack techniques are used for comparison. Further, the merits of the proposed attack and defense techniques are studied and presented. In the next subsection, the effect of proposed attack is explored and in the second subsection, the performance of the RSRNN is compared with several other resilient models. 

\subsection{Attack Experiments}
In this subsection, the performance of the modality-based adversarial label filliping (SRNN-att) is presented. The SRNN-att is compared with several other label filliping attacks that are generic and can affect all machine learning models. These label filliping techniques are gathered from the literature of label noise because they are found to be the most similar techniques that exist on the topic of this paper, as they are not designed for a specific machine learning model and can potentially deteriorate the performance of any machine learning model. 

Accordingly, in our experiments, SRNN-att is compared with NCAR, NNAR and no attack. NCAR works by changing the labels at random. NNAR is a practical approach that aims at changing labels at the margins of decision boundaries. Here, to apply this type of attack, a state of the art model is trained and then labels of the samples that are predicted with low confidence by the model are targeted. Labels of such samples are changed to the second most probable label, thus, intuitively, changing the decision boundaries. In the experiments, NNAR and NNAR-ADA represent the NNAR attack using a forest model and AdaBoost model, respectively. The size of these models is selected such that they have the same size as that of the SRNN model used for SRNN-att.

\begin{figure}
\centering
    \begin{tabular}{c@{}c@{}c@{}}
    test error for $5\%$ attack & test error for $10\%$ attack
    \\
    \includegraphics*[width=0.5\linewidth]{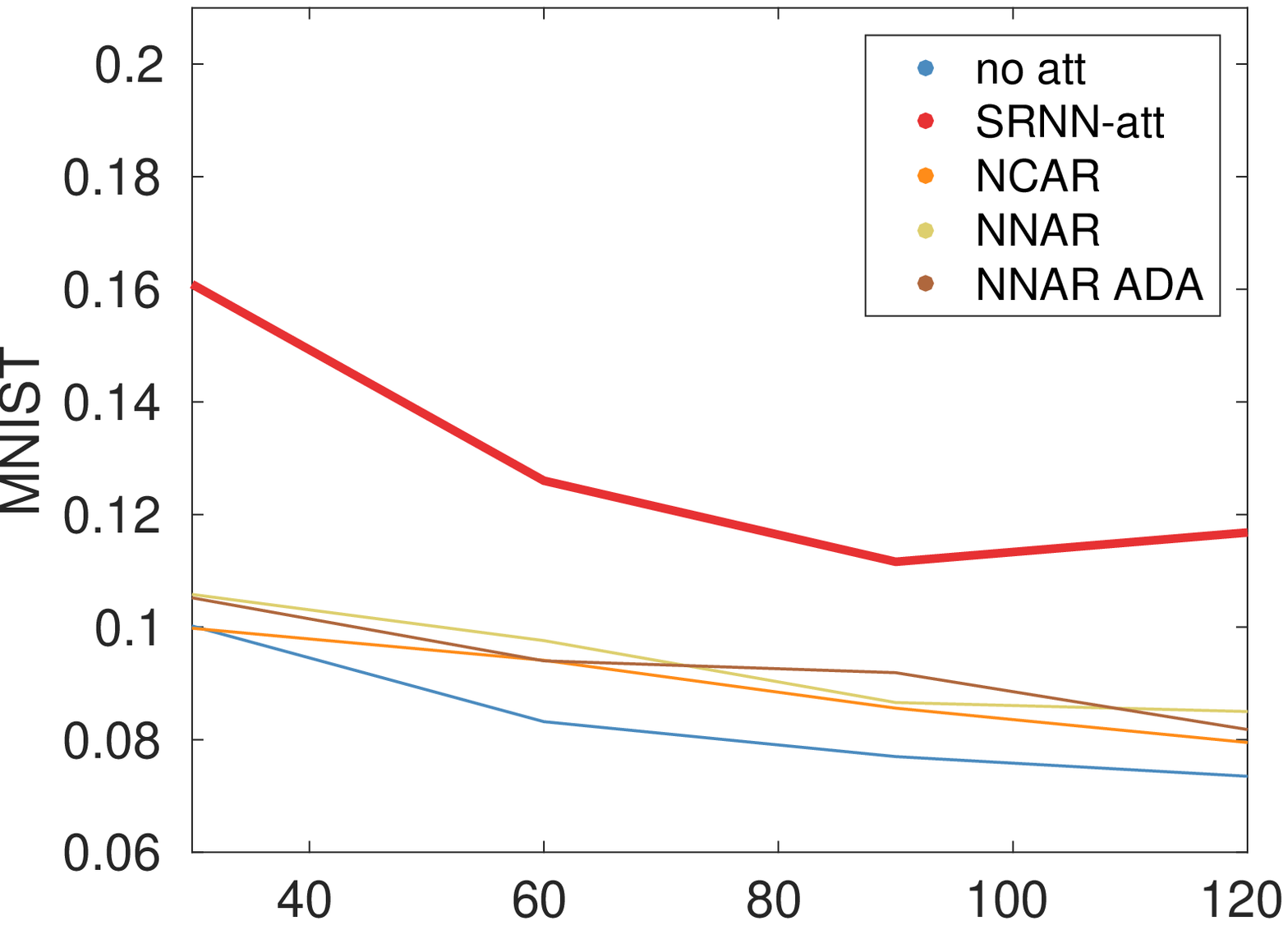}&
    \includegraphics*[width=0.5\linewidth]{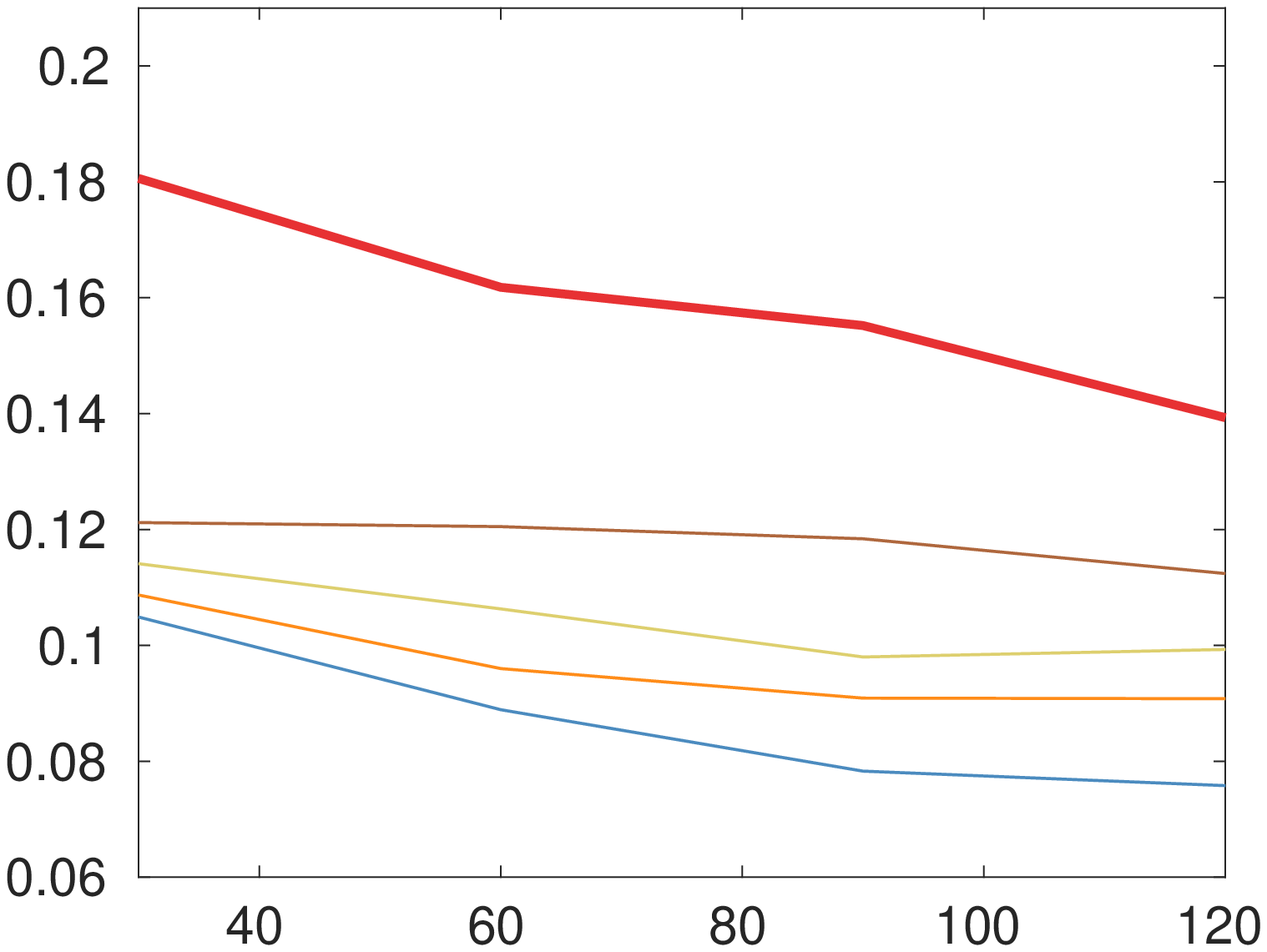}
    \\
    \includegraphics*[width=0.5\linewidth]{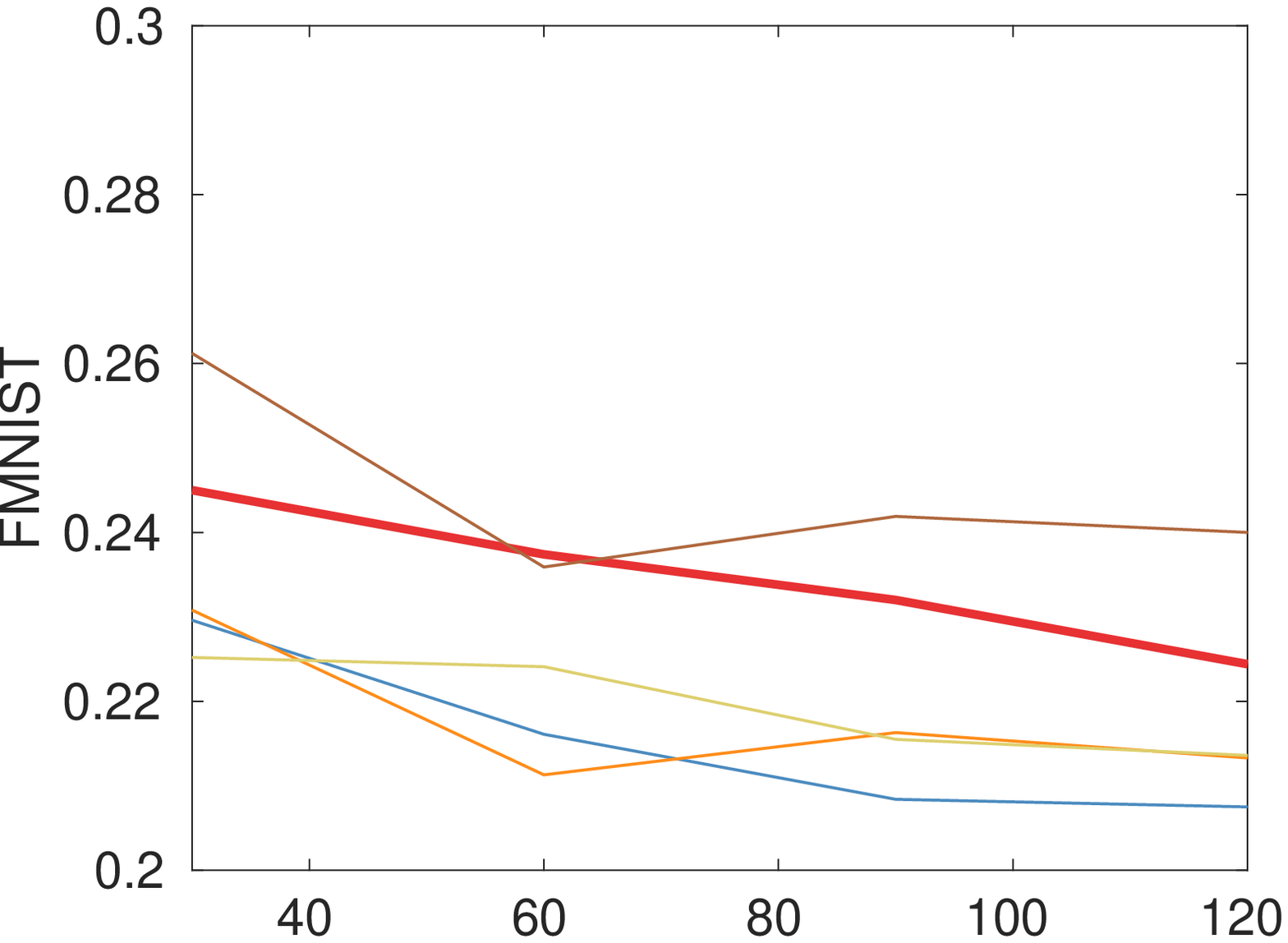}&
    \includegraphics*[width=0.5\linewidth]{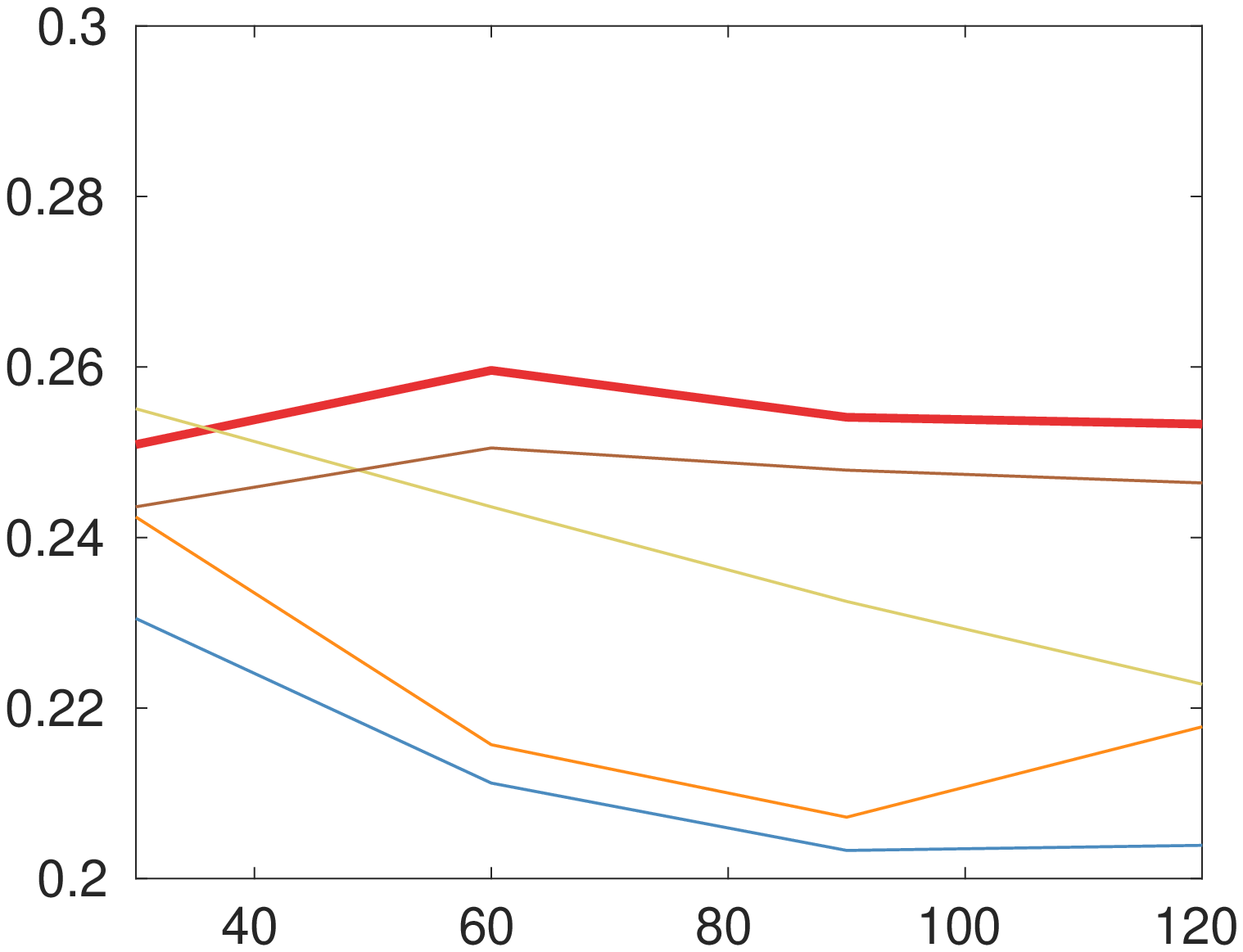}
    \\
    \includegraphics*[width=0.5\linewidth]{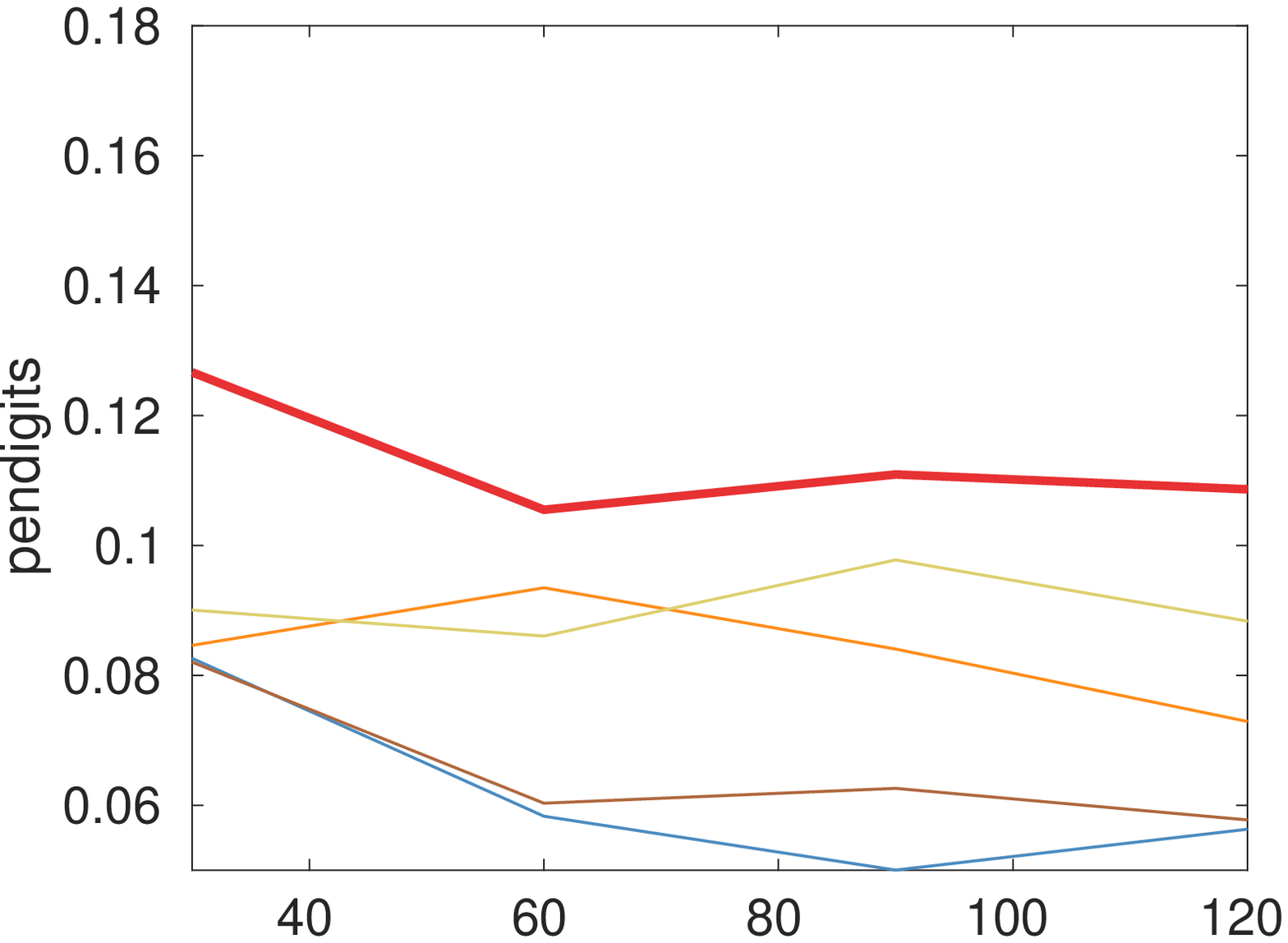}&
    \includegraphics*[width=0.5\linewidth]{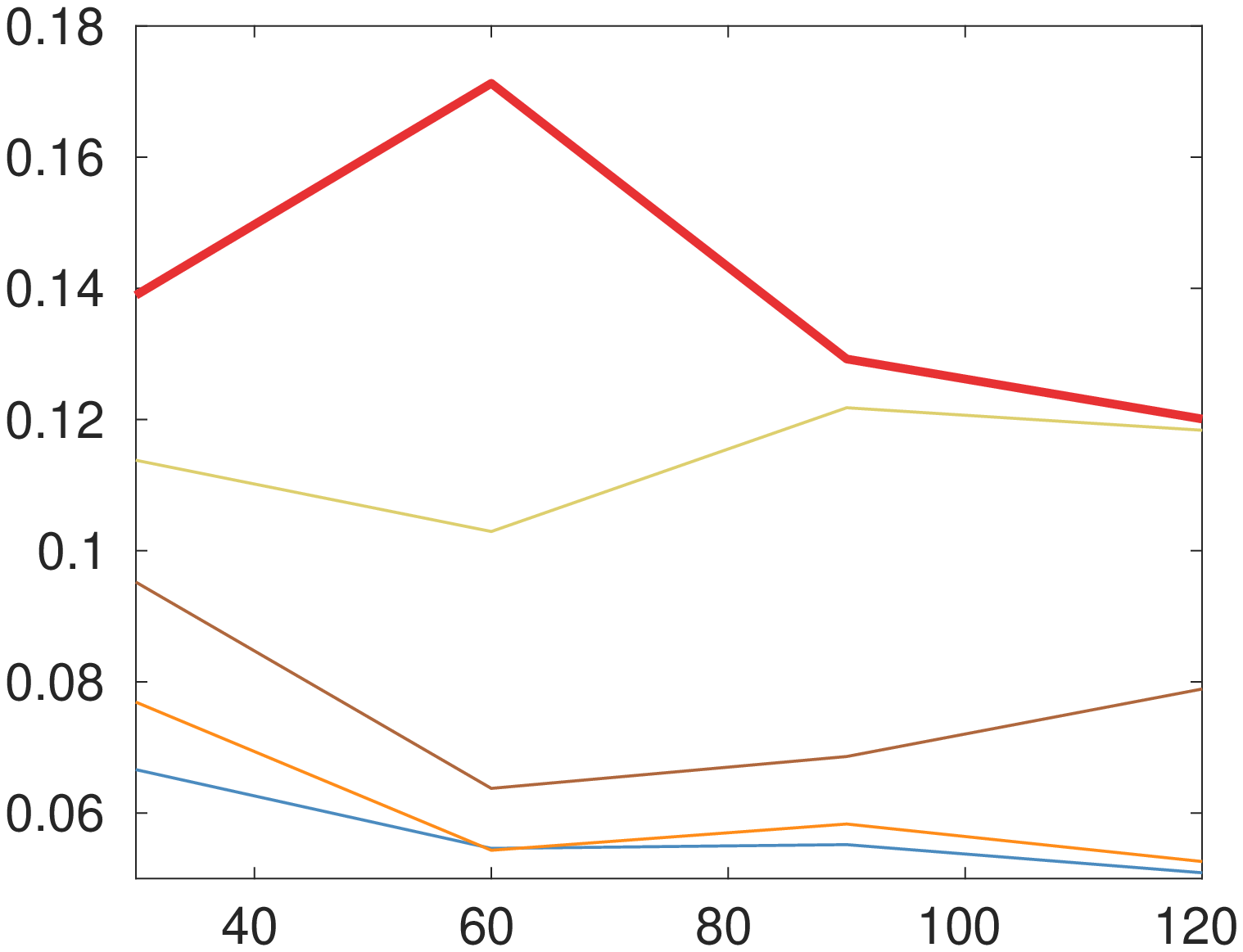}
    \\
    \includegraphics*[width=0.5\linewidth]{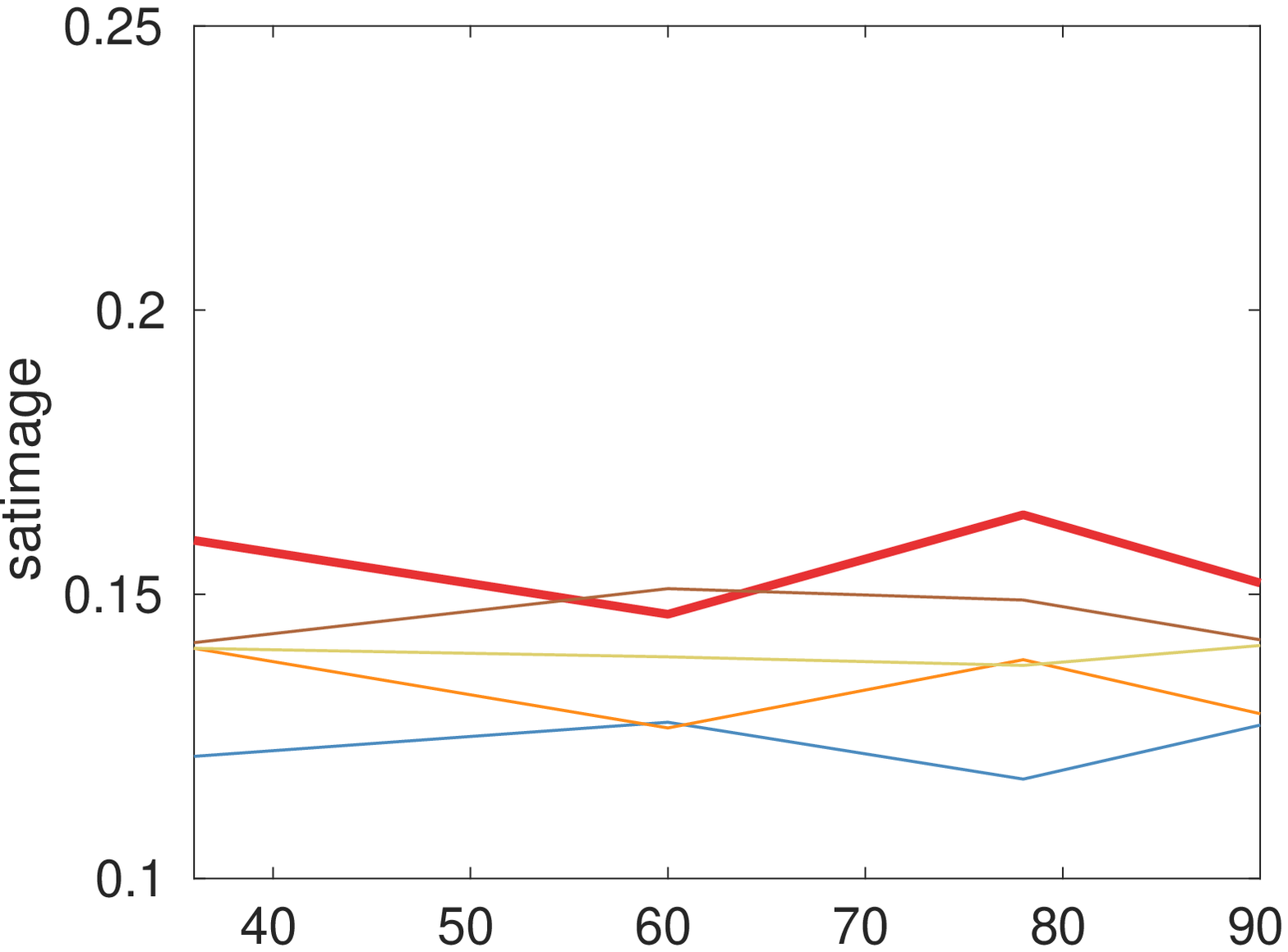}&
    \includegraphics*[width=0.5\linewidth]{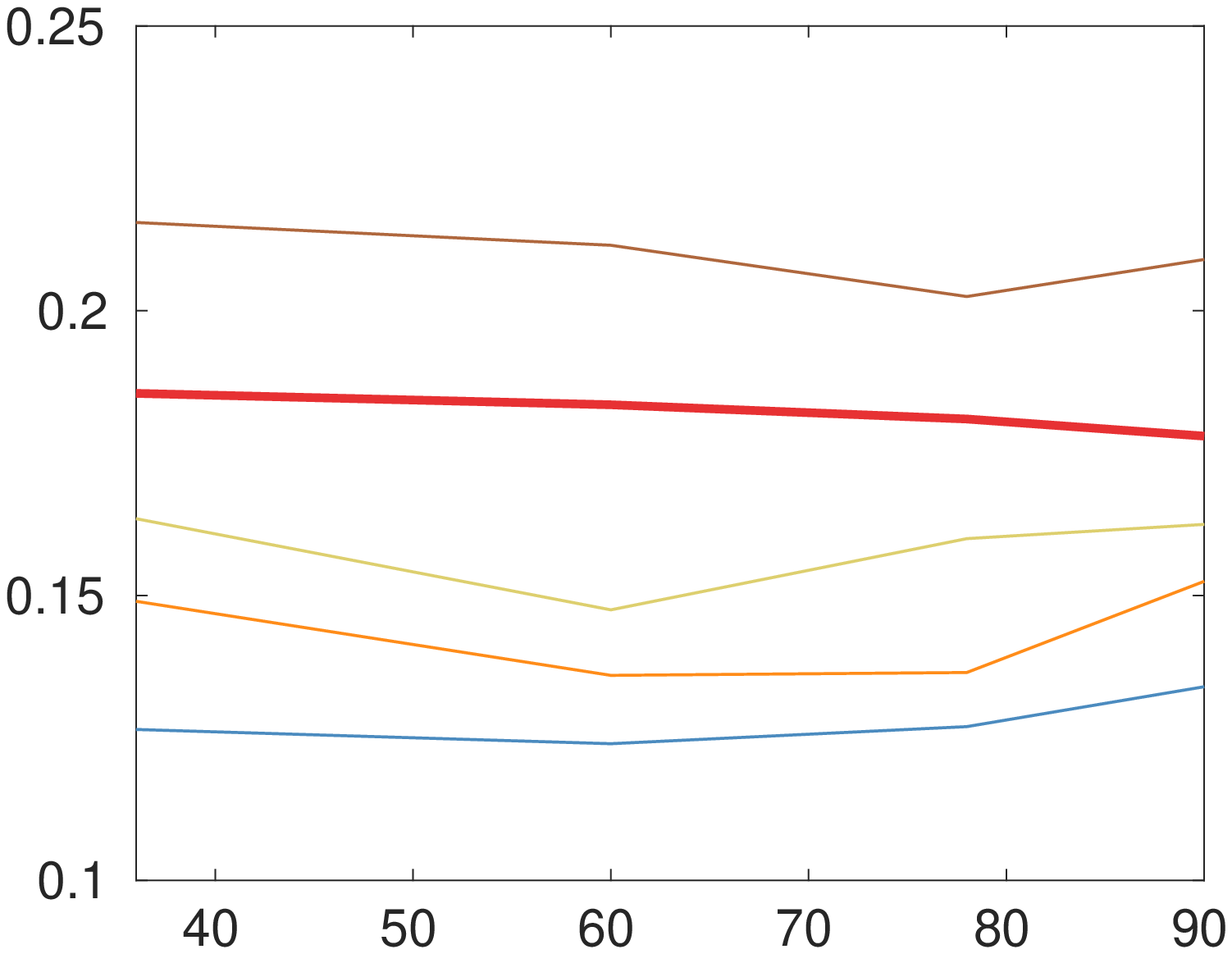}
    \\
  \end{tabular}
  \caption{Different attack techniques are presented for each dataset. Vertical axis presents error ratio over the testset and horizontal axis presents number of base models.}
\label{fig:attacks}
\end{figure}

In order to evaluate performance of the attack techniques, several machine learning models using the poisoned dataset are trained. In this paper, SRNN, RBF-SVM, OC1, CART, K-means, AdaBoost, Nearest Neighbor, Random Nearest Neighbor are used. The details of these models are explained in the supplementary material.

In order to obtain a fair comparison, for each dataset, the lowest error rate of each attack model among all the trained machine learning models is presented in figure \ref{fig:attacks}. In figure \ref{fig:attacks}, vertical axis shows error ratio over test set and horizontal axis represents number of centroids, leaf nodes, RBFs and trees for SRNN, tree, RBF-SVM and forest models, respectively. In fact, all the models are presented as a kind of ensemble of smaller models, thus, it is possible to compare them based on the number of terms they contain. Figure \ref{fig:attacks} shows that in all attack techniques, the SRNN-att consistently achieved the highest test error with a significant margin, or performed at least as good as the next best attack when compared with other attack techniques. Additionally, it is noteworthy that in a few instances one of the other techniques was able to increase the test error significantly but none of the other techniques was able to consistently achieve a high rise in the test error.

In figure \ref{fig:attacks-all}, results for best performing models are shown for the MNIST dataset. As can be observed, SRNN-att was able to make the models underperform more than other attacks. Complete set of experiments are presented in the supplementary materials.

\begin{figure}
\centering
    \begin{tabular}{c@{}c@{}c@{}}
    test error for $5\%$ attack & test error for $10\%$ attack
    \\
    \includegraphics*[width=0.5\linewidth]{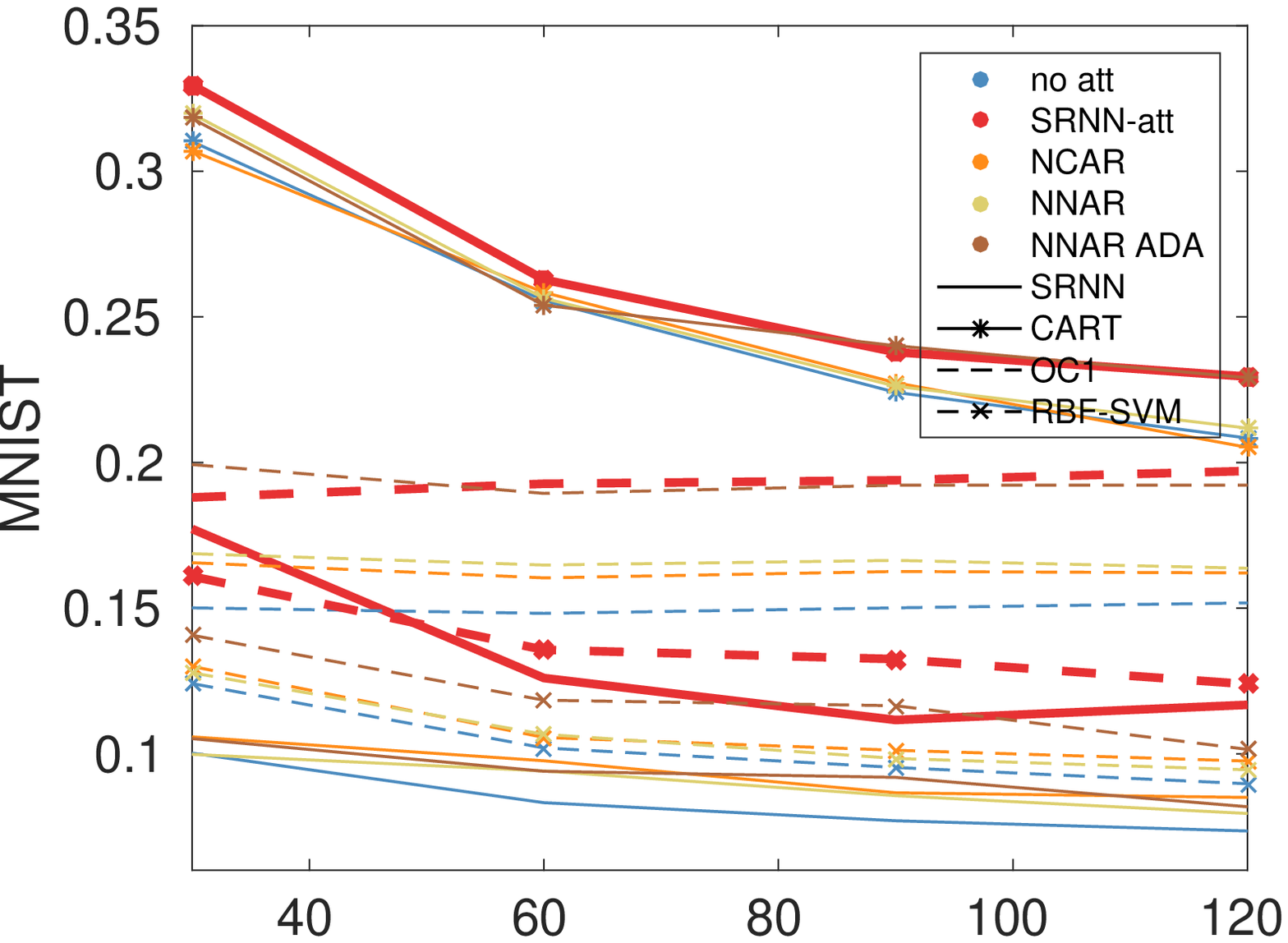}&
    \includegraphics*[width=0.5\linewidth]{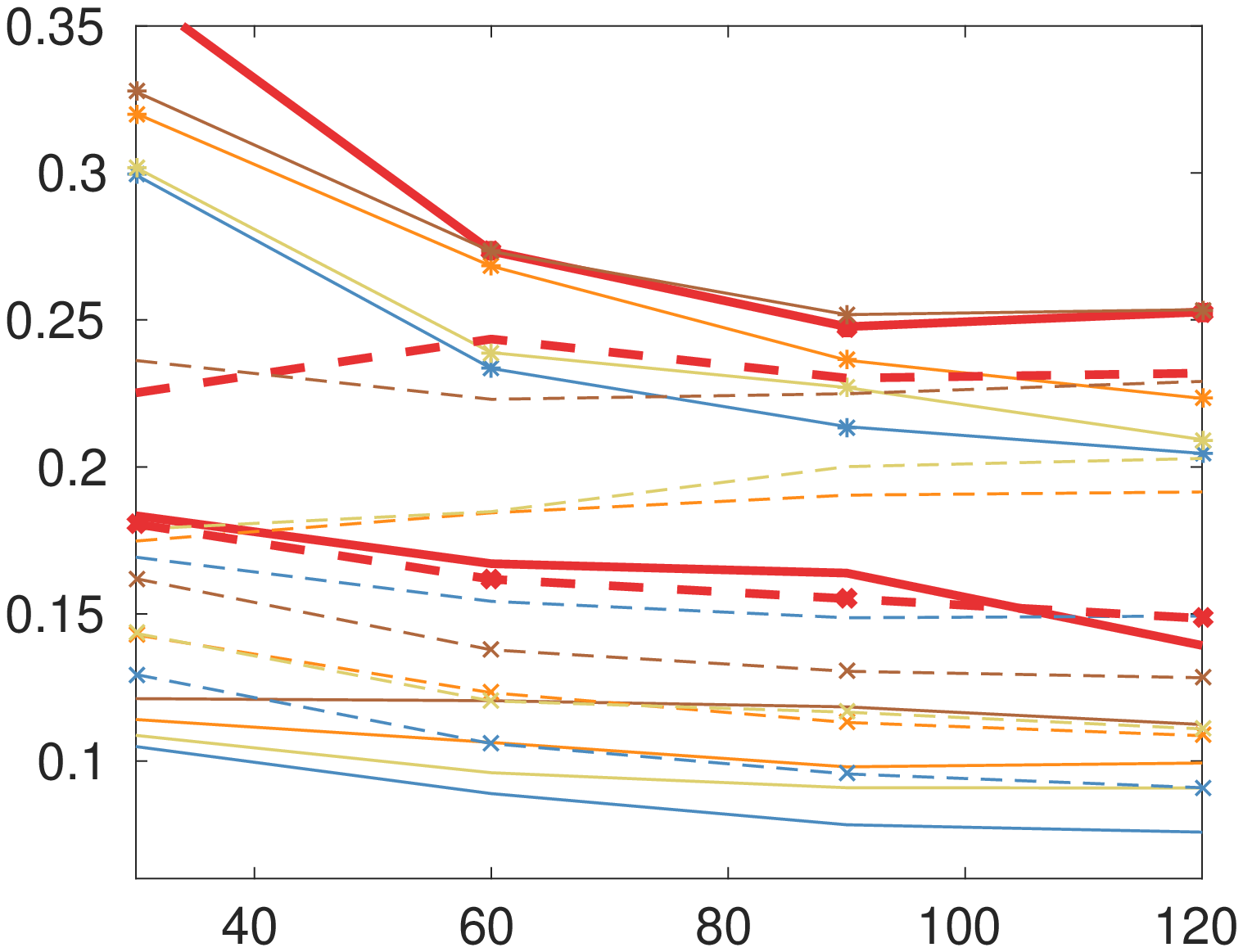}
    \\
  \end{tabular}
  \caption{Different attack techniques and models are presented for MNIST dataset. Colors and curve shapes present the attack techniques and models, respectively. model Vertical axis presents error ratio over the testset and horizontal axis presents number of base models.}
\label{fig:attacks-all}
\end{figure}


\subsection{Defense Experiments}
For the evaluation of the proposed defense against the proposed attack technique, the proposed RSRNN approach is compared with the other state of the art models that are known to be resilient against label flipping issues. Other models consist of trees, RBF-SVM, AdaBoost, K-means. For trees, tree with two split criteria of Gini index and Cross-entropy were used because they are known to be resilient against noisy or adversarial samples \cite{abellan2003building}. Another approach against adversarial/noisy label flipping attack techniques is using validationset for model selection \cite{papernot2018sok,frenay2014comprehensive}. This validation set is used for selecting the best model that has smallest error over the validation set\cite{hastie2009elements}. The validation set is used for selecting parameters of K-means and RBF-SVM in the experiments of this subsection, thus, making the models more resilient against the label noise. Finally, it is also known that ensemble models such as boosting algorithms are more robust against noise in labels \cite{frenay2014comprehensive}. Therefore, AdaBoost is added to the experiments of this subsection. The details of training each model is presented in the supplementary materials.

\begin{figure}
\centering
    \begin{tabular}{c@{}c@{}c@{}}
    $5\%$ attack & $10\%$ attack
    \\
    \includegraphics*[width=0.5\linewidth]{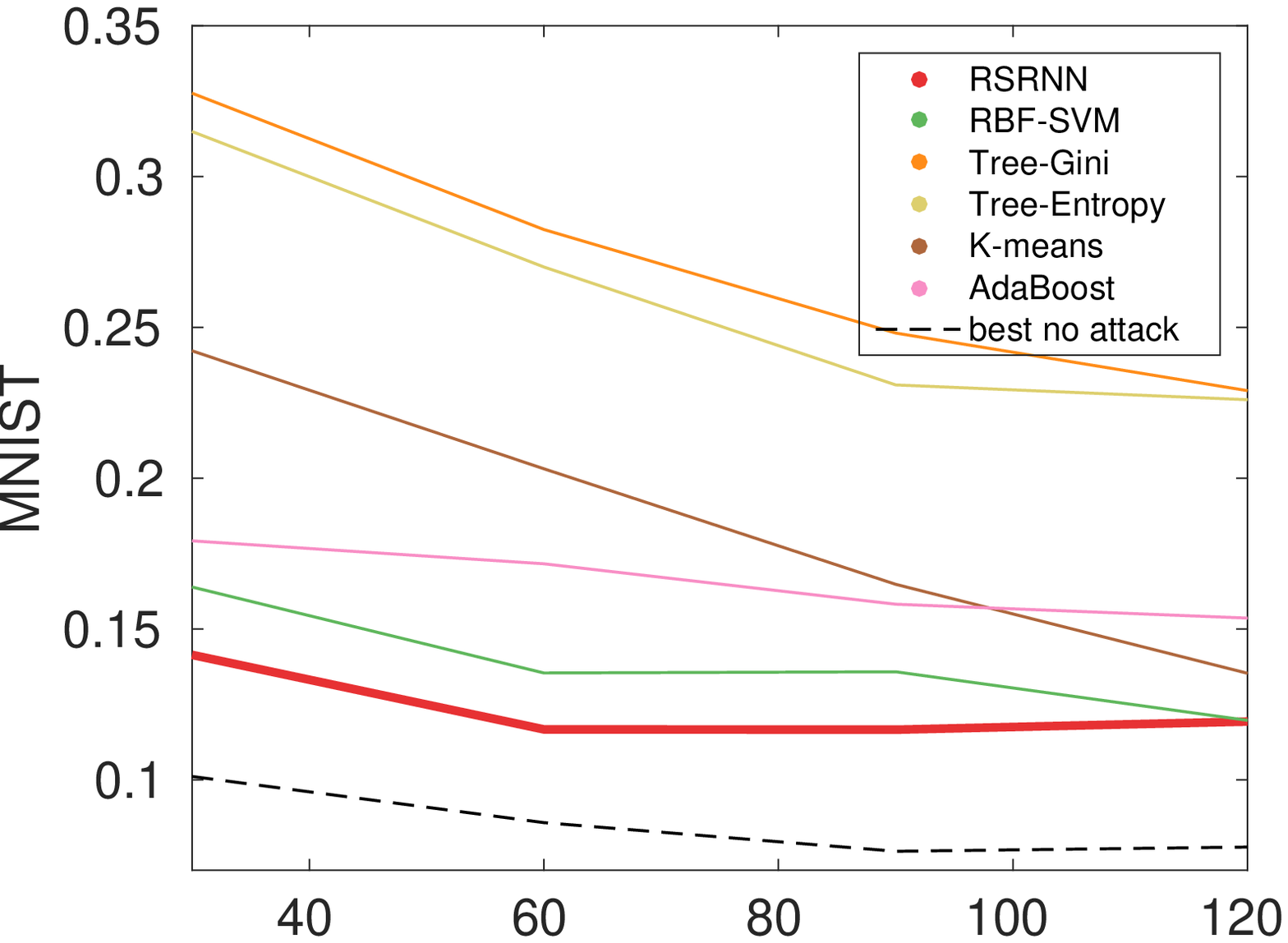}&
    \includegraphics*[width=0.5\linewidth]{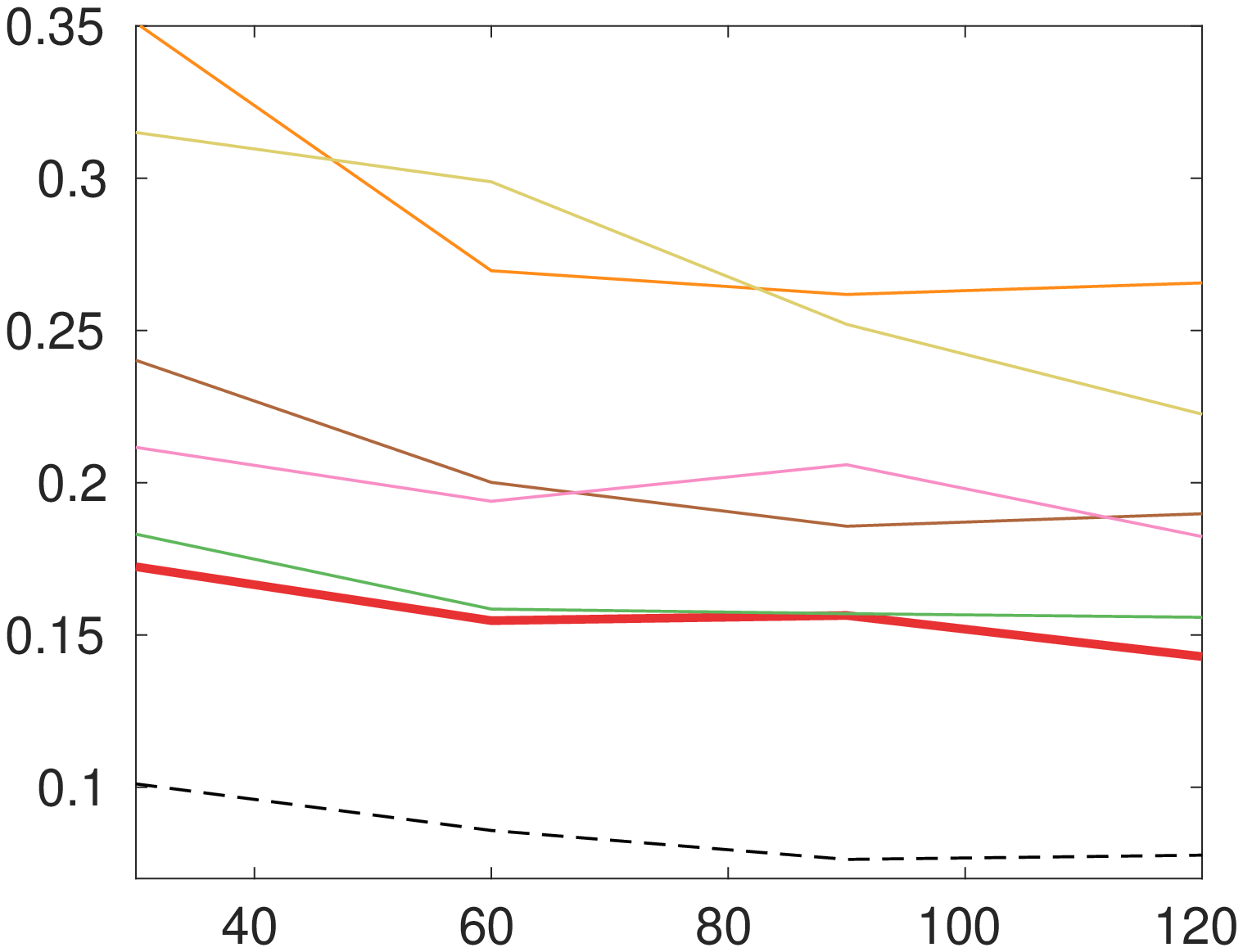}
    \\
    \includegraphics*[width=0.5\linewidth]{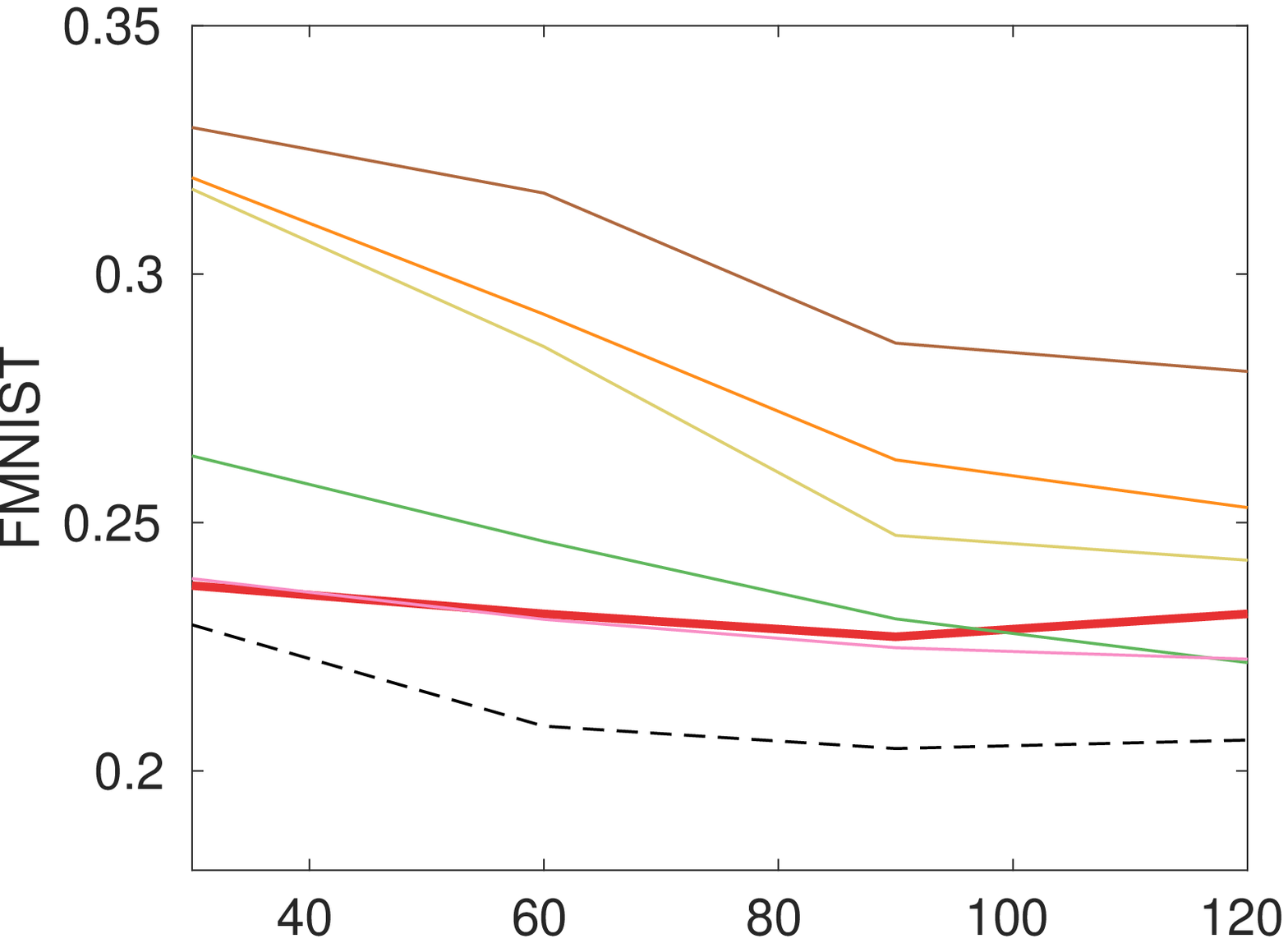}&
    \includegraphics*[width=0.5\linewidth]{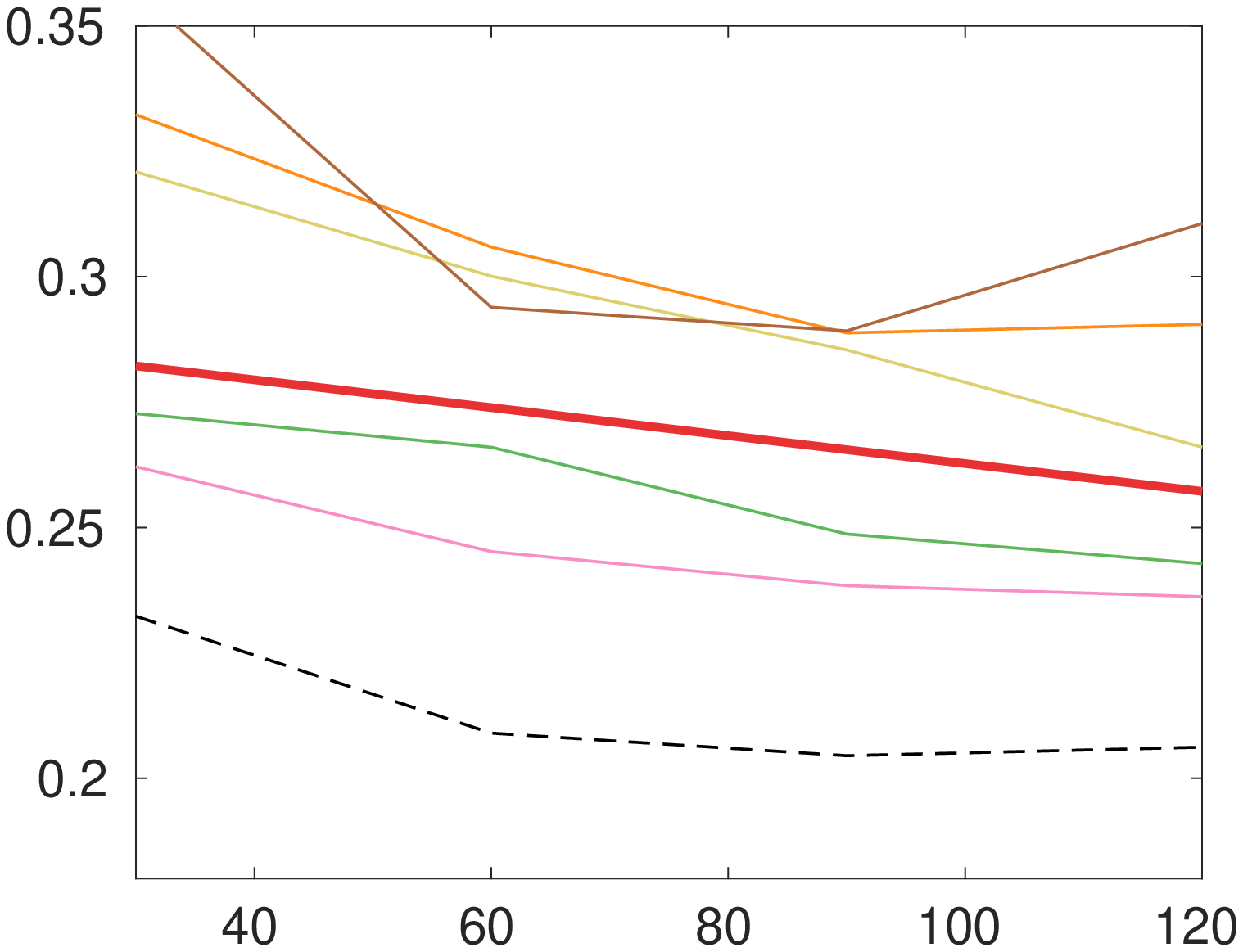}
    \\
    \includegraphics*[width=0.5\linewidth]{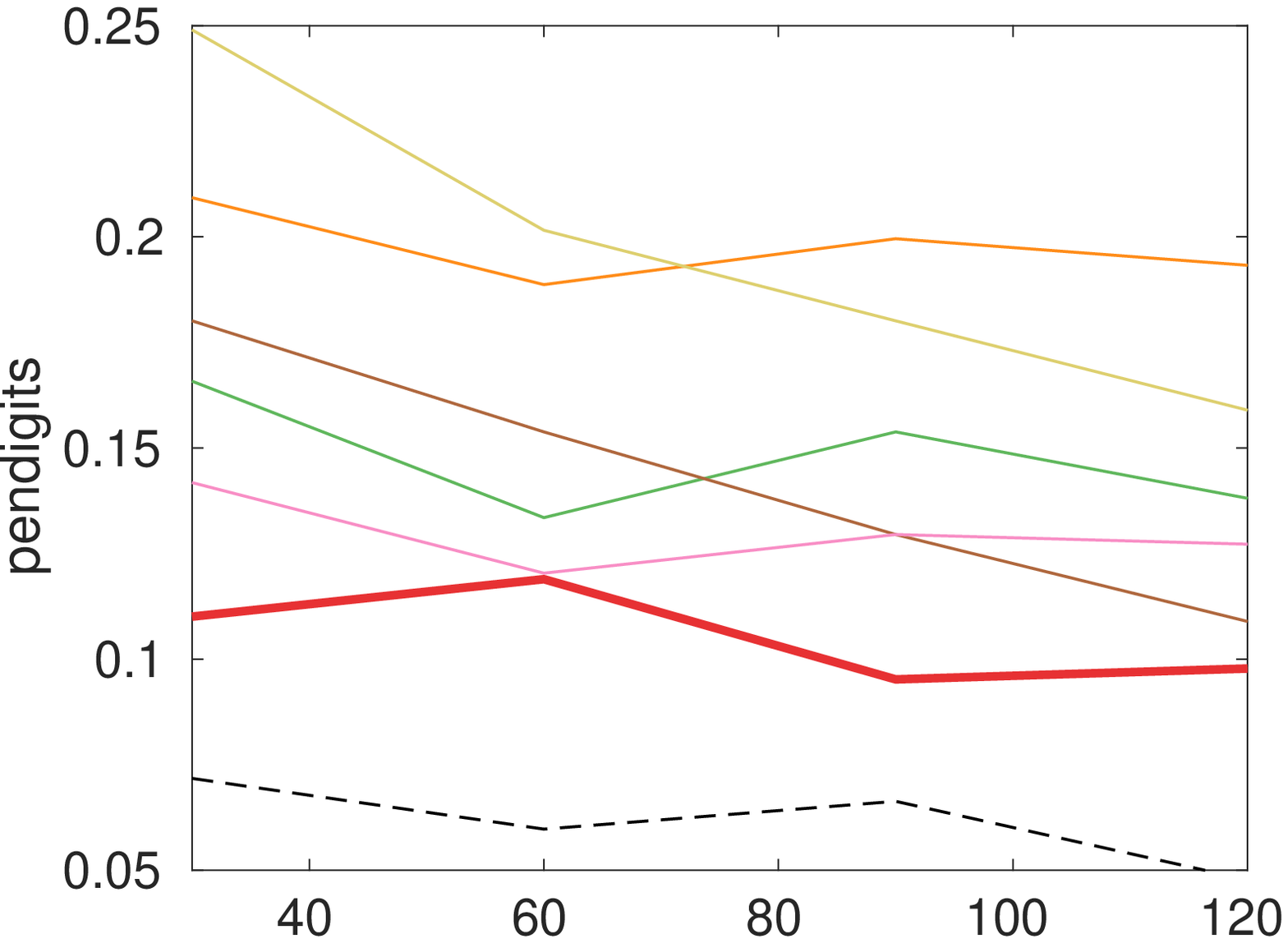}&
    \includegraphics*[width=0.5\linewidth]{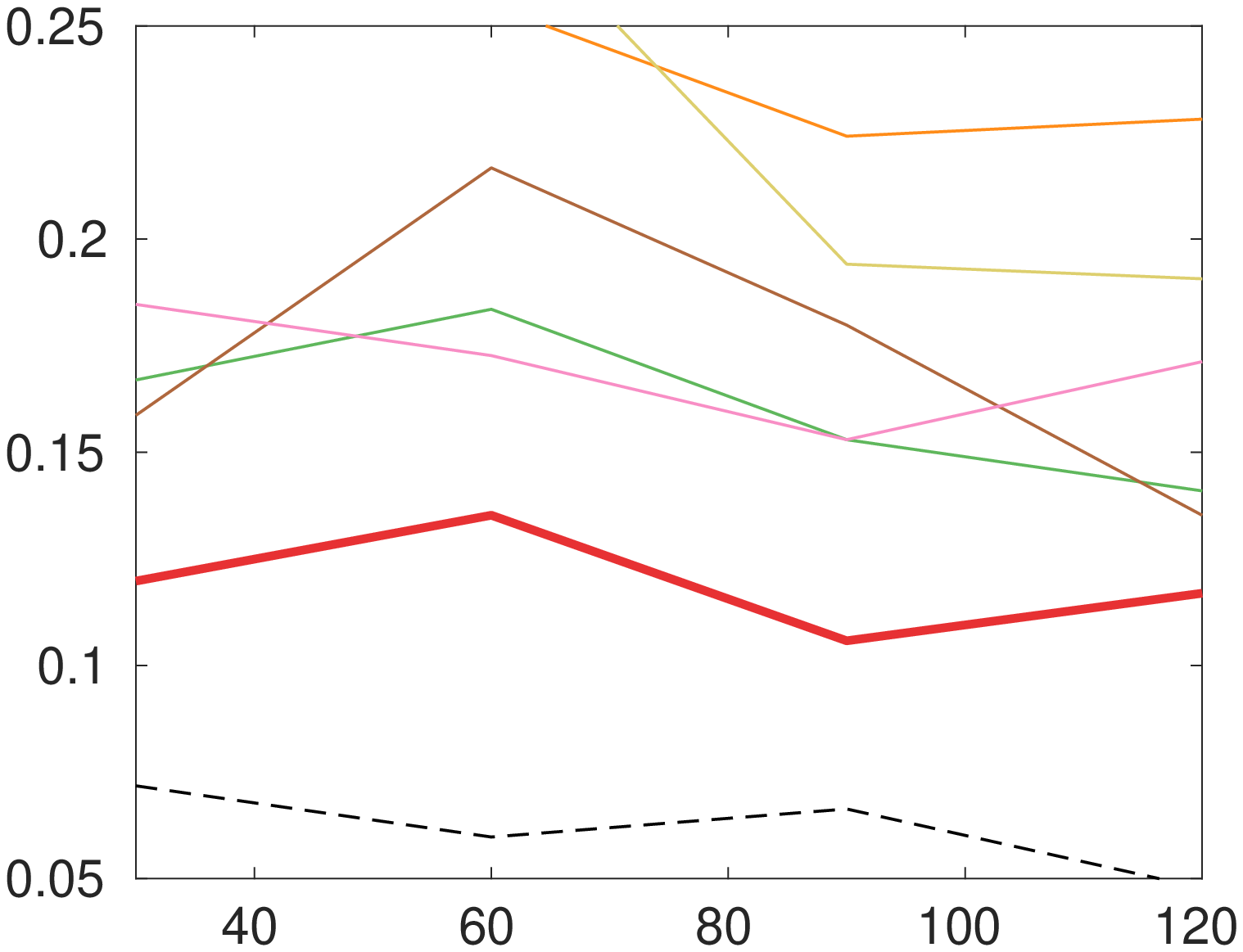}
    \\
    \includegraphics*[width=0.5\linewidth]{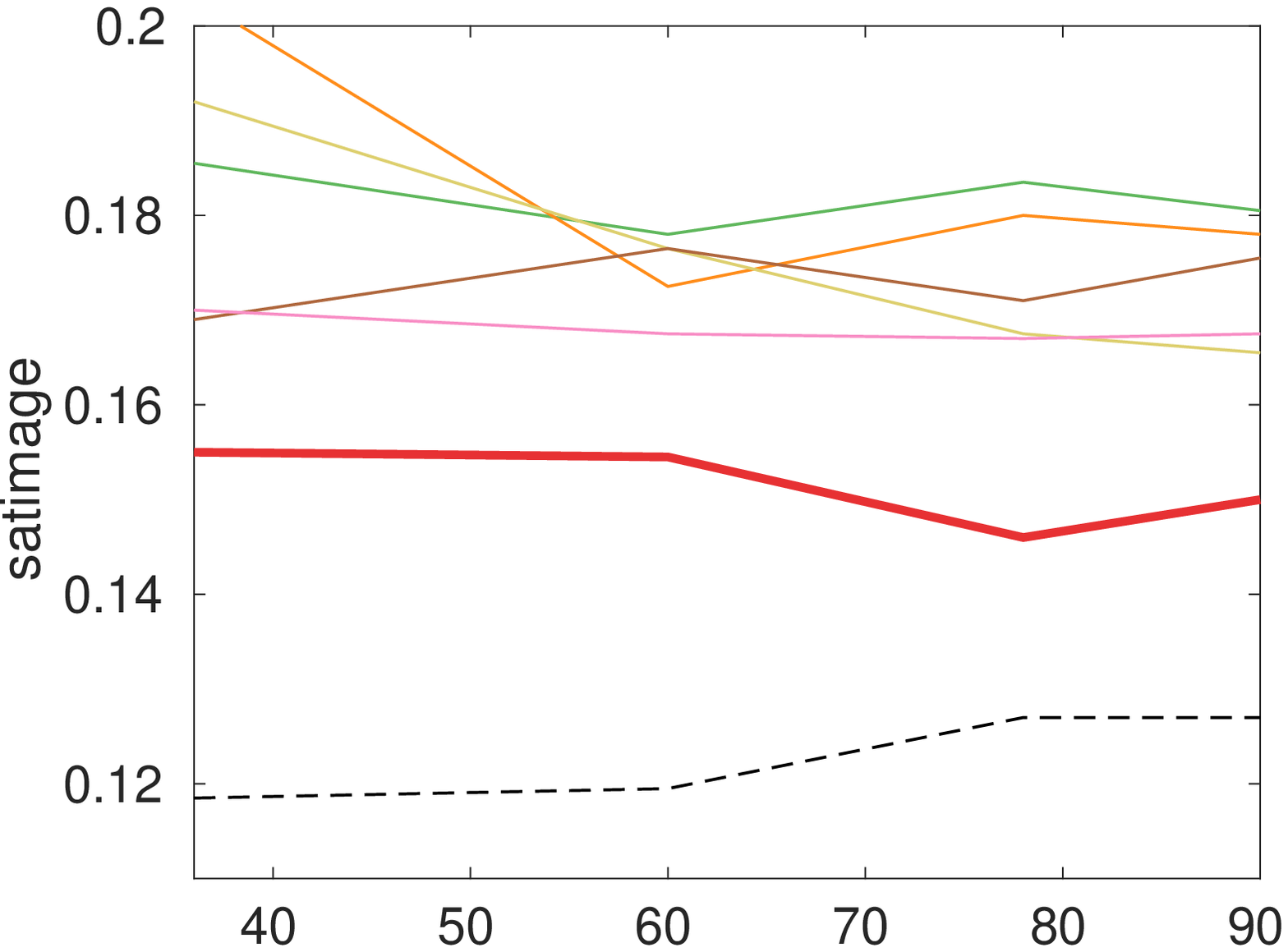}&
    \includegraphics*[width=0.5\linewidth]{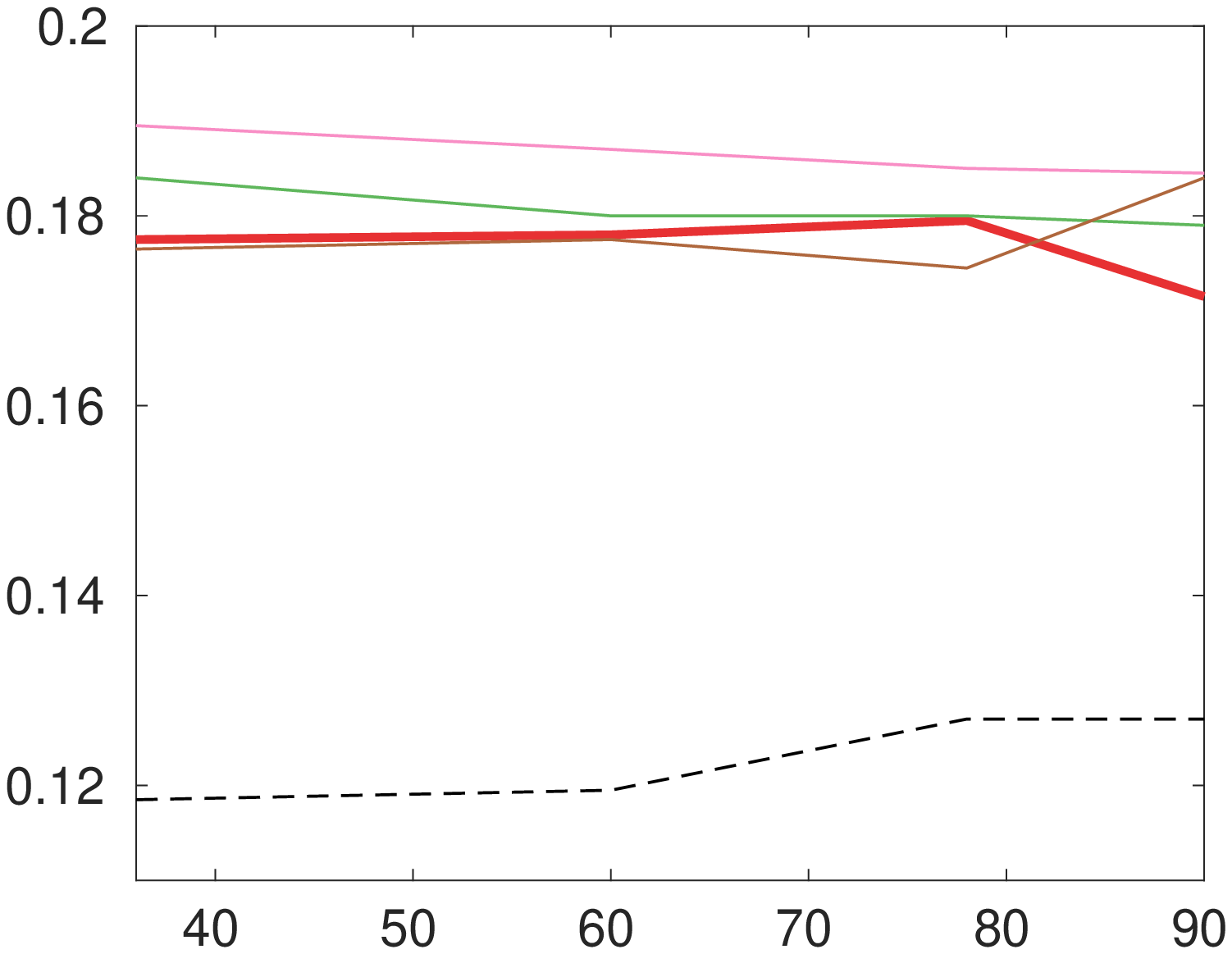}
    \\
  \end{tabular}
  \caption{Different defenses against SRNN-att. SRNN-att model was trained with 30 centroids. Vertical axis shows test error ratio and horizontal axis represents number of basis models.}
\label{fig:defs}
\end{figure}

Figure \ref{fig:defs} presents the results of experiments in this subsection. As can be observed from figure \ref{fig:defs}, RSRNN was able to constantly outperform other models with a large margin. At the same time, RSRNN was able to improve the results of SRNN by up to $2-3\%$. Additionally, RSRNN was able to detect a large portion of malicious samples up to $70\%$ with a true positive of $50-60\%$. Experiments regarding performance of RSRNN in detecting malicious samples is presented in the supplementary materials.  

RSRNN was able to achieve smallest test error while also having detected a high portion of malicious samples. RSRNN is able to find smaller clusters of data that are vulnerable to adversarial label flipping. Finally, the size of validation set used in the experiments of this section is only $8\%$ of the trainset. Further details of the experiment setup and more experimental results on the efficiency of the model under various configurations of the attack are presented in the supplementary materials.
\section{Conclusion}
In this paper, a novel data poisoning attack was proposed that is able to deteriorate the performance and undermine the integrity of state of the art machine learning models. This is the first data poisoning technique that is not limited to only binary classification, a specific model, or gradient-based approaches. At the same time, the attack technique is very fast since it only takes a linear time over the dataset to select samples for label flipping. The proposed attack shows the fact that for an adversary, it would be more efficient to target minority groups in a dataset with the goal of undermining the integrity of the learned model and achieving this goal. 
 The properties of this attack does not require the knowledge of the user model since its attack can affect any model.

In addition, a novel defense technique based on SRNN model was proposed that is resistant against the proposed attack technique. The defense technique intrinsically detects and excludes the adversarial samples on the fly during its training procedure. Additionally, our experimental results showed that the RSRNN model is capable of detecting a large portion the malicious samples, thus, making it more robust to data poisoning. Finally, in experiments, RSRNN showed ability to achieve significantly lower test error compared to other known resilient models against the label flipping data poisoning.




\section{Supplementary Materials}
\appendix
\section{Upper Bound for the Proposed Attack Technique}

\begin{thm}[Fixed SRNN upper bound]
  \label{th:upper_bound}
  Assuming that the centroids are fixed, and known to both user and adversary, the modality-based attack can increase the error of $NN^*$ by at most ${\cal O}(2 \times Cost)$ over the trainset with unperturbed labels.
\end{thm}
\begin{proof}
The centroids are fixed. Therefore, the loss incurred by each centroid is independent of the other centroids, thus, it is possible to break down the total loss by each centroids loss.
\begin{equation}
\label{eq:total_loss}
\sum_{i=1}^N L(y_i,NN^*(x_i))=\sum_{j=1}^K L_j
\end{equation}
Where, $L_j=\sum_{x_i \in S_j^{train}} L(y_i,NN^*(x_i))$.
Changing $\hat{y}_j$ to label of minority class will increase the error in the $j^{th}$ cluster by at most $|S_j|$. The cost of such action is $Cost_j=\frac{|S_j|}{2}$. Therefore, the change in loss of $j^{th}$ centroid is 
$\Delta L_j\leq 2\times Cost_j$.
The second constraint of problem (4) of the main paper is satisfied by the adversary. Therefore, using \eqref{eq:total_loss} and the second constraint of (4) of the main paper, we have  
\begin{equation}
\Delta \sum_{j=1}^K L_j = \sum_{j=1}^K \Delta L_j I_j \leq \sum_{j=1}^K 2\times Cost_j I_j < 2\times Cost.
\end{equation}

\end{proof}

\section{Computational Complexity and Convergence}
\label{sec:suppintro}
The attack technique consist of first training SRNN model that takes ${\cal O}(NDK)$ \cite{tavallali2020interpretable} and then perturbing labels of samples. The second step consists of calculating cost of each cluster (takes ${\cal O}(|S_j|)$) and sorting the costs of clusters (takes ${\cal O}(KlogK)$). Therefore, in total, the attack technique takes ${\cal O}(NDK+N+KlogK)$.

The problem of learning RSRNN is 
\begin{equation}
	\label{eq:RSRNN3_obj}
	\begin{split}
	& \underset{\{(c_j,\hat{y_j},r_j)\}_1^K}{\text{min}}  \quad \sum_{i=1}^{N} L(y_i,NN(x_i)) + \lambda \sum_{j=1}^K r_j +\alpha \sum_{j=1}^K cost(S_j))\\
	& \text{s.t.} \quad NN(x_i)=\begin{cases}
	\hat{y_{j^*_i}} \quad r_{ij^*}<r_{j^*} \\ Malicious \quad \text{otherwise}
	\end{cases}
	\end{split}
\end{equation}

\begin{thm}[Convergence of RSRNN]
  \label{th2:upper_bound}
  Iterating over the first two steps of optimizing RSRNN converge to a local minimum of the first two terms of objective function in \eqref{eq:RSRNN3_obj}. This takes a finite number of iterations.
\end{thm}
\begin{proof}
The objective function has a lower bound of zero. Both the assignment and the centroid steps decrease or do not change the first two terms of \eqref{eq:RSRNN3_obj}. Different combinations of assignments of the samples to the centroids are finite. As a result, the loss function does not decrease after a finite number of iterations over both steps. This proof is similar to proof of convergence for K-means \cite{lloyd1982least}.
\end{proof}

\section{Pruning}
After optimizing \eqref{eq:RSRNN_obj} using the assignment step, pruning is performed. The pruning step consists of removing malicious centroids and samples. To do so, other parameters are kept constant, a validation set is needed and optimization over $\alpha$ must be performed since $\alpha$ is a hyperparameter. A direct selection of $\alpha$ is not practical since $\alpha$ can accept any values from $0$ to $\infty$. Thus, we will show that Values of $\alpha$ that cause change to the structure of RSRNN are quantized and correspond to a specific cut-off threshold over $cost(S_j)$. $cost(S_j)$ is a real number in $[0,1]$. Therefore, a range of numbers between $[0,1]$ are used to remove centroids based on the trainset and evaluate their performance over the validationset. Any cut-off that had smallest validation error was selected.

For each centroid, the cut-off threshold is the threshold that by removing the centroid, the objective function of \eqref{eq:RSRNN_obj} decreases. Such threshold can be calculated as follows:
\begin{equation}
	\label{eq:RSRNN2_obj}
	\begin{split}
	cost(S_j)= \frac{L(y_i,NN_{C-c_j}(x_i))+\lambda r_j}{\alpha}
	\end{split}
\end{equation}
Where, $NN_{C-c_j}(.)$ is the nearest neighbor function over all centroids except $j^{th}$ centroid.

In practice, we used cut-off thresholds between $[0.2-0.9]$ with steps of $0.05$. For each cut-off, the centroid and its sampleset ($S_j$) was removed. Further, for each cut-off a retraining consisting of only initialization with $K$ centroids over the cleaned dataset was applied and its validation error was evaluated. Please note that the purpose of this pruning approach is detecting malicious samples along with removing malicious centroids. Finally, the model with smallest validation error was selected for further optimization over the cleaned dataset.

\section{Experimental results}

\textbf{Dataset setups:}We performed the experiments over two setups for each dataset. In one setup, $20\%$ of the dataset for selected as trainset and a validation set with size of $8\%$ of the trainset was selected. The validationset contains the original(unperturbed) labels for the samples. The trainset was poisoned using SRNN-att. The SRNN-att was trained over $80\%$ of the dataset for all setups. The reason for such setup is that in the proposed method the attacker was supposed to have the knowledge of the optimal centroids. Therefore, to mimic such assumption, the attacker was trained over a larger set of samples. The second setup, $80\%$ of the dataset was used as trainset and a validationset with the size of $8\%$ of the trainset was selected. The trainset was poisoned using same SRNN-att as the first setup. The second setup is used for all datasets. In the main paper the first setup was used for datasets of MNIST, FMNIST and pendigits. For satimage the second setup was used since using the first setup was causing overfitting in all trained models. The experiments over MNIST, FMNIST and pendigits using the second setup are presented in this supplementary materials. Finally, we noticed that SRNN-att can increase the error significantly even if it is trained over the same portion of samples as the other models.

\textbf{Models setups:} In the main paper various models were used. Each model can be represented as a kind of ensemble model that consist of several base models. Here, the setups for each model is presented. For decision tree models of OC1 \cite{murthy1994system} and CART \cite{breiman1984classification} the number of leaf nodes represents the number of base models. For RBF-SVM, first several $K$ centroids using K-means over the trainset was found and used as centers for RBFs. Each centroid represents one base model. The hyperparameters of RBFs were found using the trainset for the attack experiments. In the defense experiments, the RBF hyperparameters were selected using the clean validationset. For the K-means model, a simple unsupervised K-means was trained over the features of the trainset and then the labels were selected using the assignment step. The labels for attack and defense experiments were selected using the trainset and validationset, respectively. The number of centroids represents the number of base models. For AdaBoost, each base model was a tree of depth 4 and the model was trained such that it contained similar number of parameters as RBF-SVM. SRNN was trained using same number of centroids as the other models. In random Nearest Neighbors, $K$ random samples were selected as a the model. In the experiments, we have select models that are similar in the sense that each model would partition the input space or uses centroids as basis of the model. Neural Networks do not fall into any of the categories mentioned here.

\subsection{Attack Experiments}
In this subsection, the performance of all models over each dataset is presented in figure \ref{fig:attacks-all2} for first setup.
\begin{figure}[!ht]
\centering
    \begin{tabular}{c@{}c@{}c@{}}
    $5\%$ attack & $10\%$ attack
    \\
    \includegraphics*[width=0.5\linewidth]{grf/MNIST/5_20_30_att_test_acc.eps}&
    \includegraphics*[width=0.5\linewidth]{grf/MNIST/10_20_30_att_test_acc.eps}
    \\
    \includegraphics*[width=0.5\linewidth]{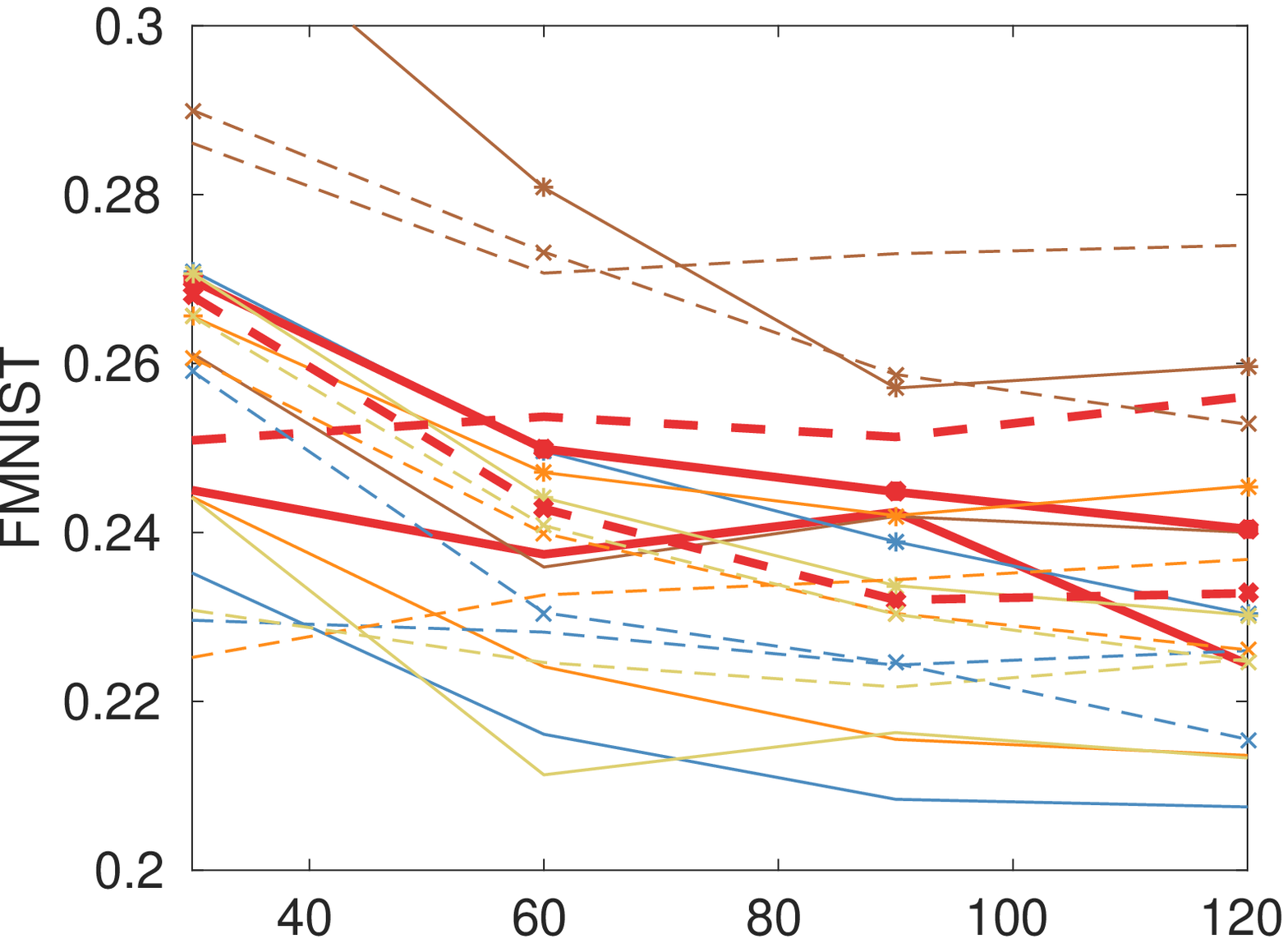}&
    \includegraphics*[width=0.5\linewidth]{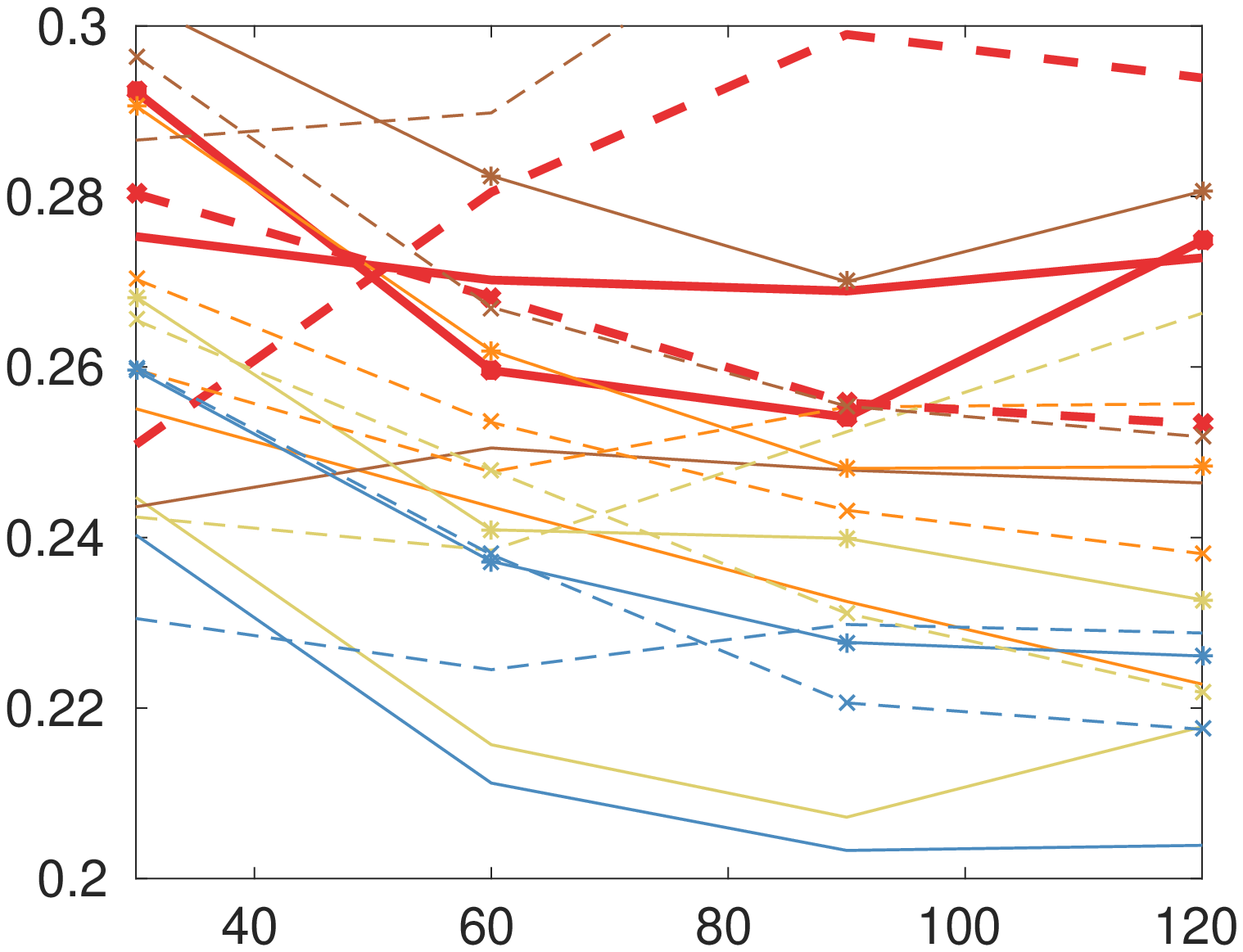}
    \\
    \includegraphics*[width=0.5\linewidth]{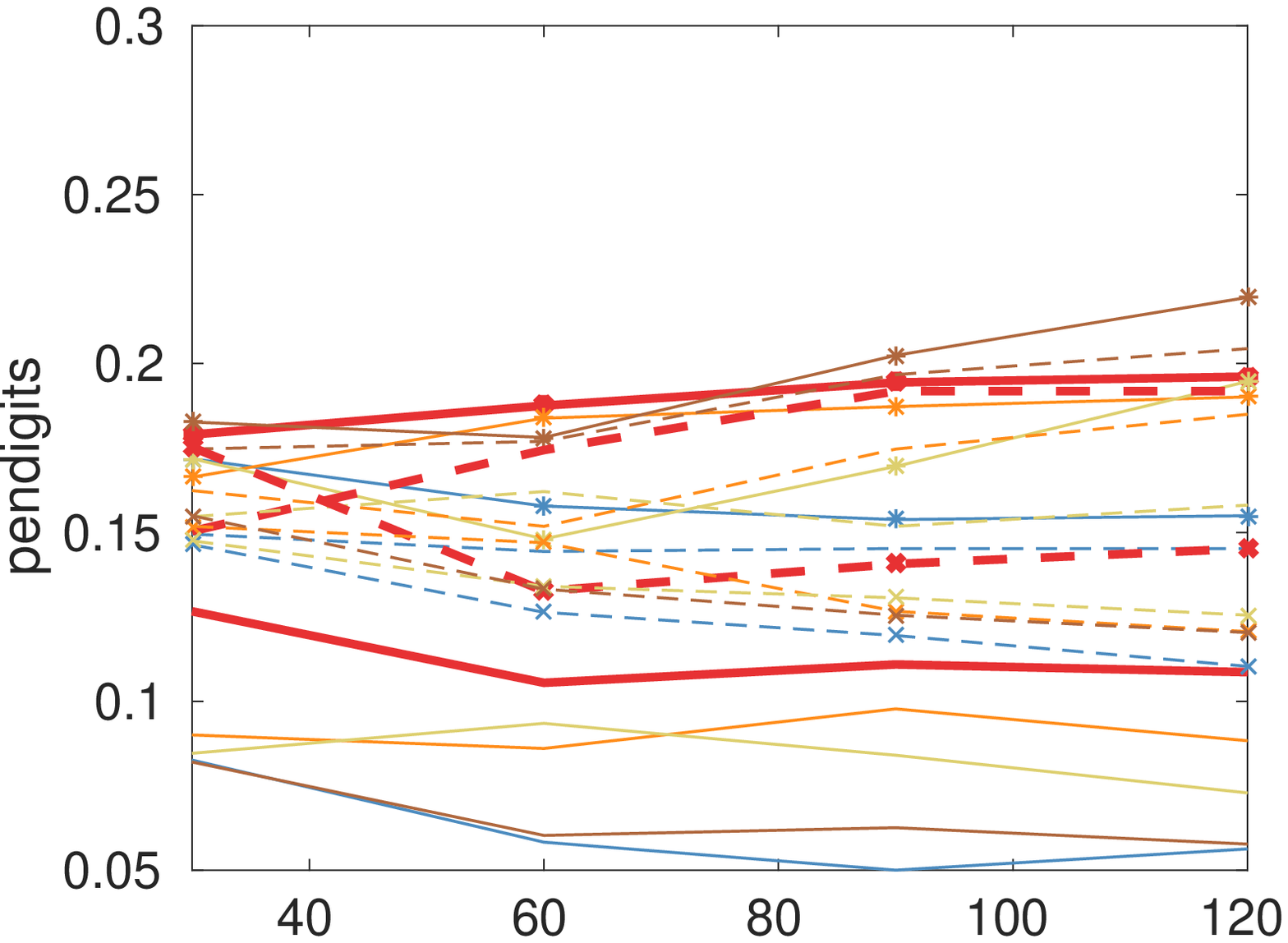}&
    \includegraphics*[width=0.5\linewidth]{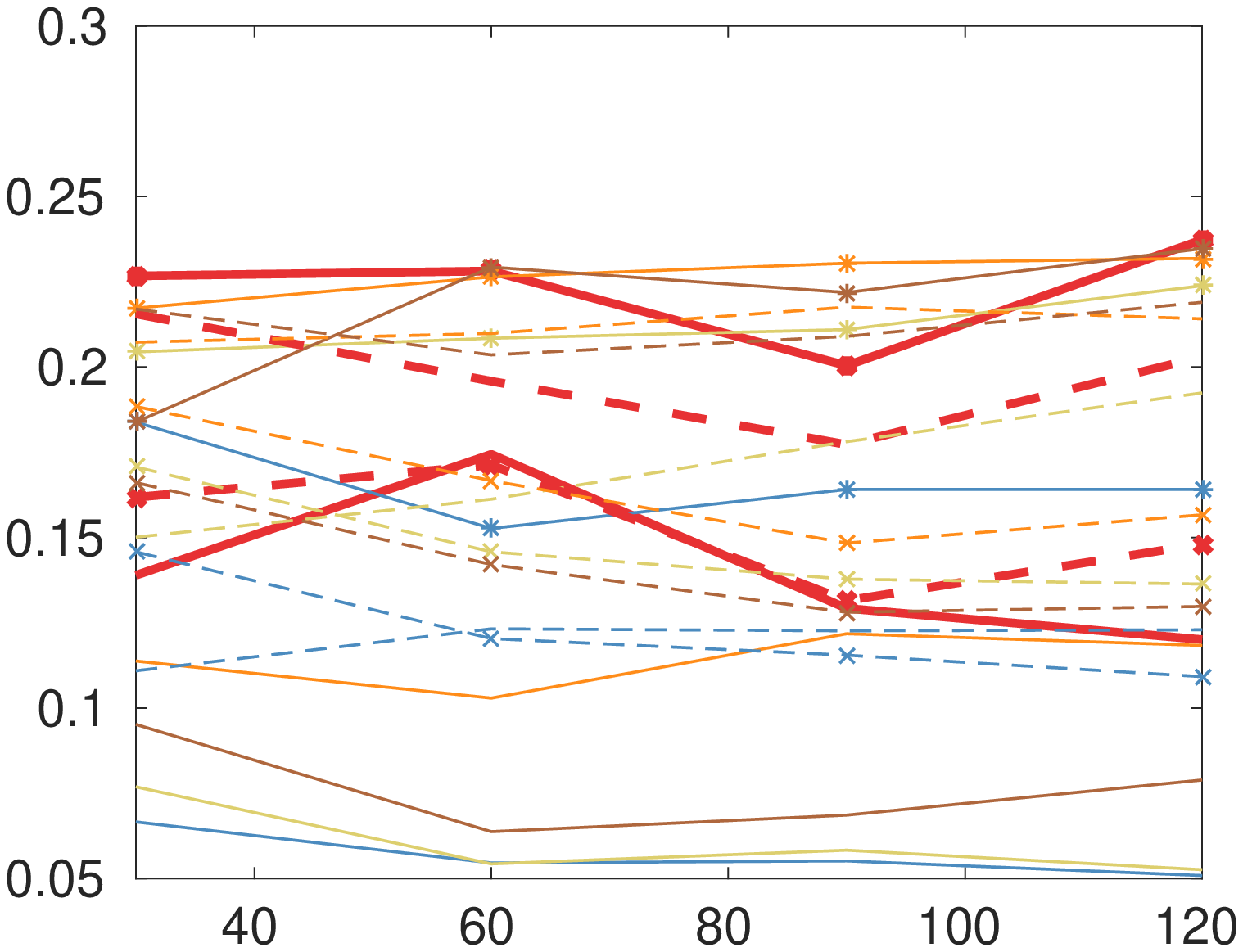}
    \\
  \end{tabular}
  \caption{Different attack techniques and models are presented for each dataset. Colors and curve shapes present the attack techniques and models, respectively. model Vertical axis presents error ratio over the testset and horizontal axis presents number of base models. For this figure, first setup was used.}
\label{fig:attacks-all2}
\end{figure}
As can be observed from figure \ref{fig:attacks-all2}, SRNN-att was able to increase the error of all models higher than other attack techniques.

Figure \ref{fig:attacks-80} presents the best of all models under each attack for second setup.
\begin{figure}[!ht]
\centering
    \begin{tabular}{c@{}c@{}c@{}}
    test error for $5\%$ attack & test error for $10\%$ attack
    \\
    \includegraphics*[width=0.5\linewidth]{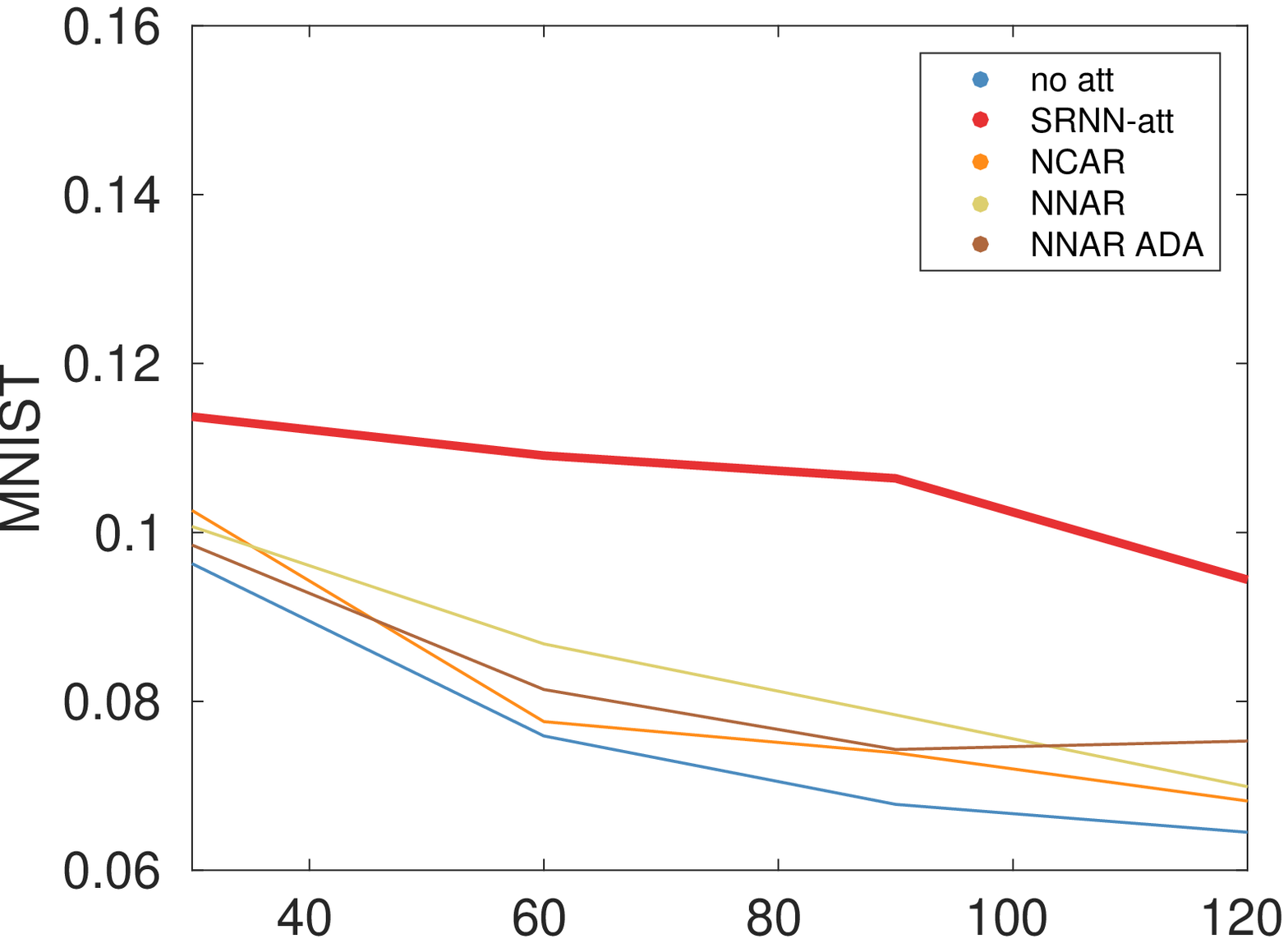}&
    \includegraphics*[width=0.5\linewidth]{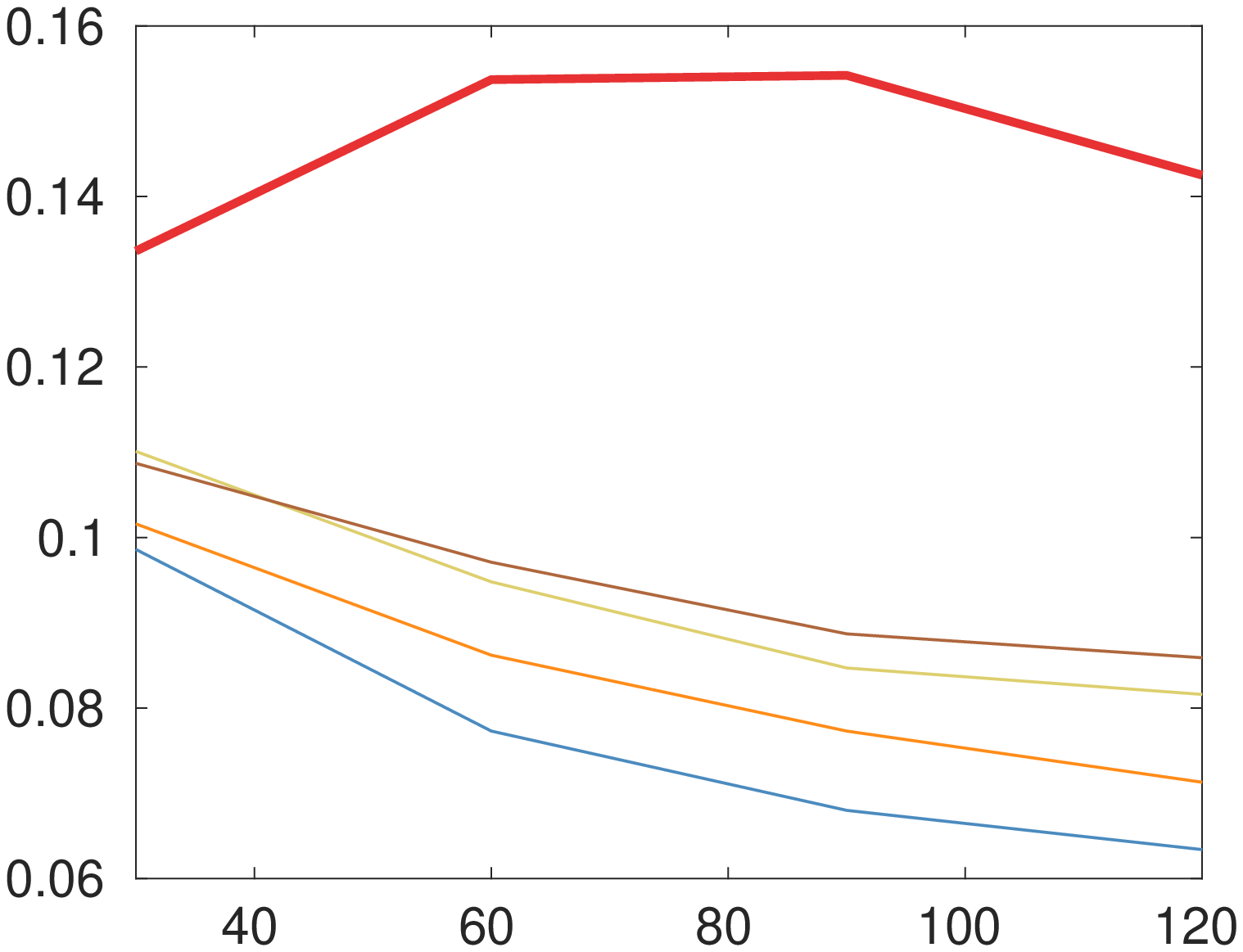}
    \\
    \includegraphics*[width=0.5\linewidth]{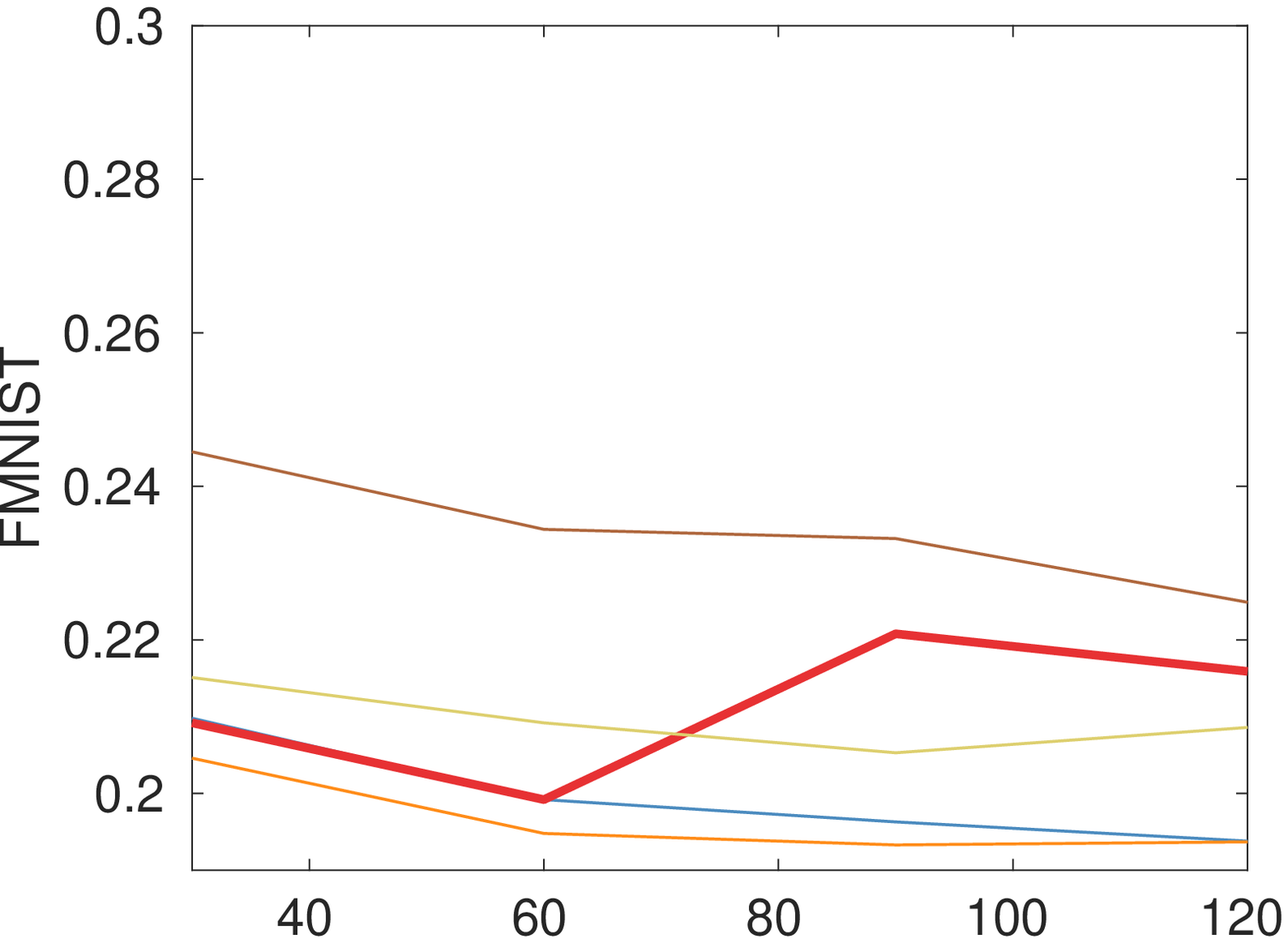}&
    \includegraphics*[width=0.5\linewidth]{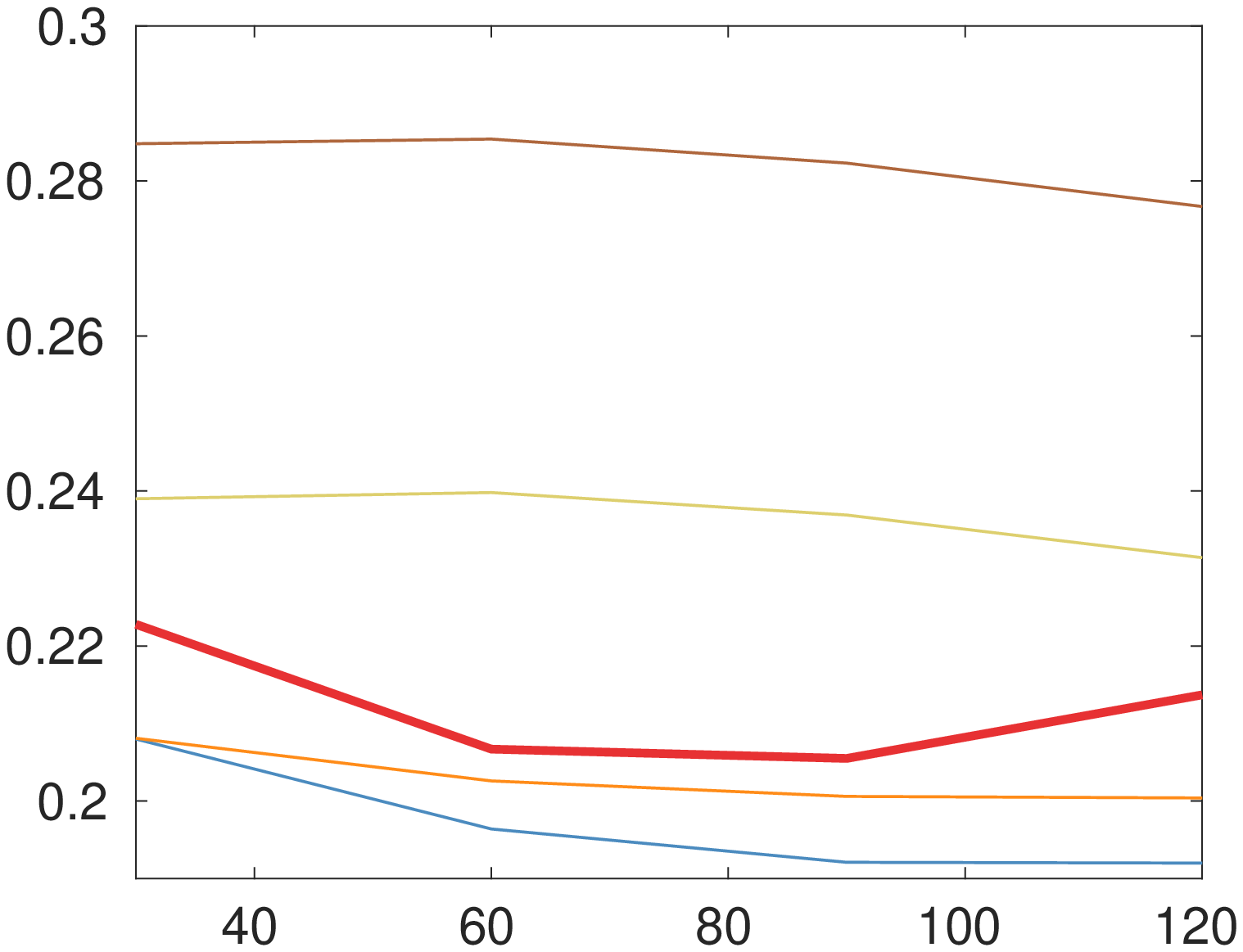}
    \\
    \includegraphics*[width=0.5\linewidth]{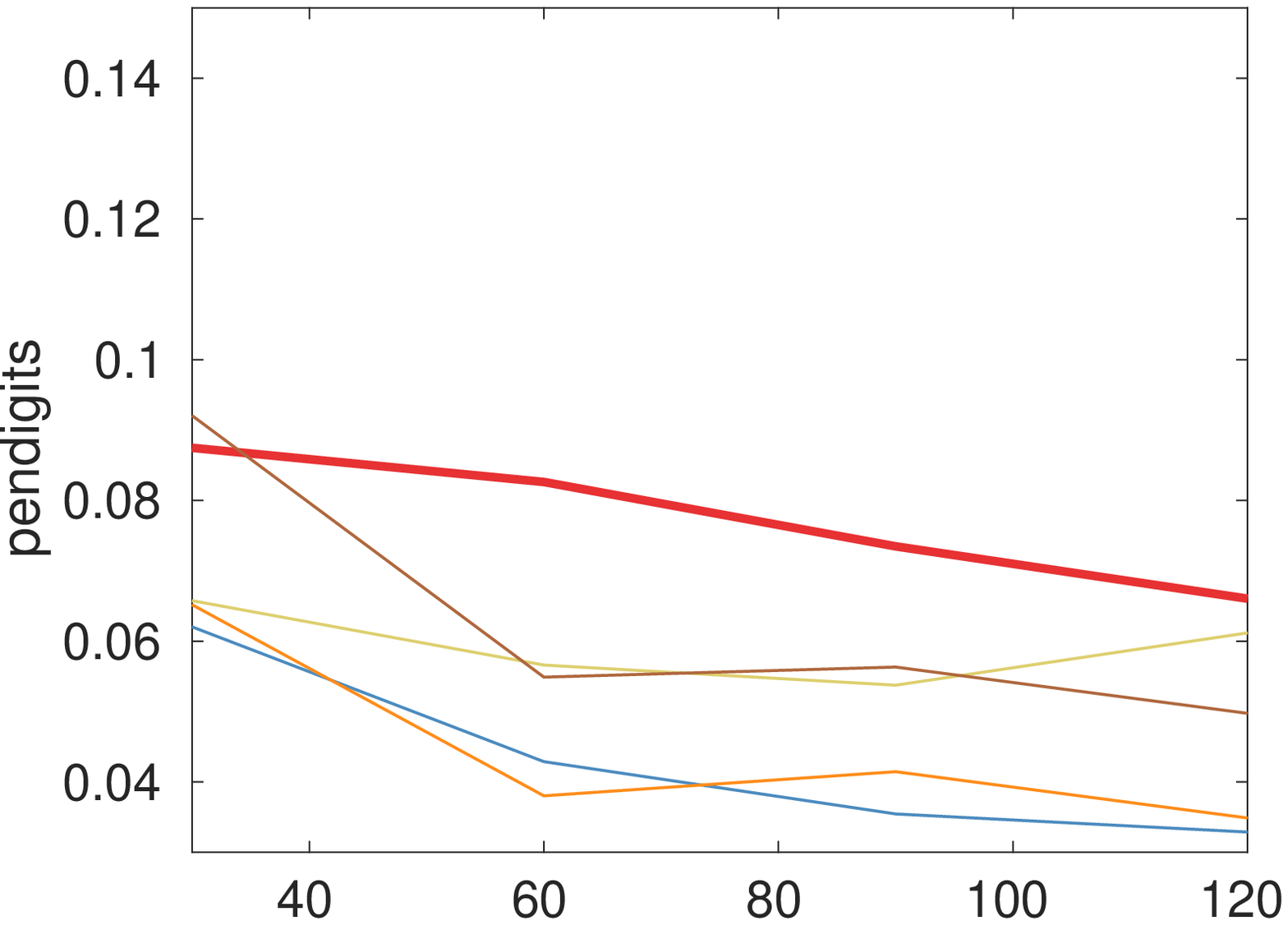}&
    \includegraphics*[width=0.5\linewidth]{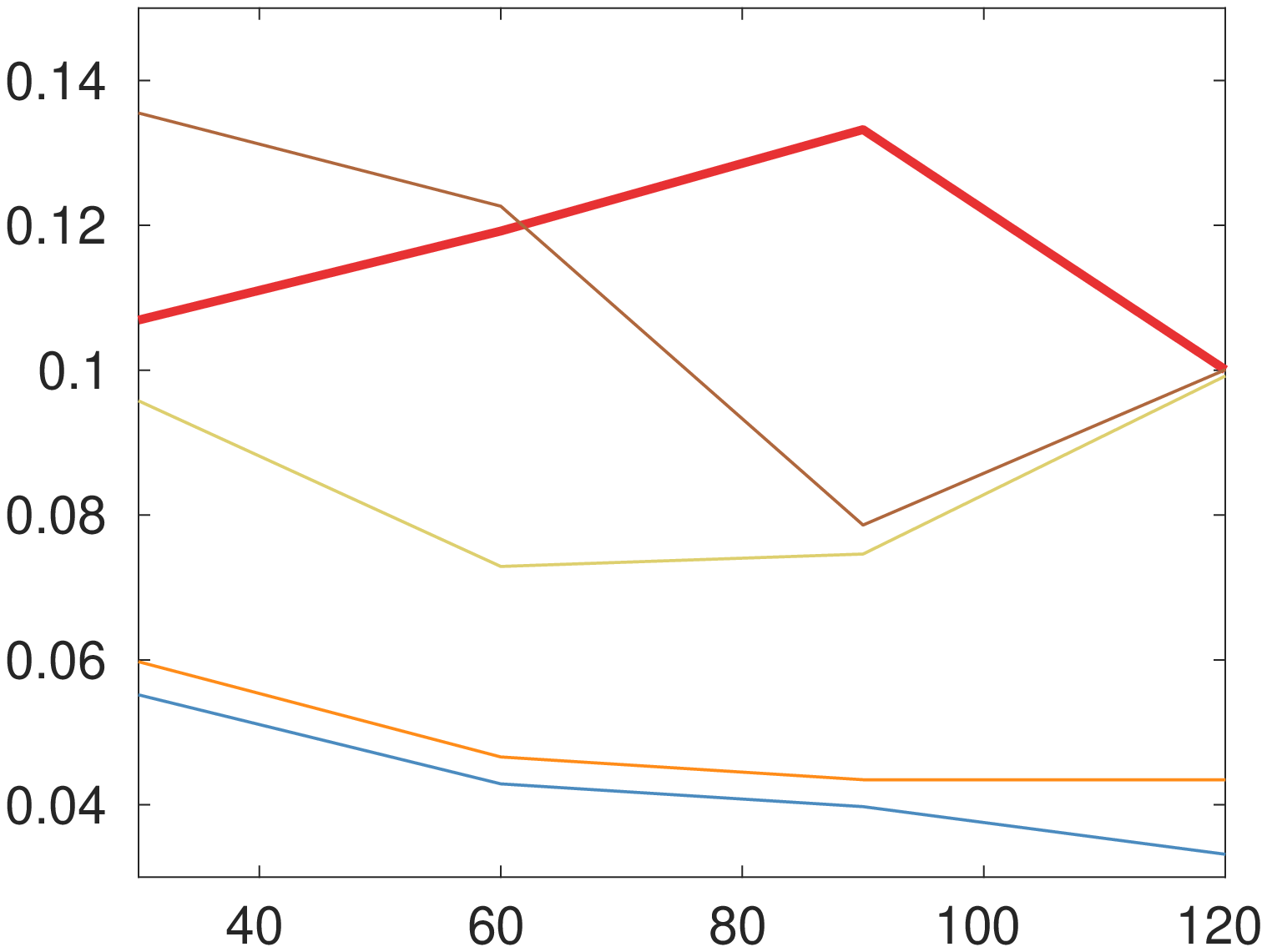}
    \\
  \end{tabular}
  \caption{Different attack techniques are presented for each dataset. Vertical axis presents error ratio over the testset and horizontal axis presents number of base models. The second setup was used for experiments.}
\label{fig:attacks-80}
\end{figure}

Figure \ref{fig:attacks-all2-80} presents all models for second setup.
\begin{figure}[!ht]
\centering
    \begin{tabular}{c@{}c@{}c@{}}
    $5\%$ attack & $10\%$ attack
    \\
    \includegraphics*[width=0.5\linewidth]{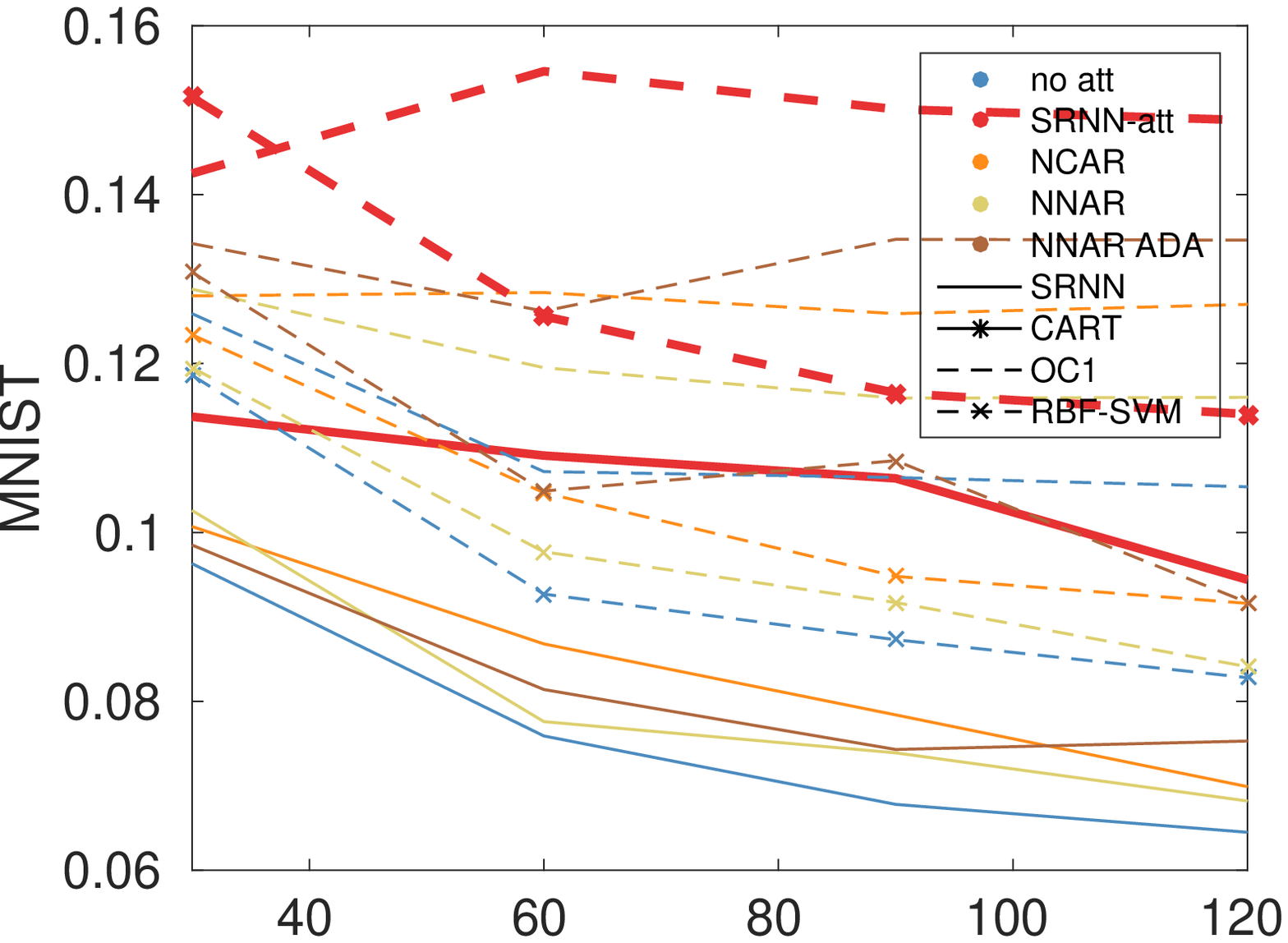}&
    \includegraphics*[width=0.5\linewidth]{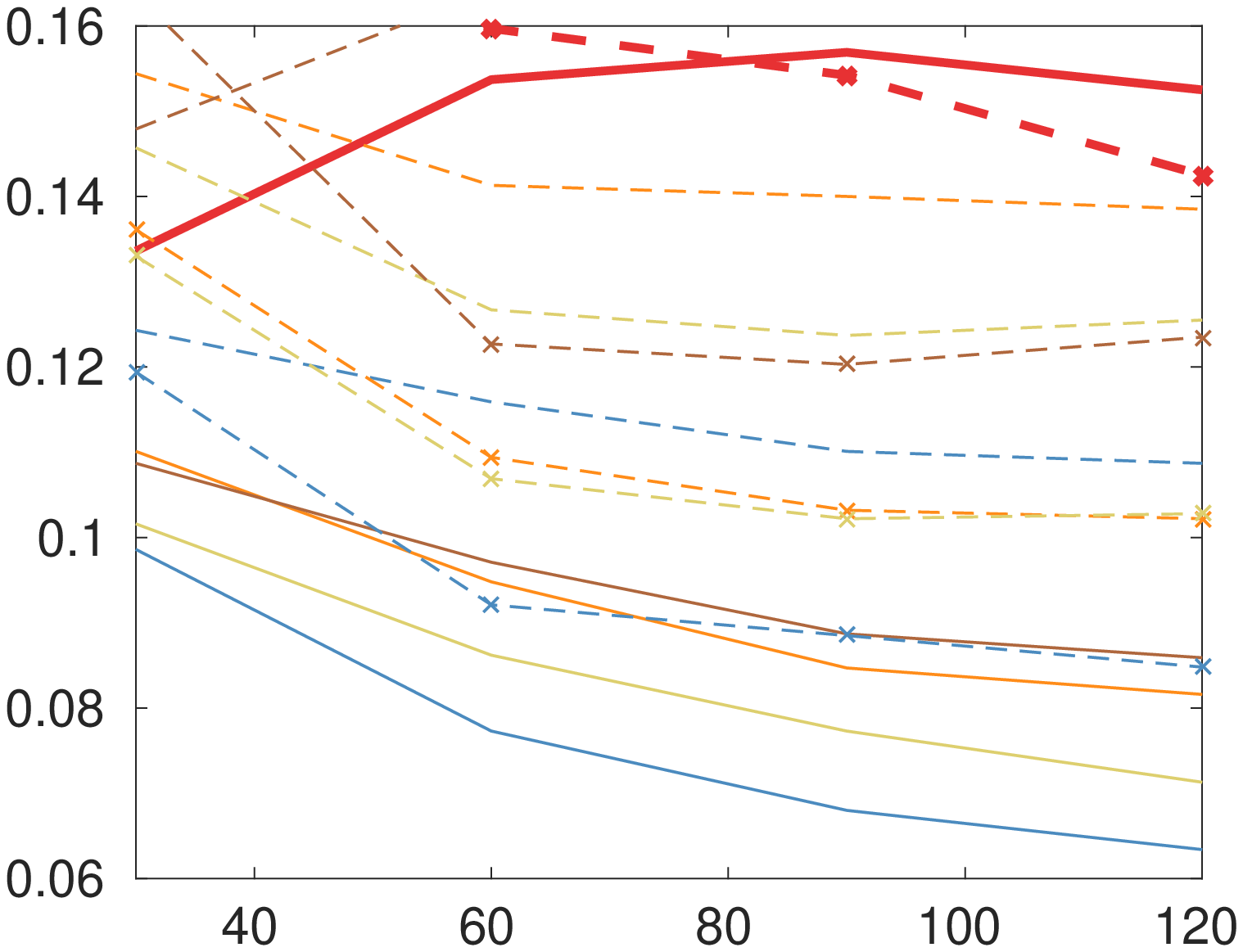}
    \\
    \includegraphics*[width=0.5\linewidth]{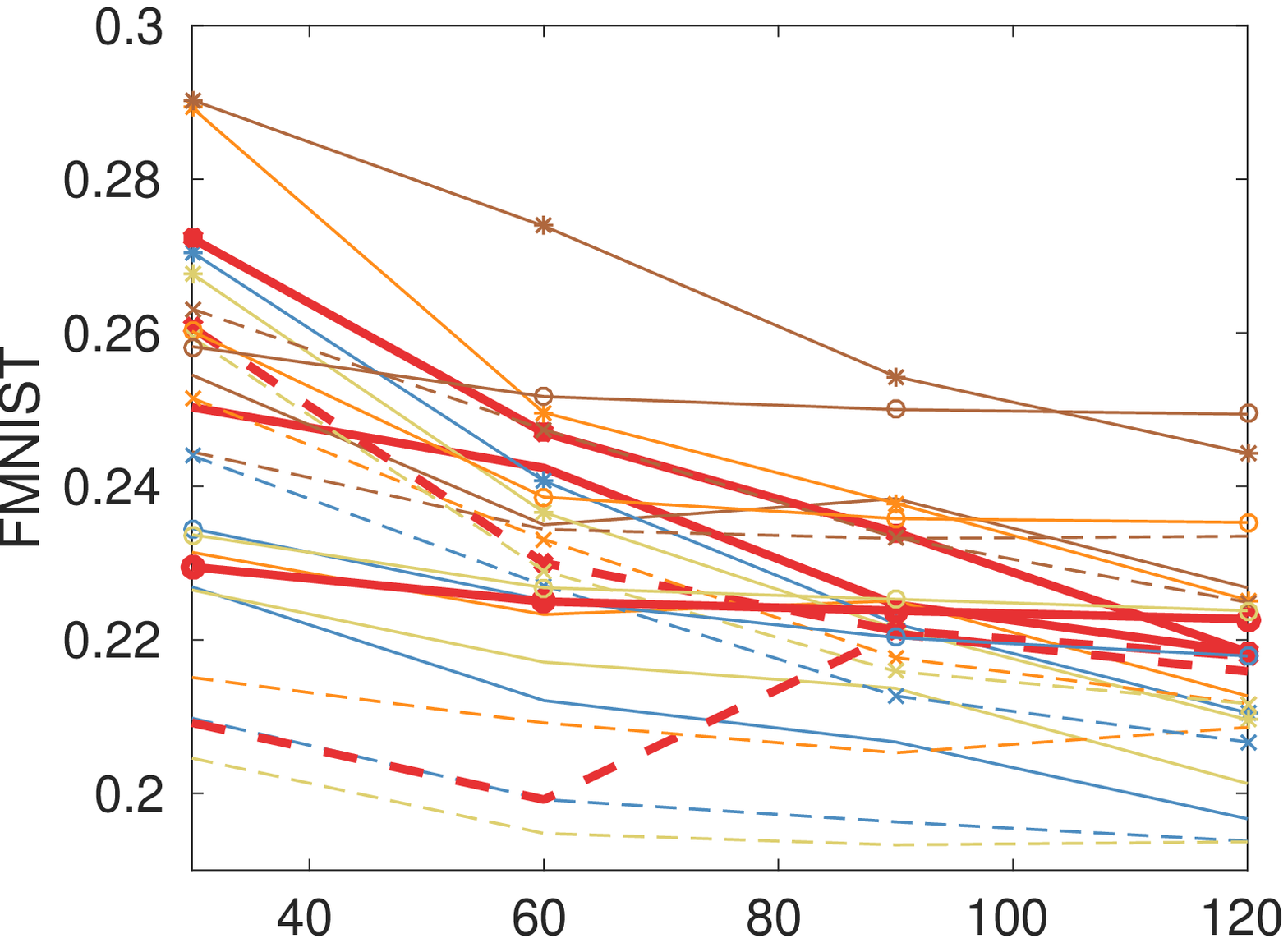}&
    \includegraphics*[width=0.5\linewidth]{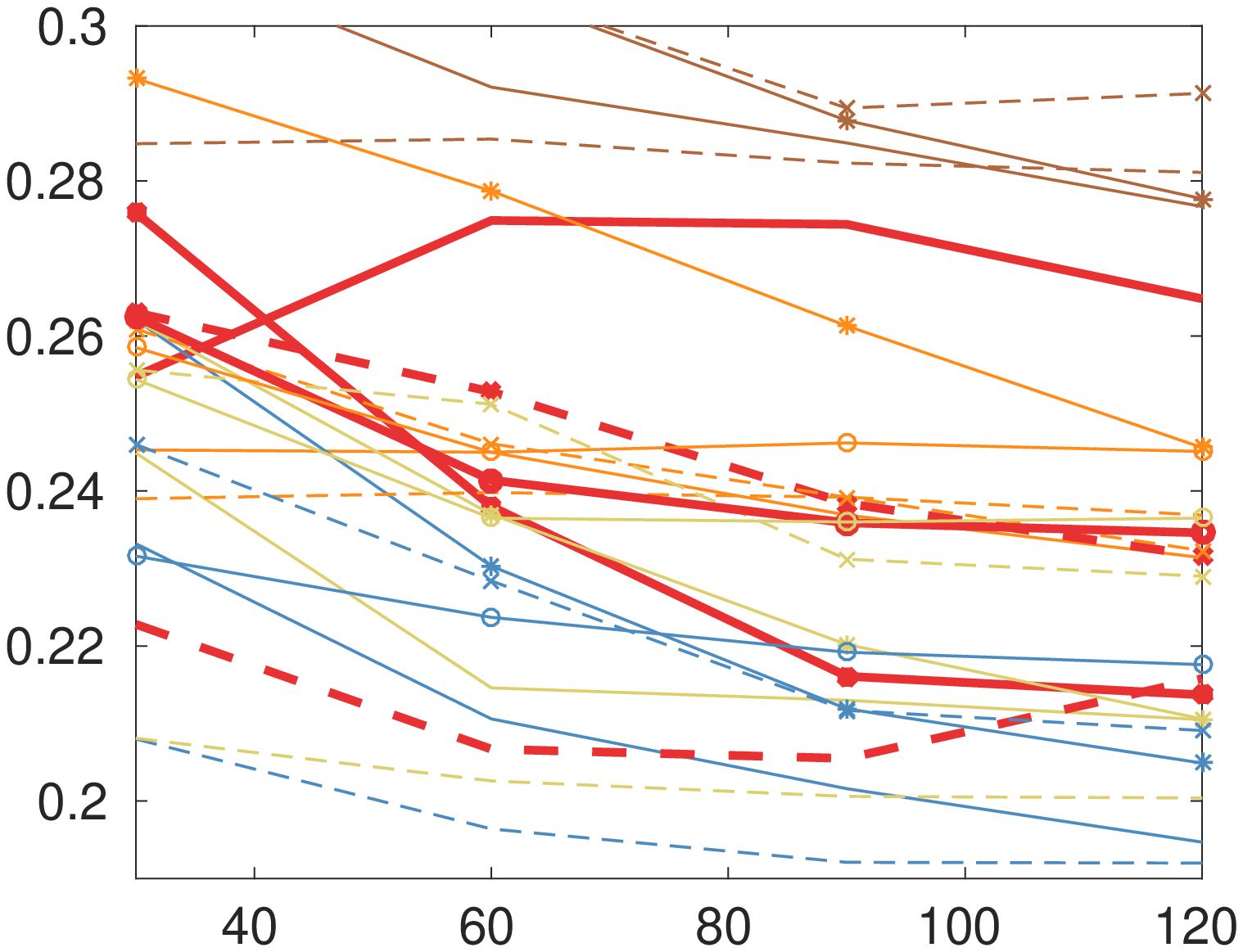}
    \\
    \includegraphics*[width=0.5\linewidth]{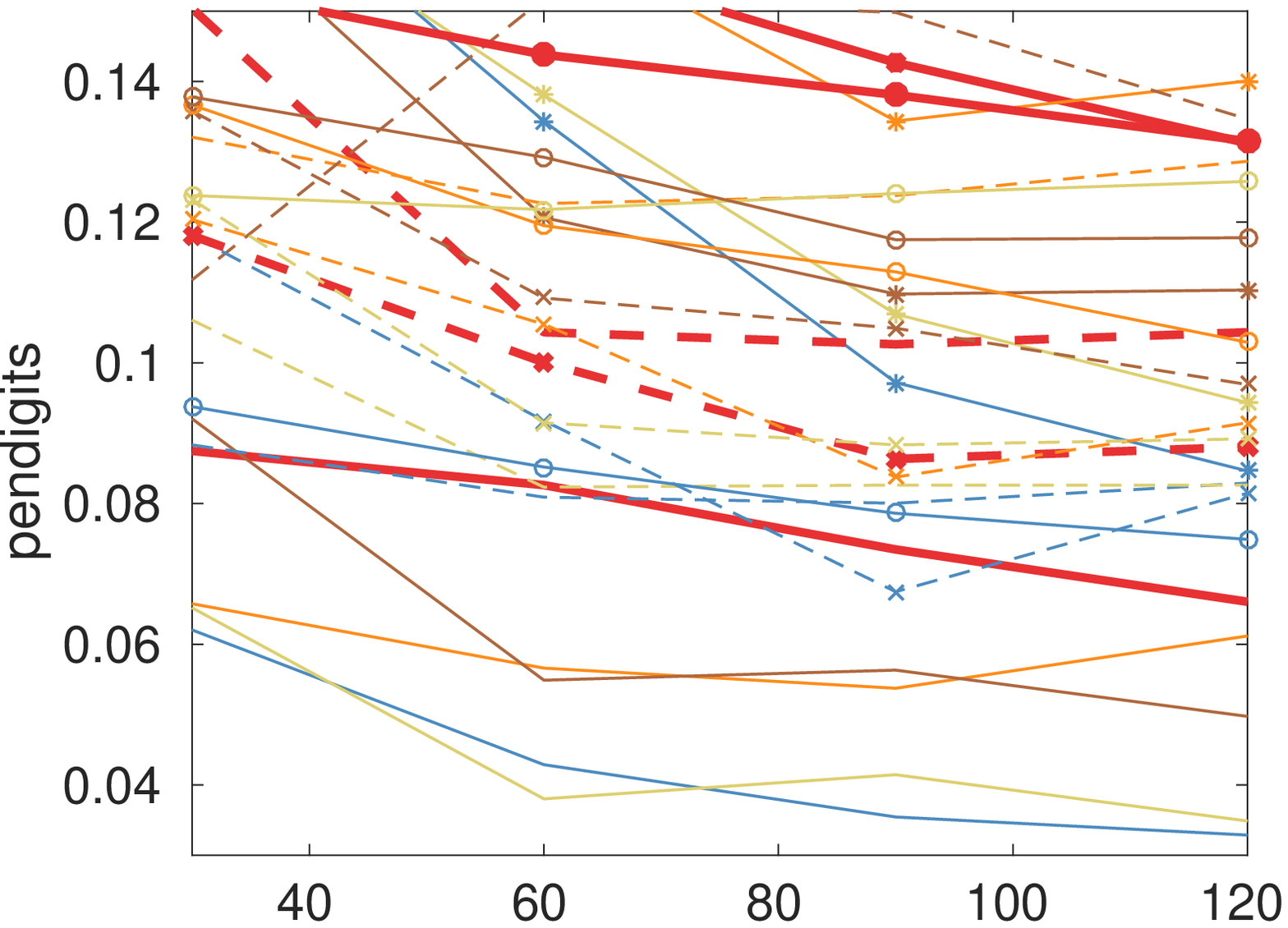}&
    \includegraphics*[width=0.5\linewidth]{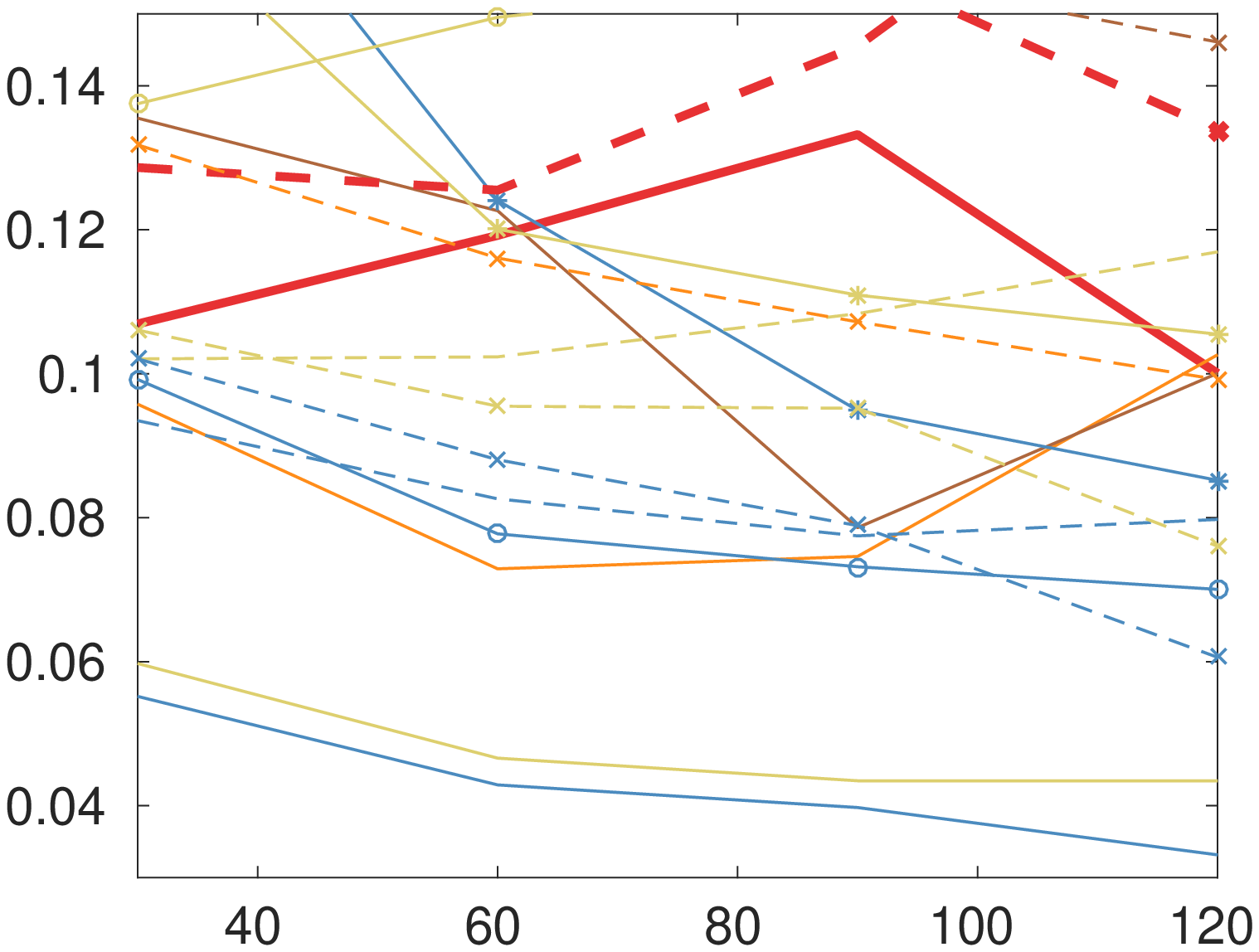}
    \\
  \end{tabular}
  \caption{Different attack techniques and models are presented for each dataset. Colors and curve shapes present the attack techniques and models, respectively. model Vertical axis presents error ratio over the testset and horizontal axis presents number of base models. For this figure, second setup was used.}
\label{fig:attacks-all2-80}
\end{figure}
\subsection{Defense Experiments}
In the experiments of this section, we further investigated the resilience of the proposed defense technique against SRNN-att with $60$ centroids. Figure \ref{fig:defs-20-60} presents results of this experiments. $20\%$ of each dataset was used as the trainset.
\begin{figure}[!ht]
\centering
    \begin{tabular}{c@{}c@{}c@{}}
    $5\%$ attack & $10\%$ attack
    \\
    \includegraphics*[width=0.5\linewidth]{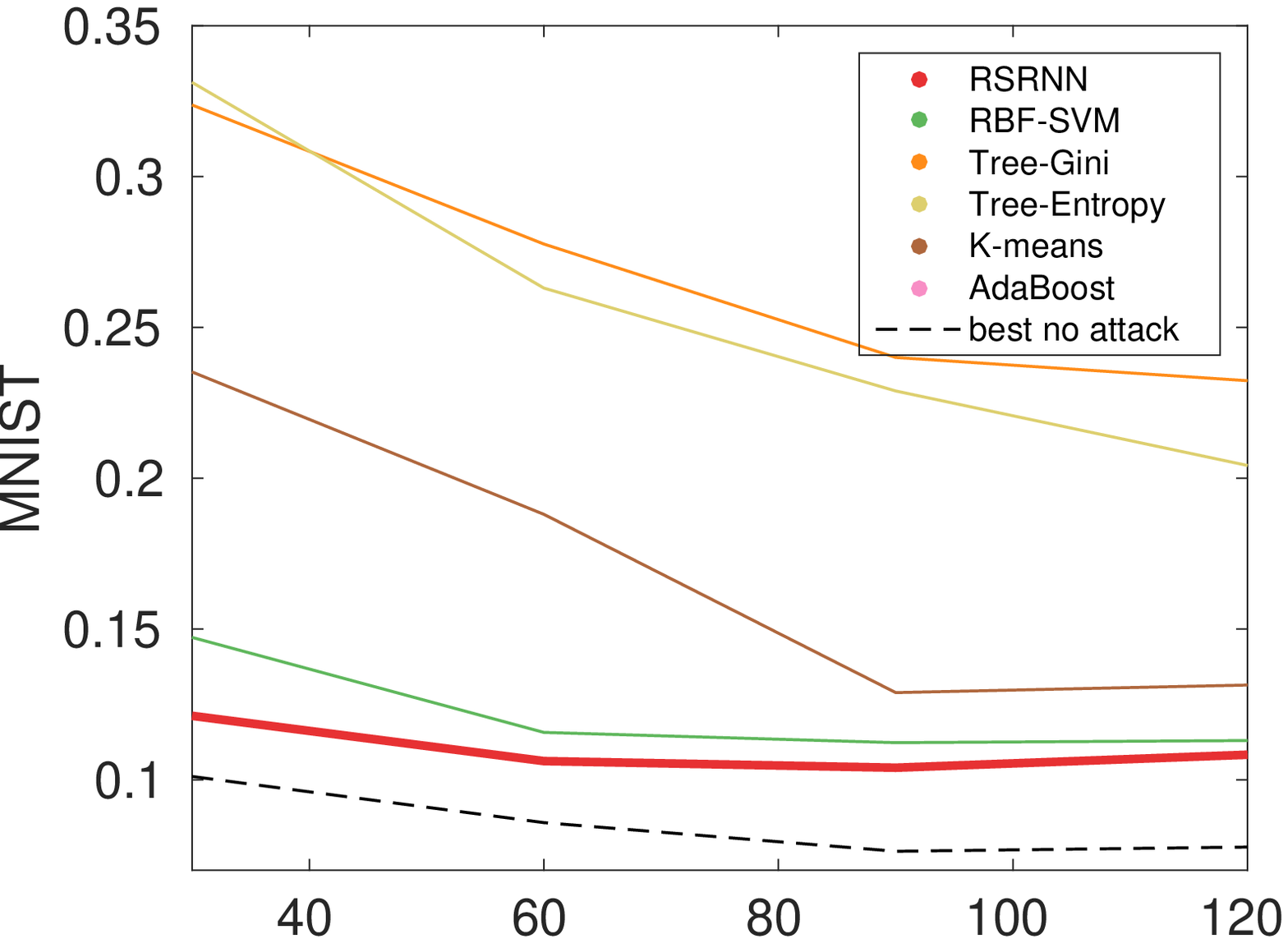}&
    \includegraphics*[width=0.5\linewidth]{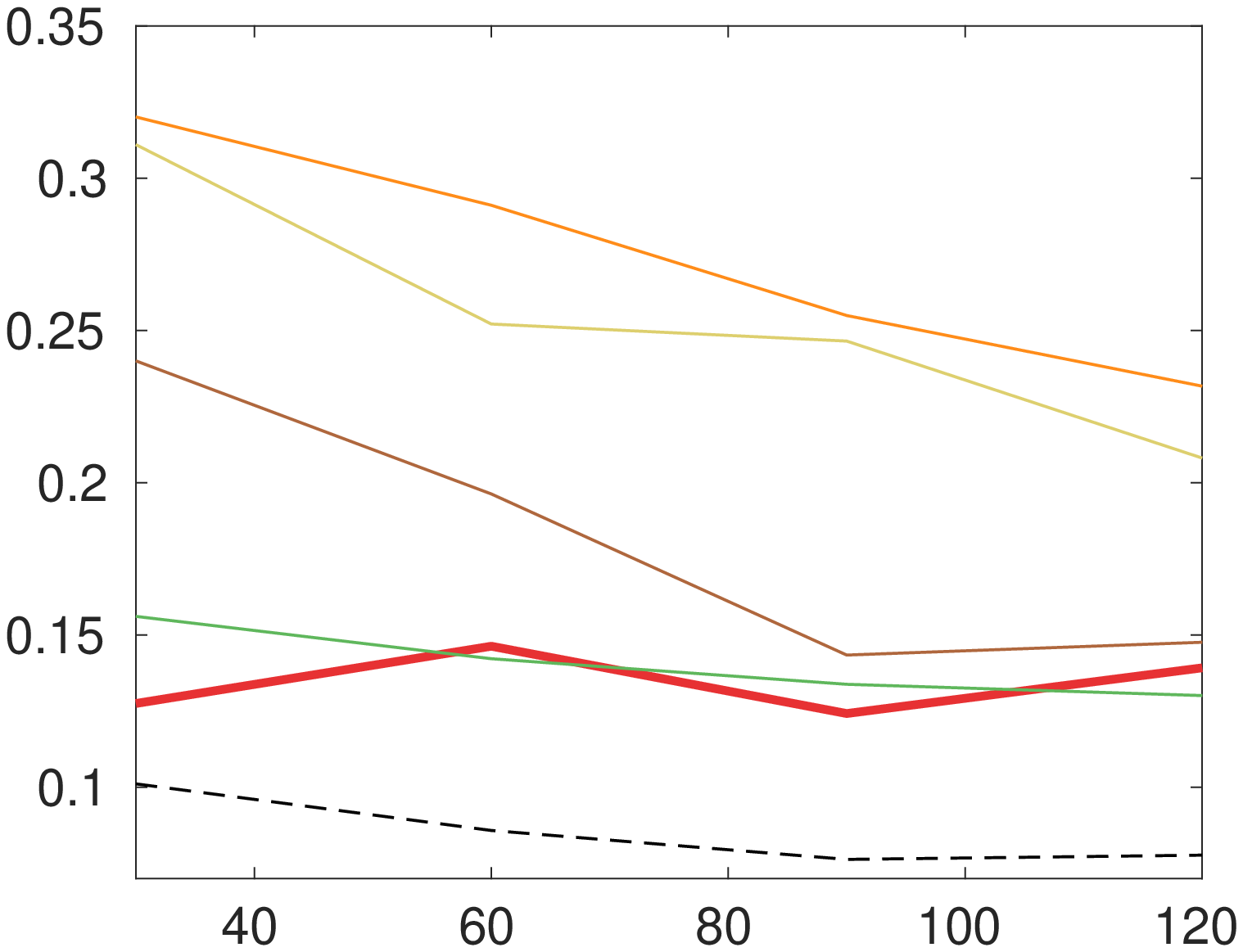}
    \\
    \includegraphics*[width=0.5\linewidth]{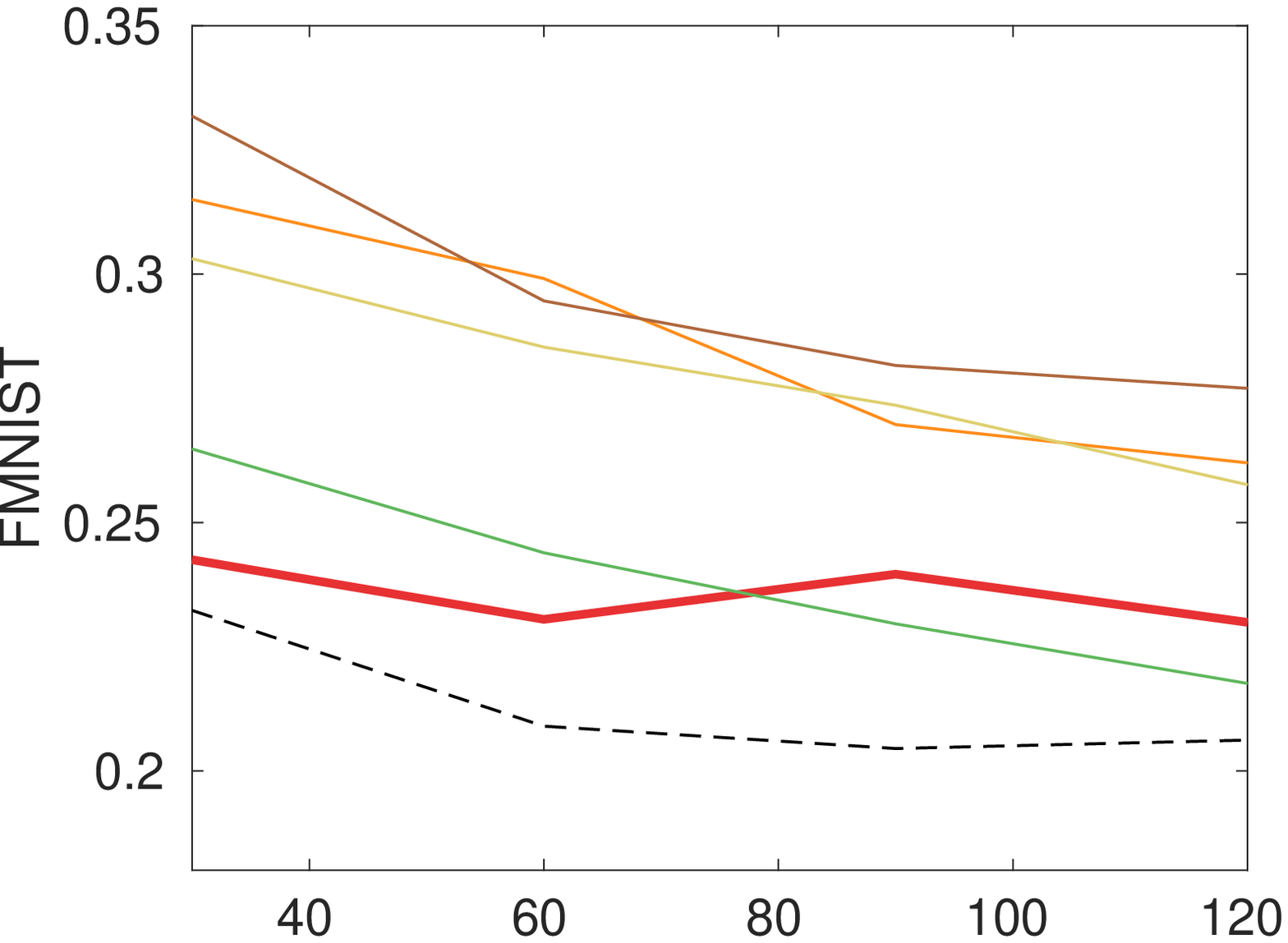}&
    \includegraphics*[width=0.5\linewidth]{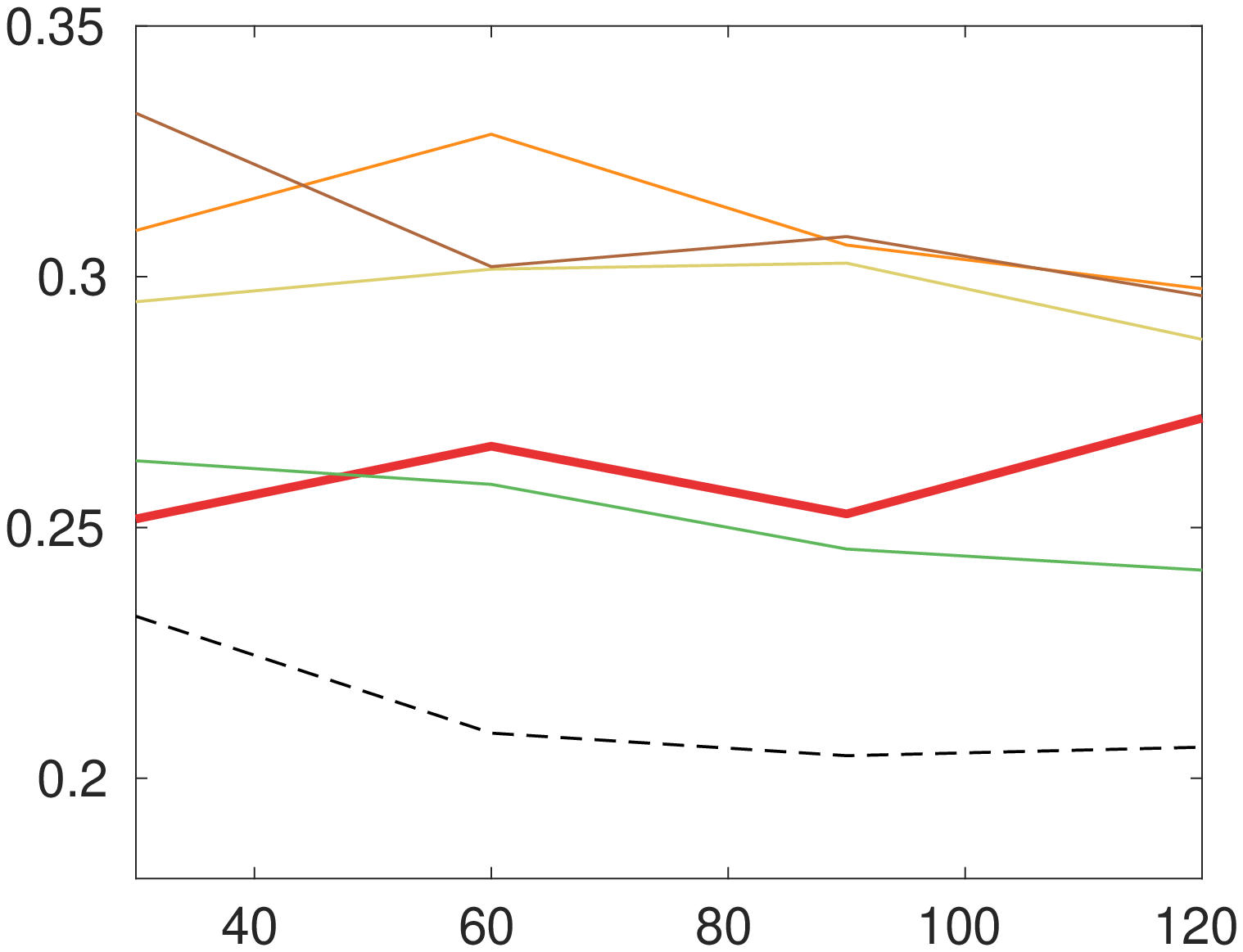}
    \\
    \includegraphics*[width=0.5\linewidth]{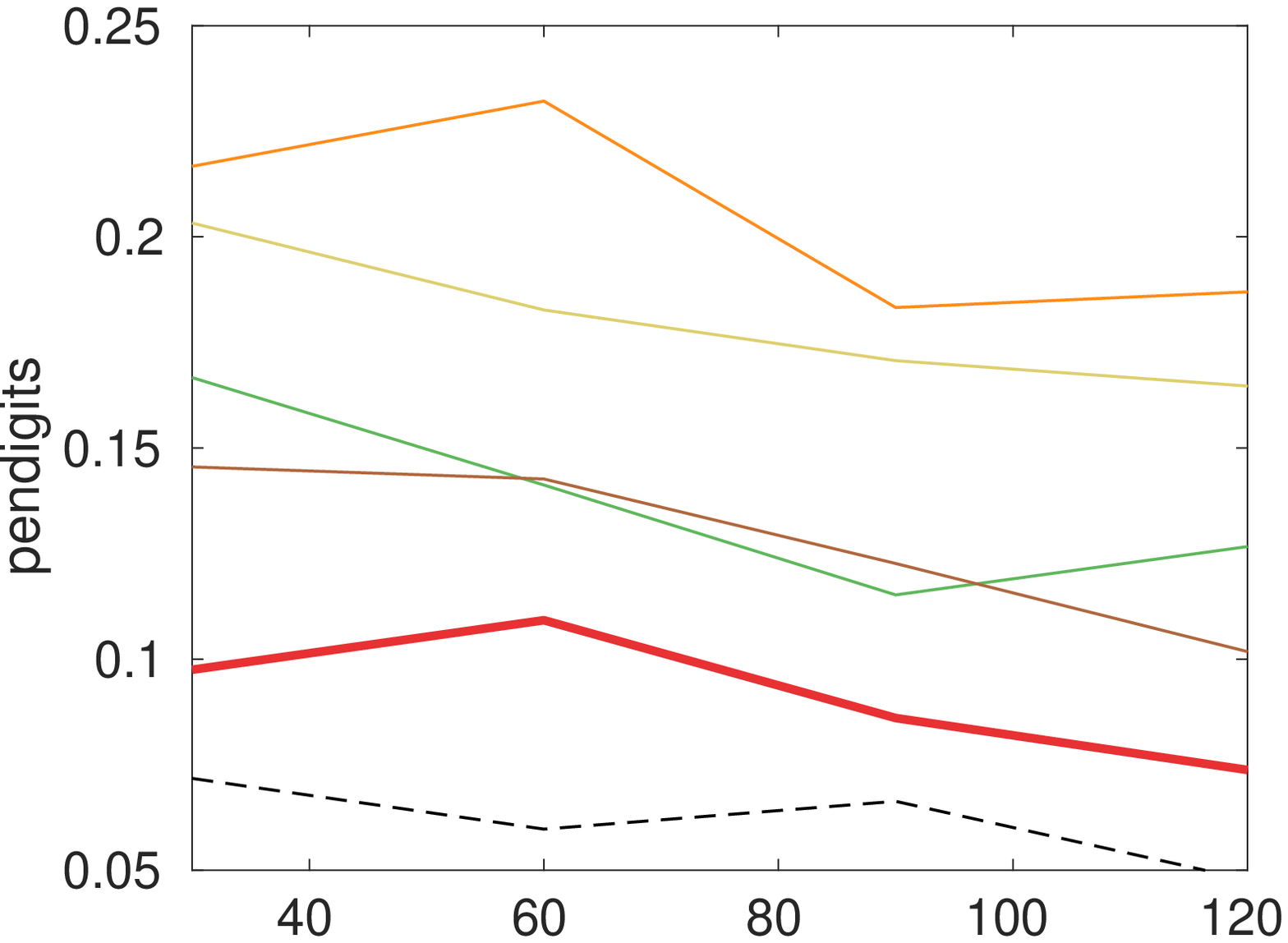}&
    \includegraphics*[width=0.5\linewidth]{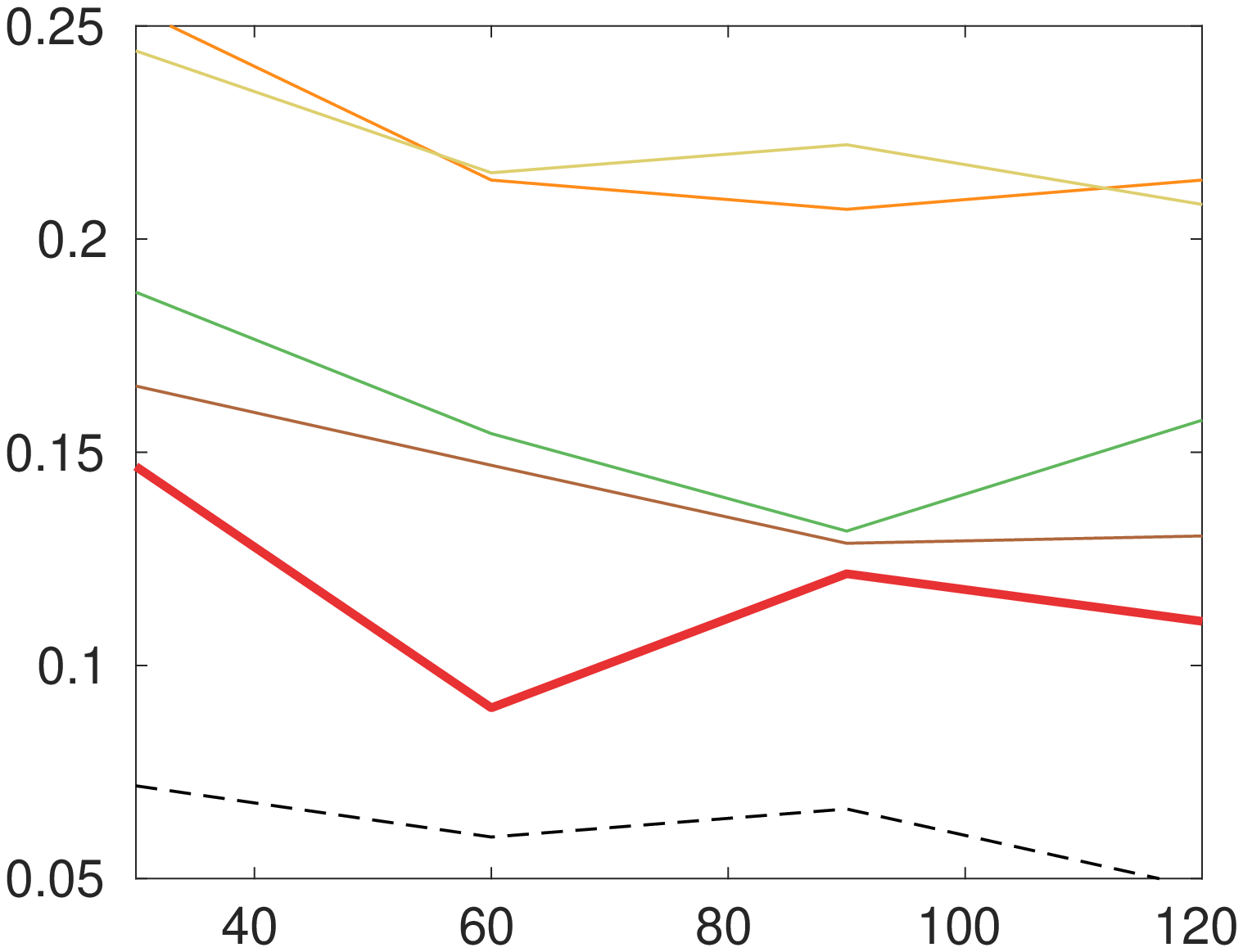}
    \\
    \includegraphics*[width=0.5\linewidth]{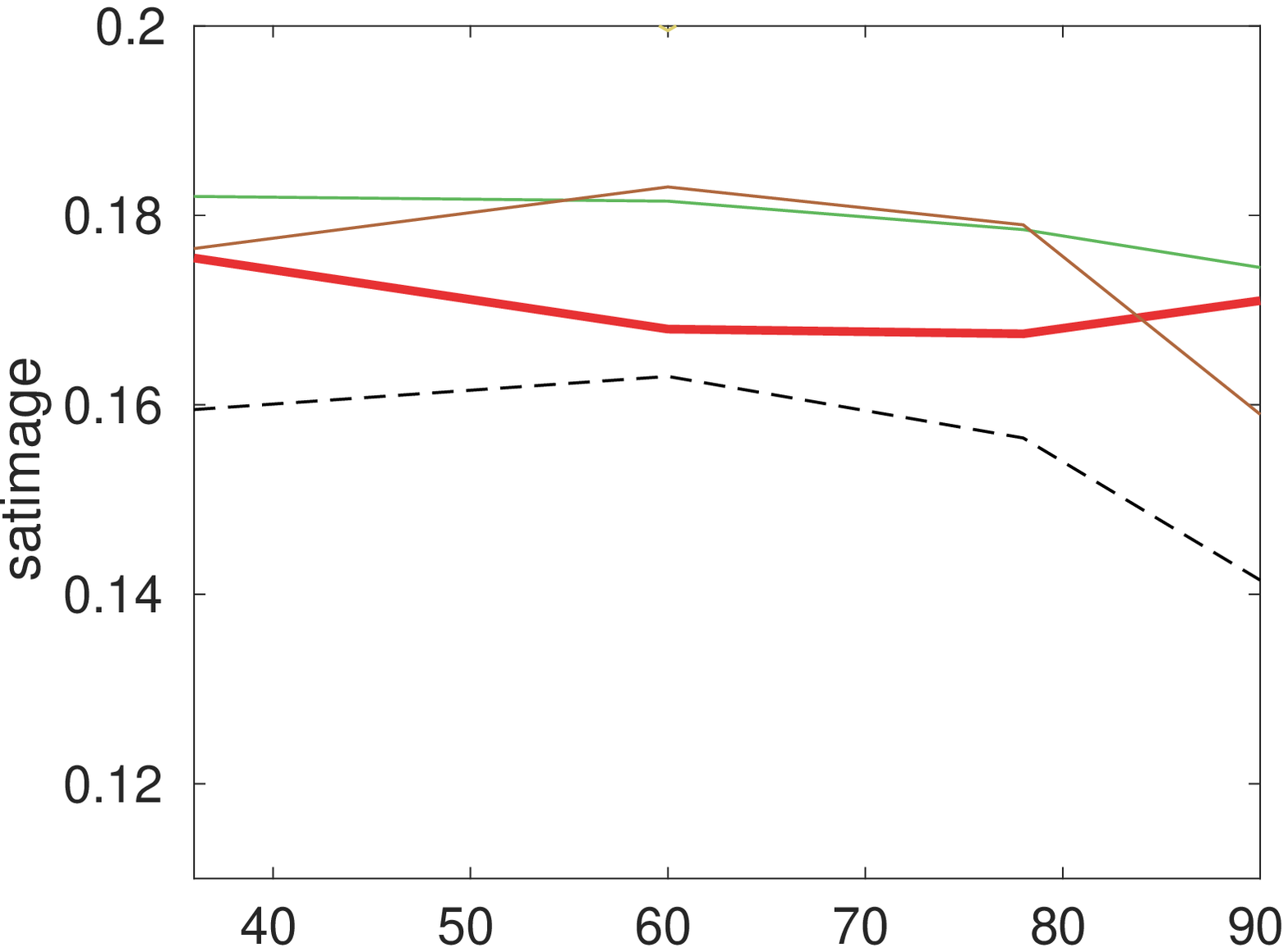}&
    \includegraphics*[width=0.5\linewidth]{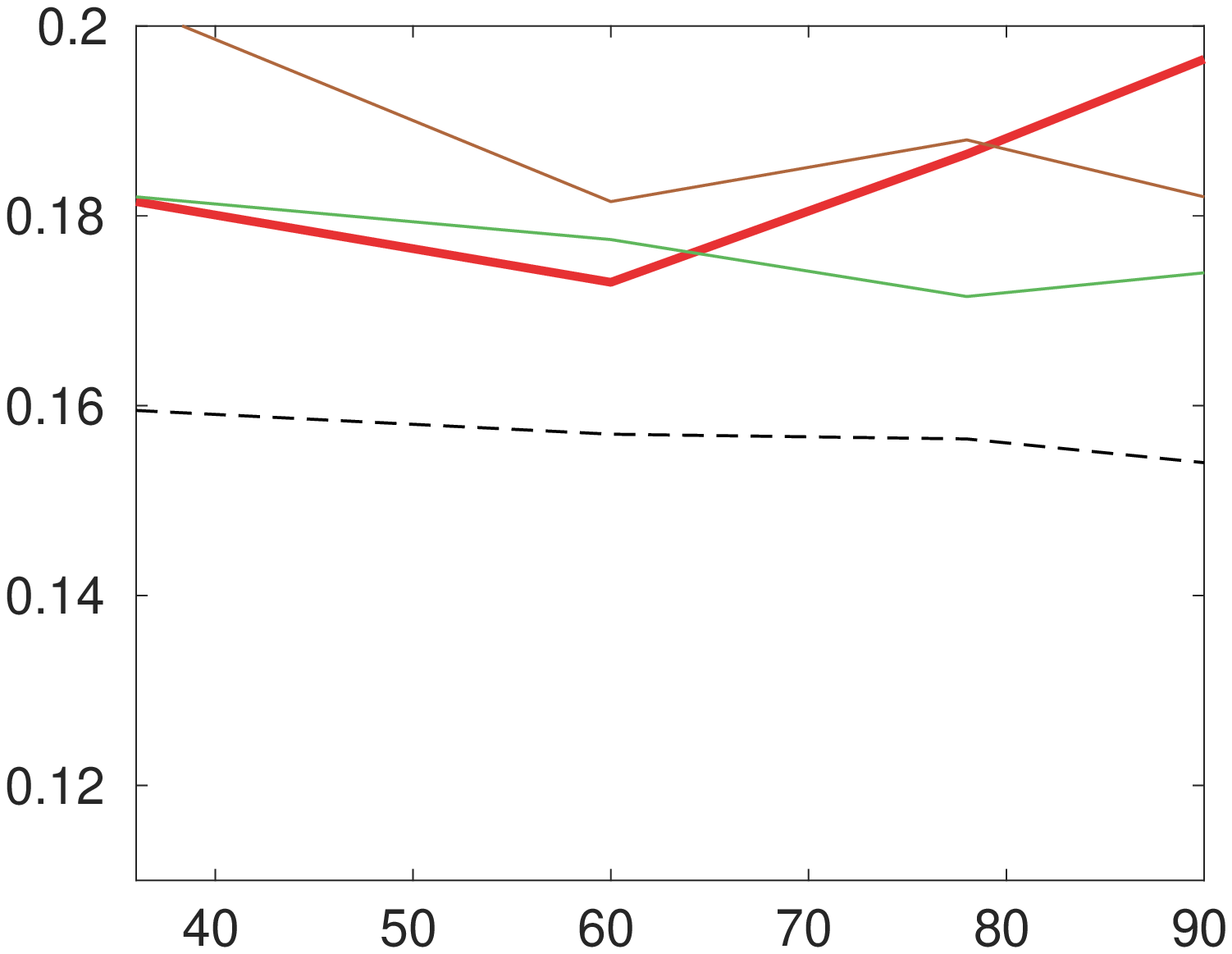}
    \\
  \end{tabular}
  \caption{Different defenses against SRNN-att. SRNN-att model was trained with 60 centroids. Vertical axis shows test error ratio and horizontal axis represents number of basis models. First setup for datasets was used.}
\label{fig:defs-20-60}
\end{figure}

Figure \ref{fig:defs-20-60} presents that the defense technique was able to achieve lower or as good as best of other models.

Figure \ref{fig:defs-80-30} presents experiments with $80\%$ of datasets used as trainset and $30$ centroids for SRNN-att.

\begin{figure}[!ht]
\centering
    \begin{tabular}{c@{}c@{}c@{}}
    $5\%$ attack & $10\%$ attack
    \\
    \includegraphics*[width=0.5\linewidth]{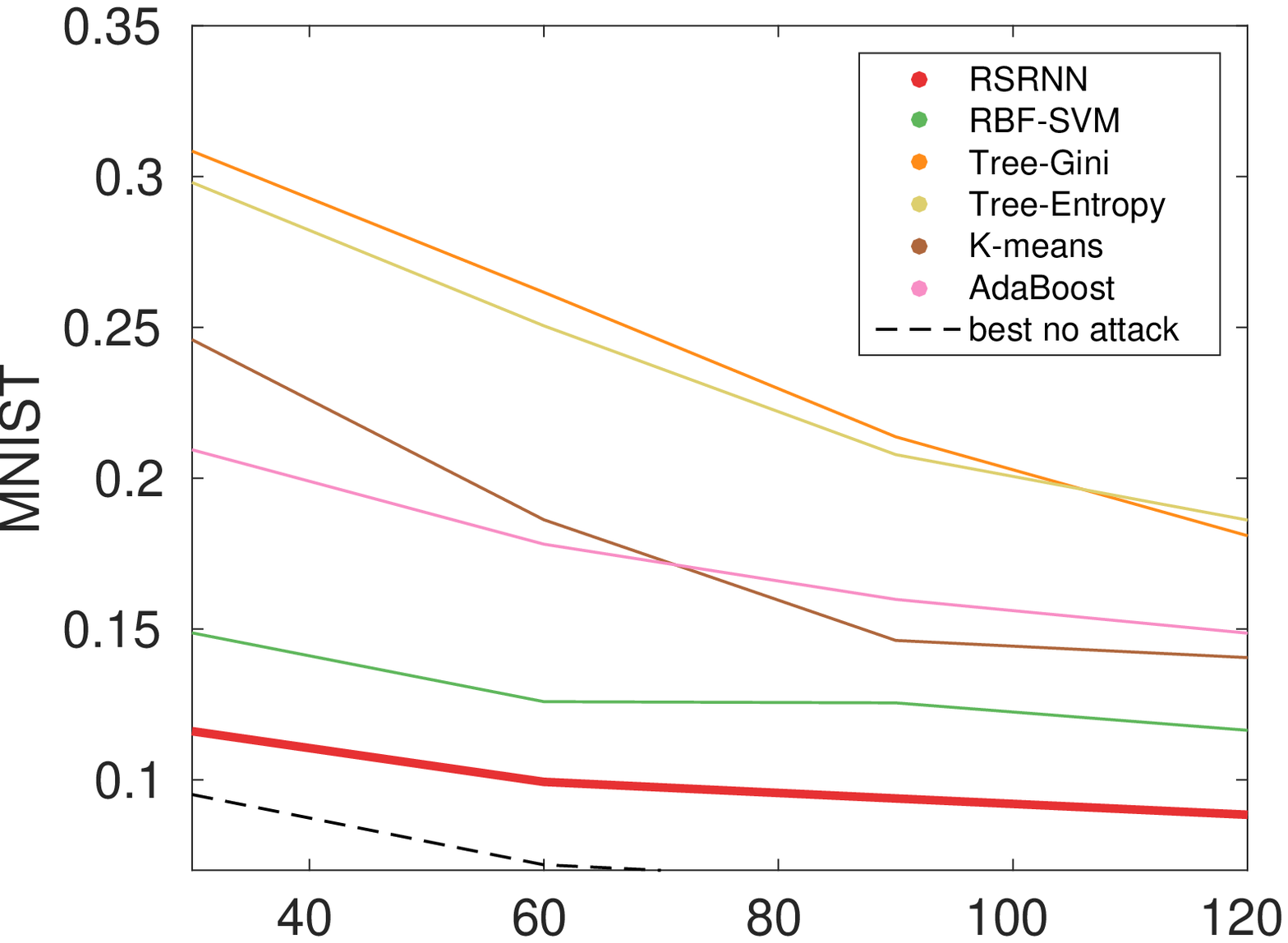}&
    \includegraphics*[width=0.5\linewidth]{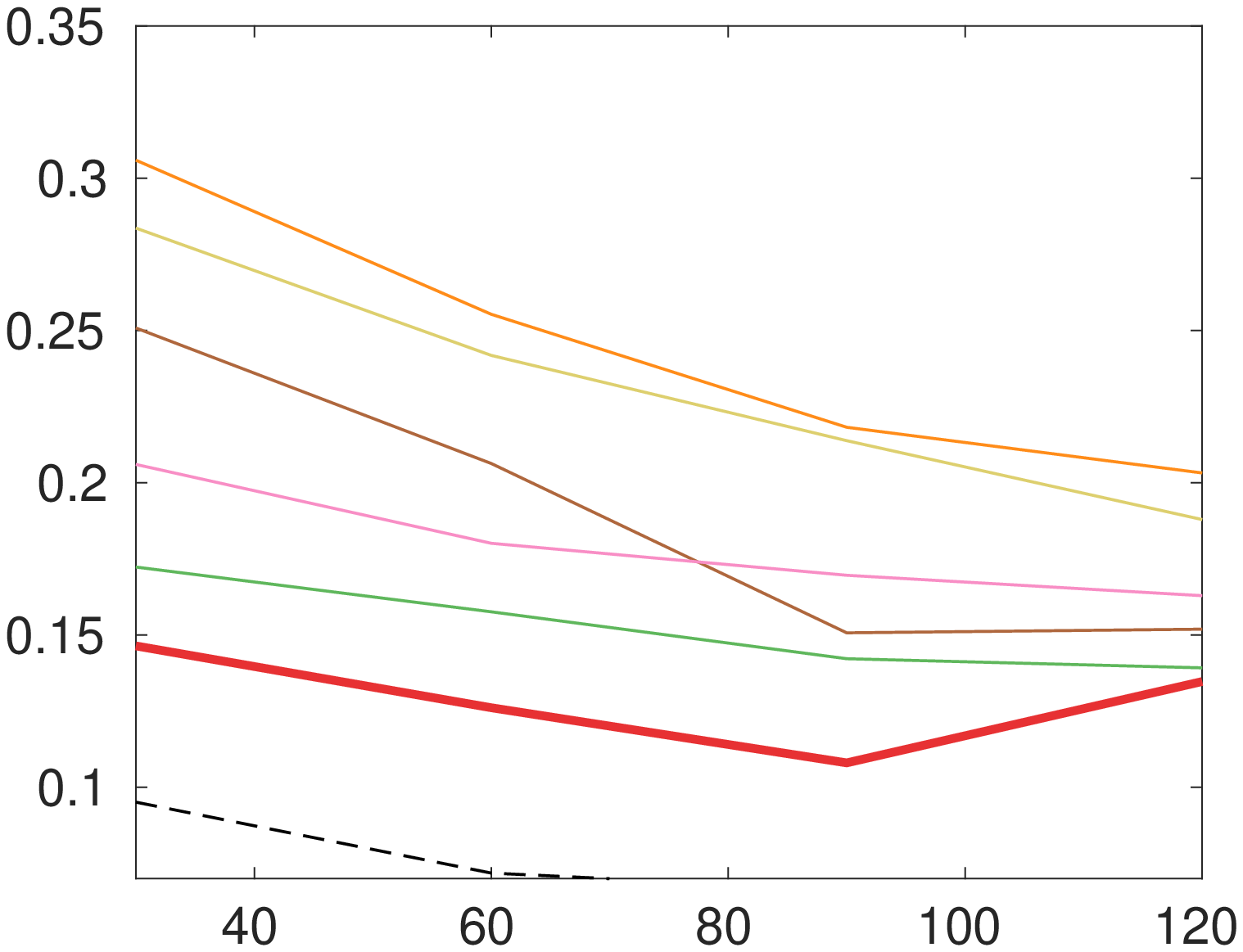}
    \\
    \includegraphics*[width=0.5\linewidth]{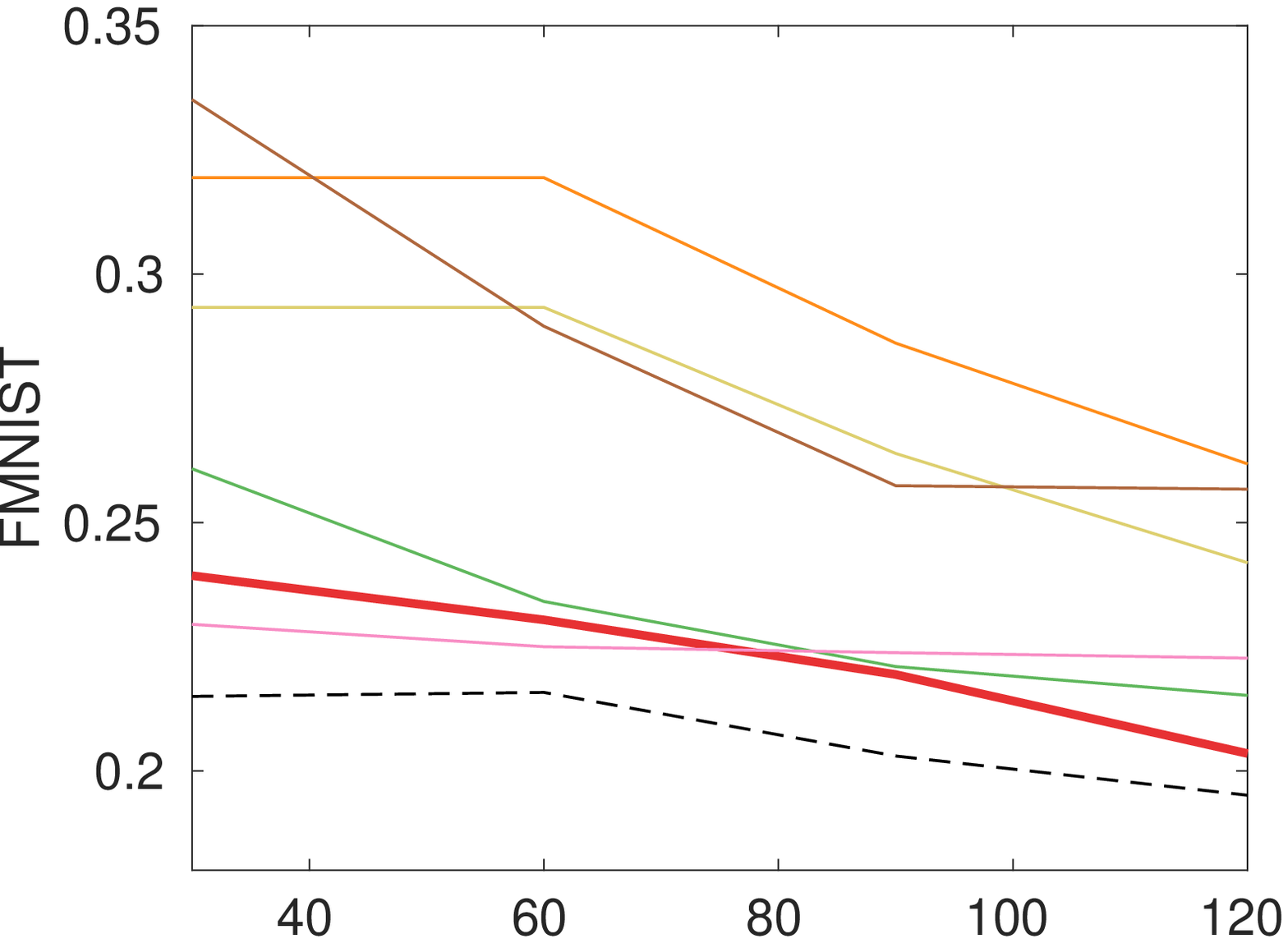}&
    \includegraphics*[width=0.5\linewidth]{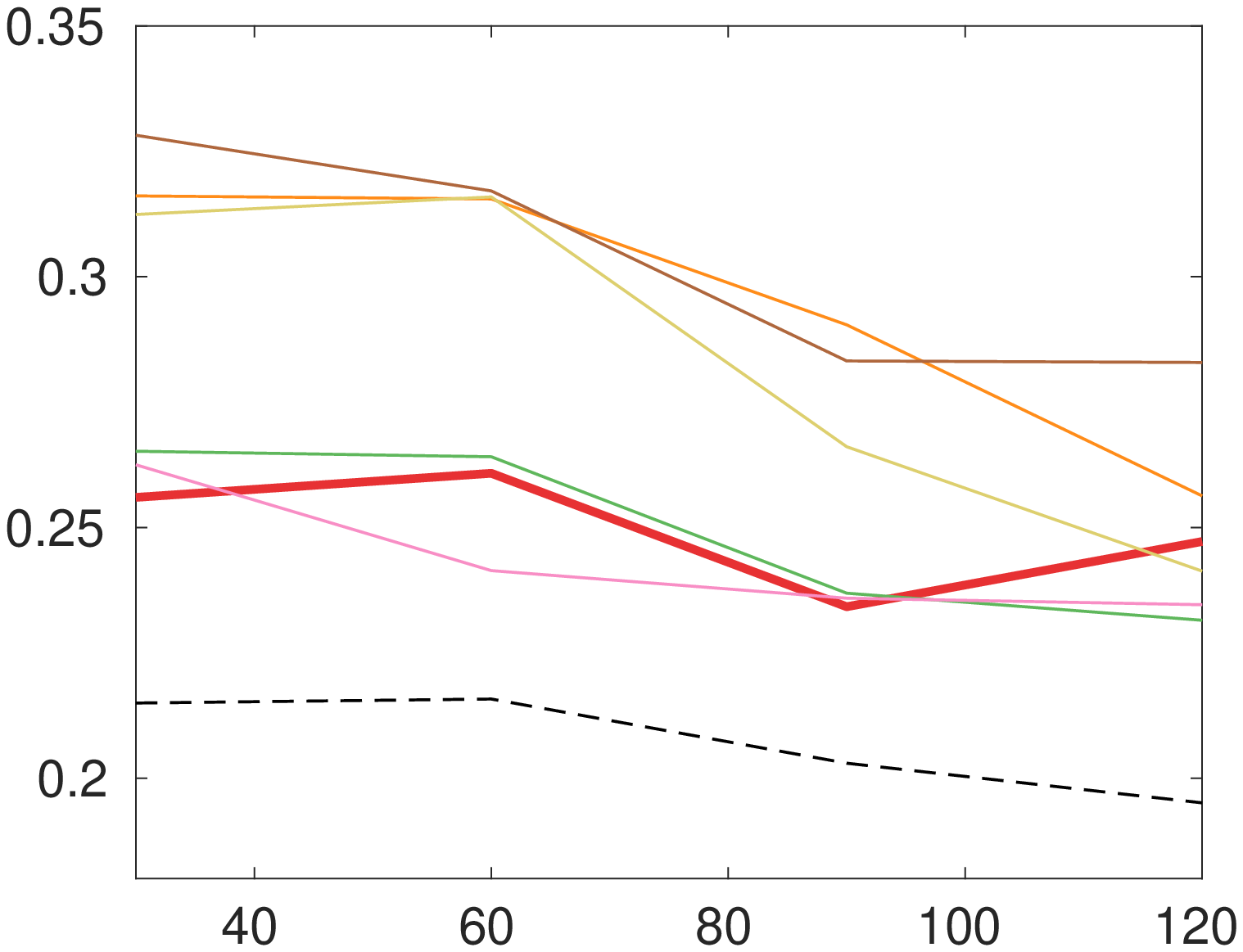}
    \\
    \includegraphics*[width=0.5\linewidth]{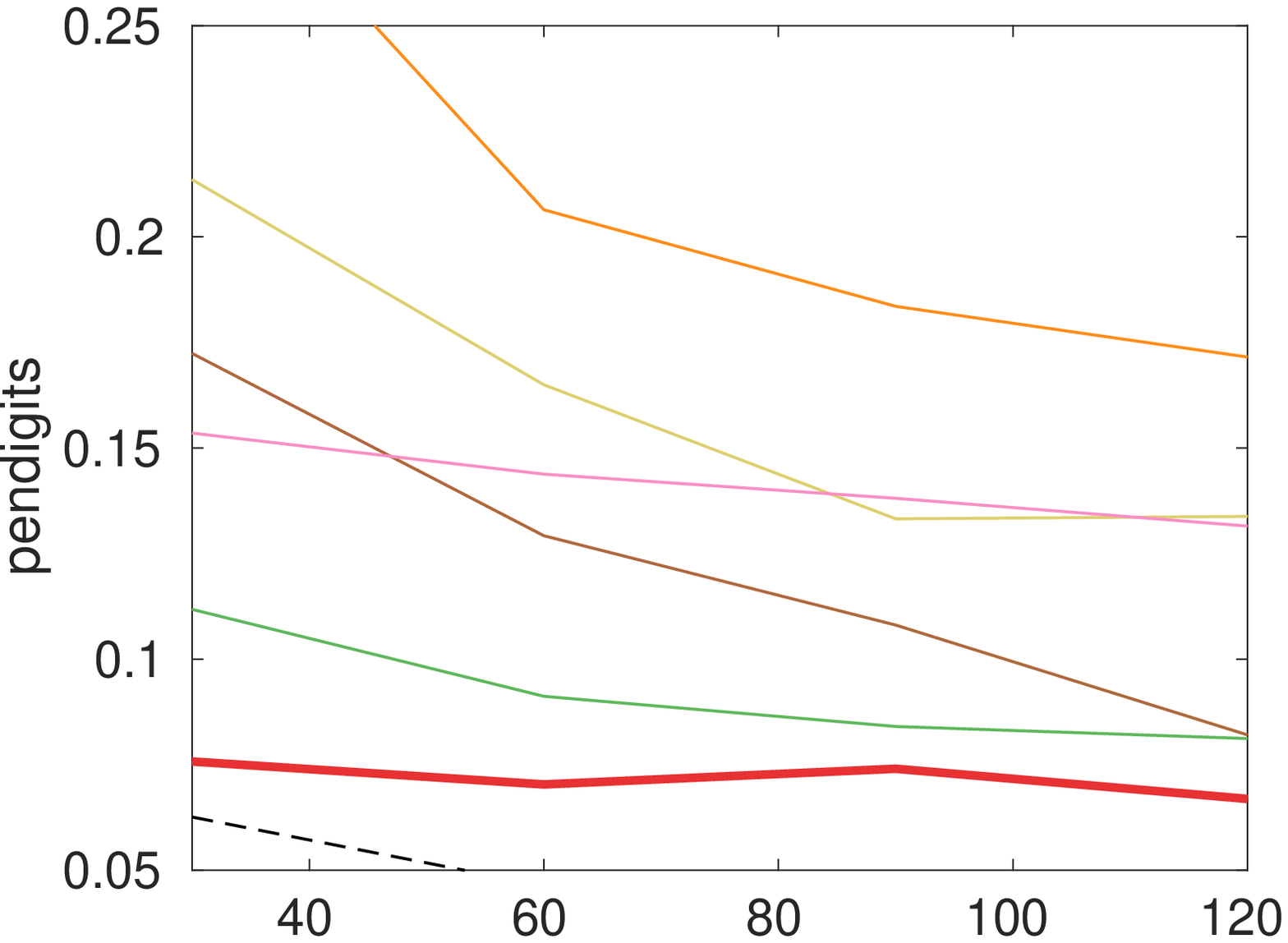}&
    \includegraphics*[width=0.5\linewidth]{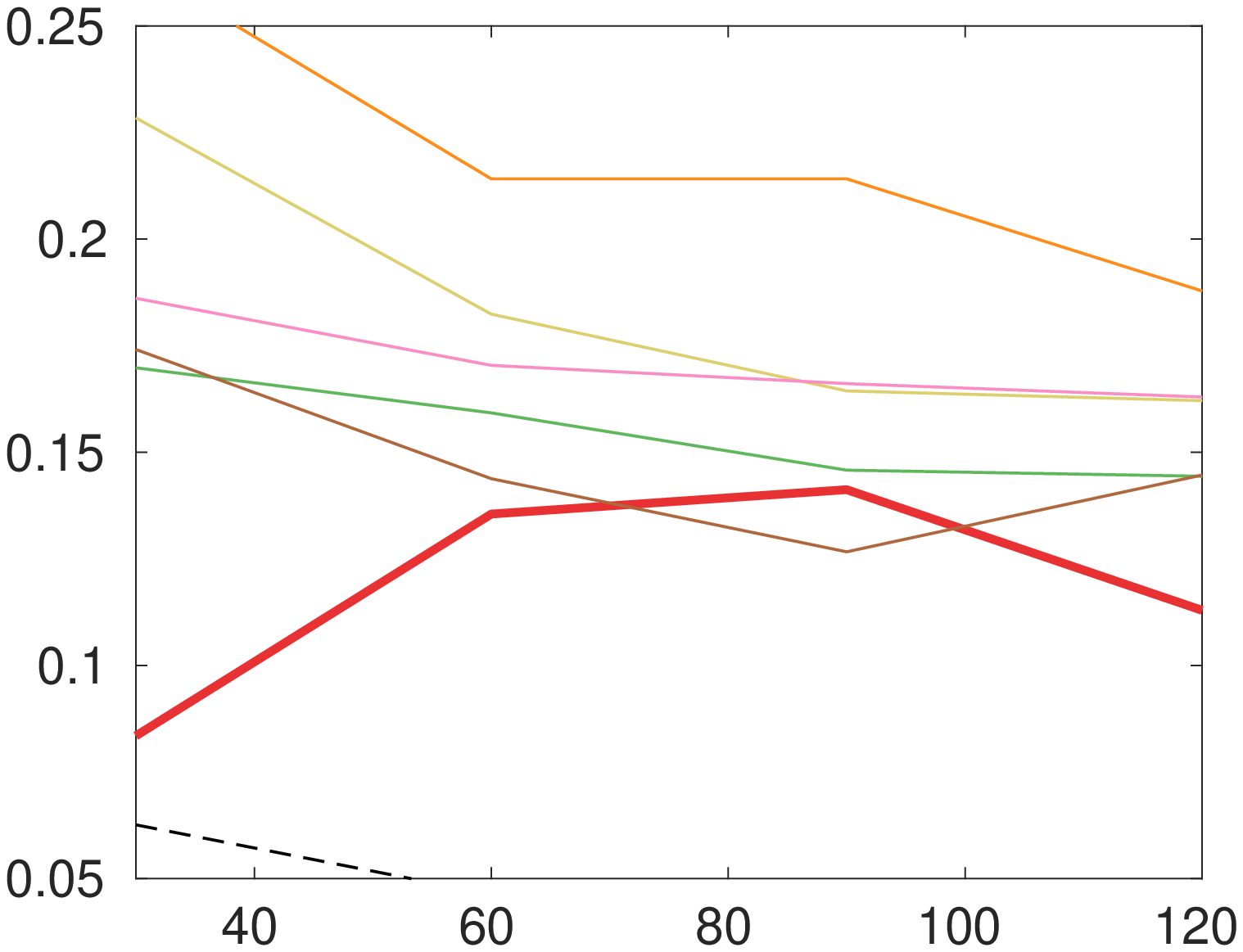}
    \\
  \end{tabular}
  \caption{Different defenses against SRNN-att. SRNN-att model was trained with 30 centroids. Vertical axis shows test error ratio and horizontal axis represents number of basis models. Second setup for datasets was used.}
\label{fig:defs-80-30}
\end{figure}
In experiments of figure \ref{fig:defs-80-30}, RSRNN was able to outperform other defense techniques by having smaller test error.
\subsection{Malicious Sample Detection}
As explained in the pruning sections of the main paper and supplementary materials, the algorithm is capable of detecting the malicious samples. In this part, detection ratio and true positive (TP) ratio of RSRNN are shown. Figure \ref{fig:TP-20-30} presents the ratio of malicious found samples and TP for figure 3 of the main paper. Figures \ref{fig:TP-80-30} and \ref{fig:TP-20-60} present the ratio of found malicious samples and TP for figures \ref{fig:defs-80-30} and \ref{fig:defs-20-60}. It is observable that the RSRNN was able to find up to $0.7$ of malicious samples with a true positive ratio of $0.5$.
\begin{figure}[!ht]
\centering
    \begin{tabular}{c@{}c@{}c@{}}
    ratios with $5\%$ attack & ratios with $10\%$ attack
    \\
    \includegraphics*[width=0.5\linewidth]{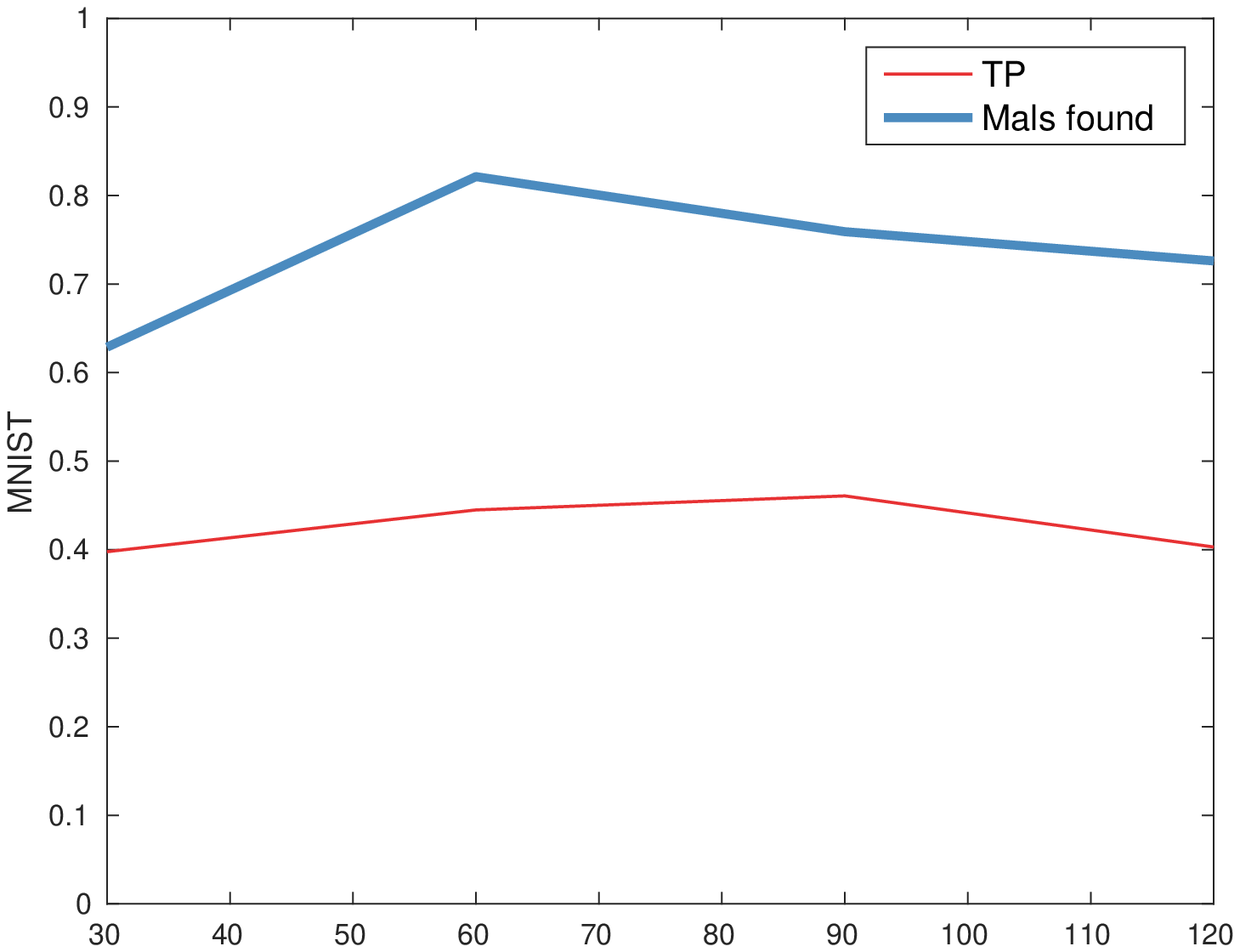}&
    \includegraphics*[width=0.5\linewidth]{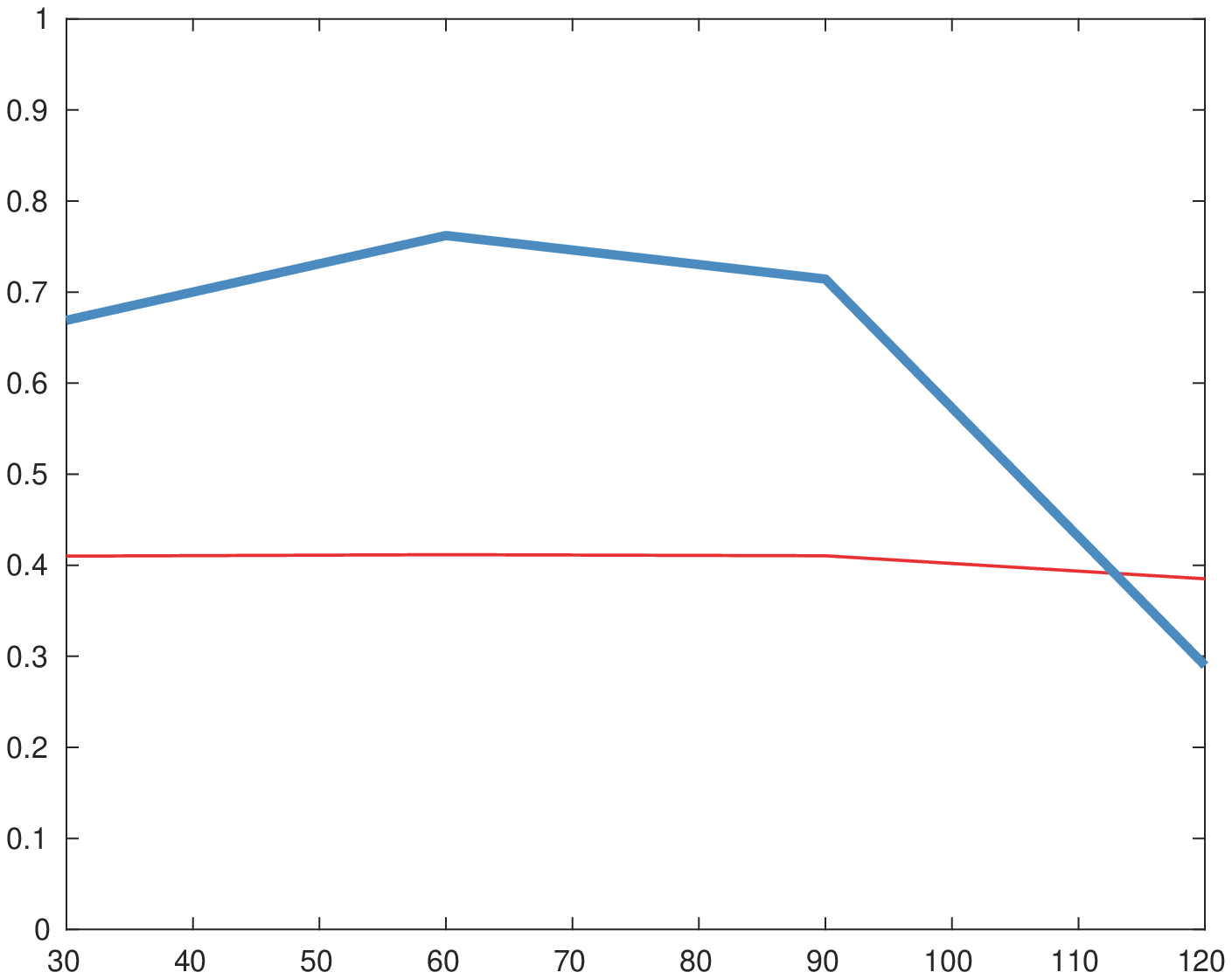}
    \\
    \includegraphics*[width=0.5\linewidth]{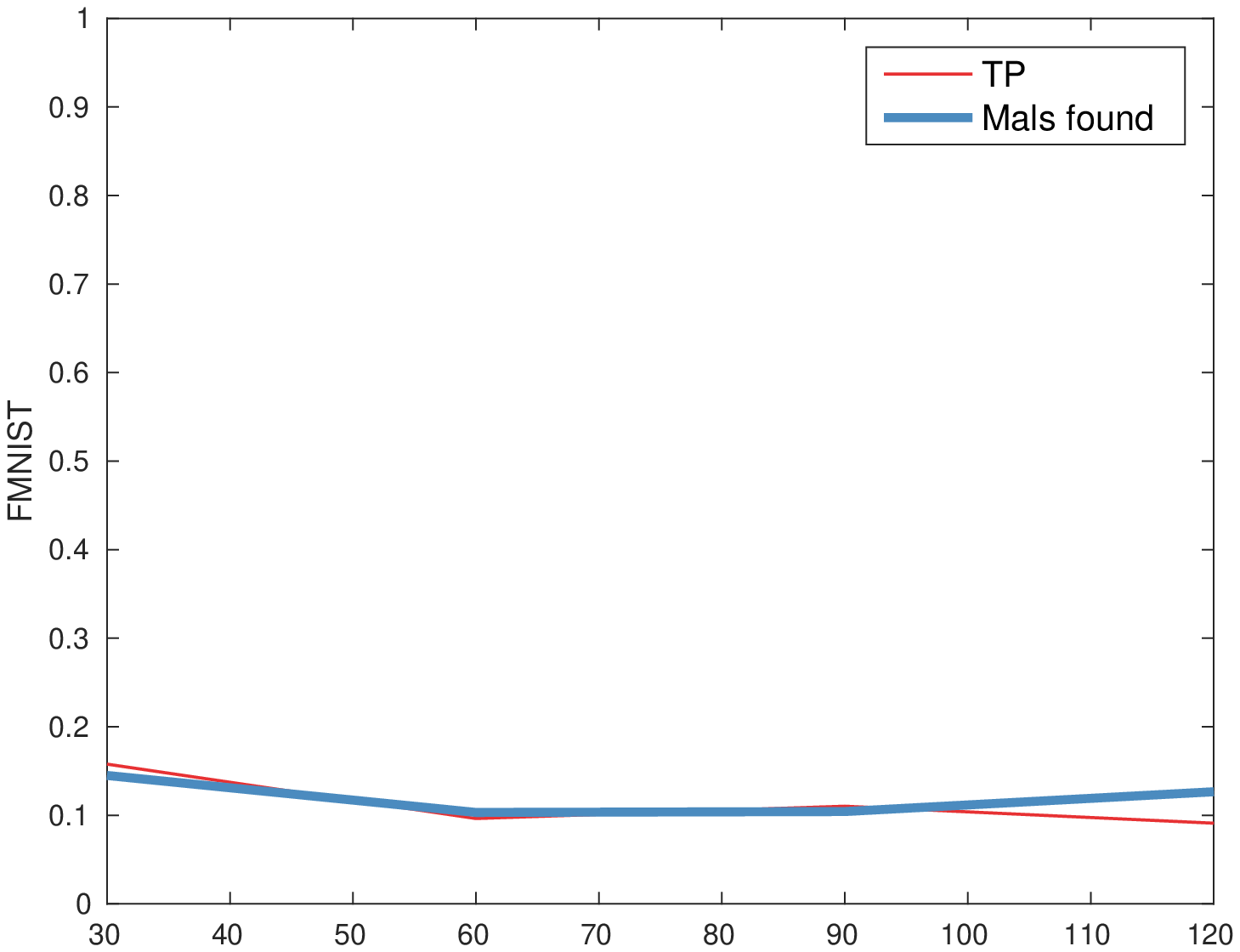}&
    \includegraphics*[width=0.5\linewidth]{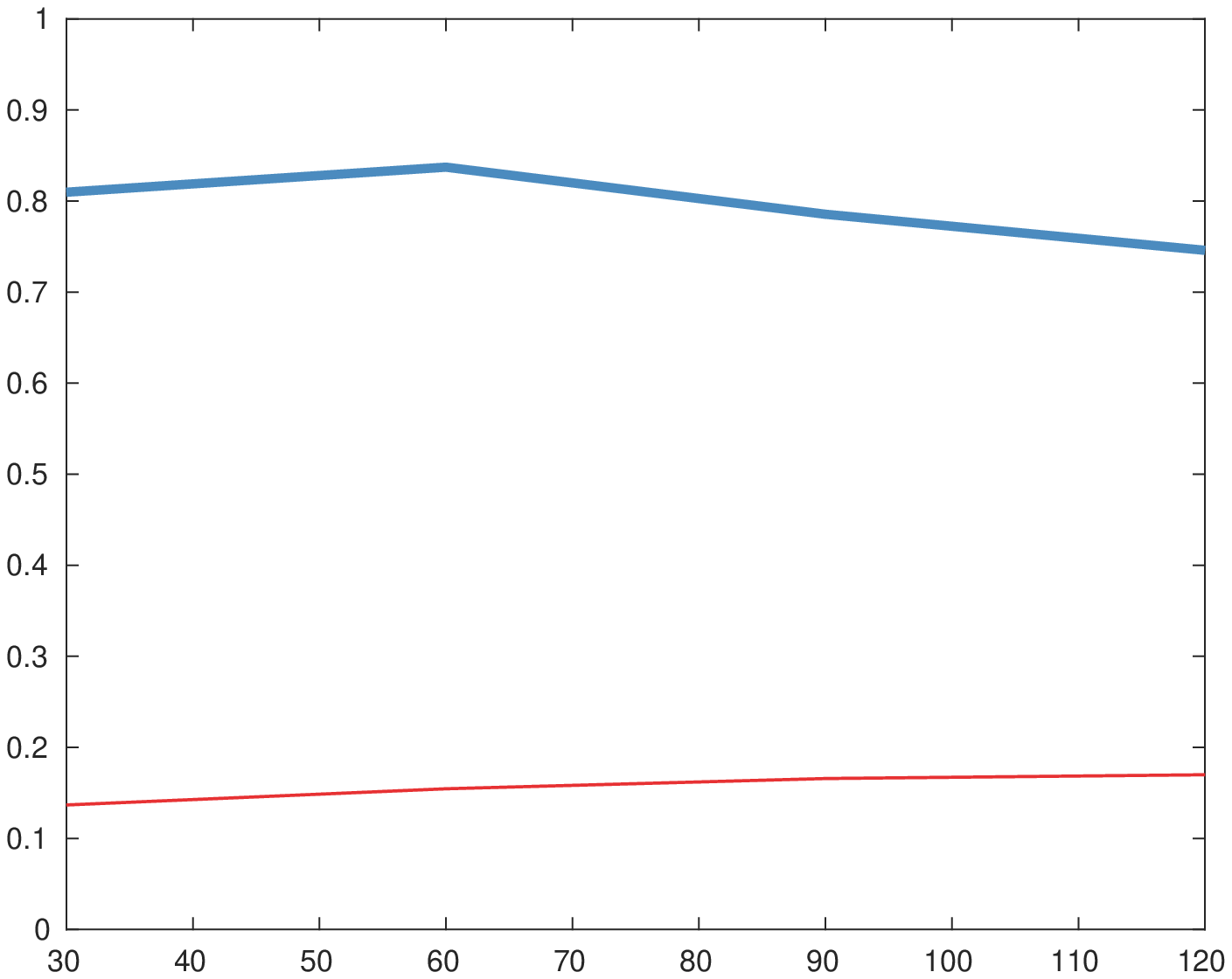}
    \\
    \includegraphics*[width=0.5\linewidth]{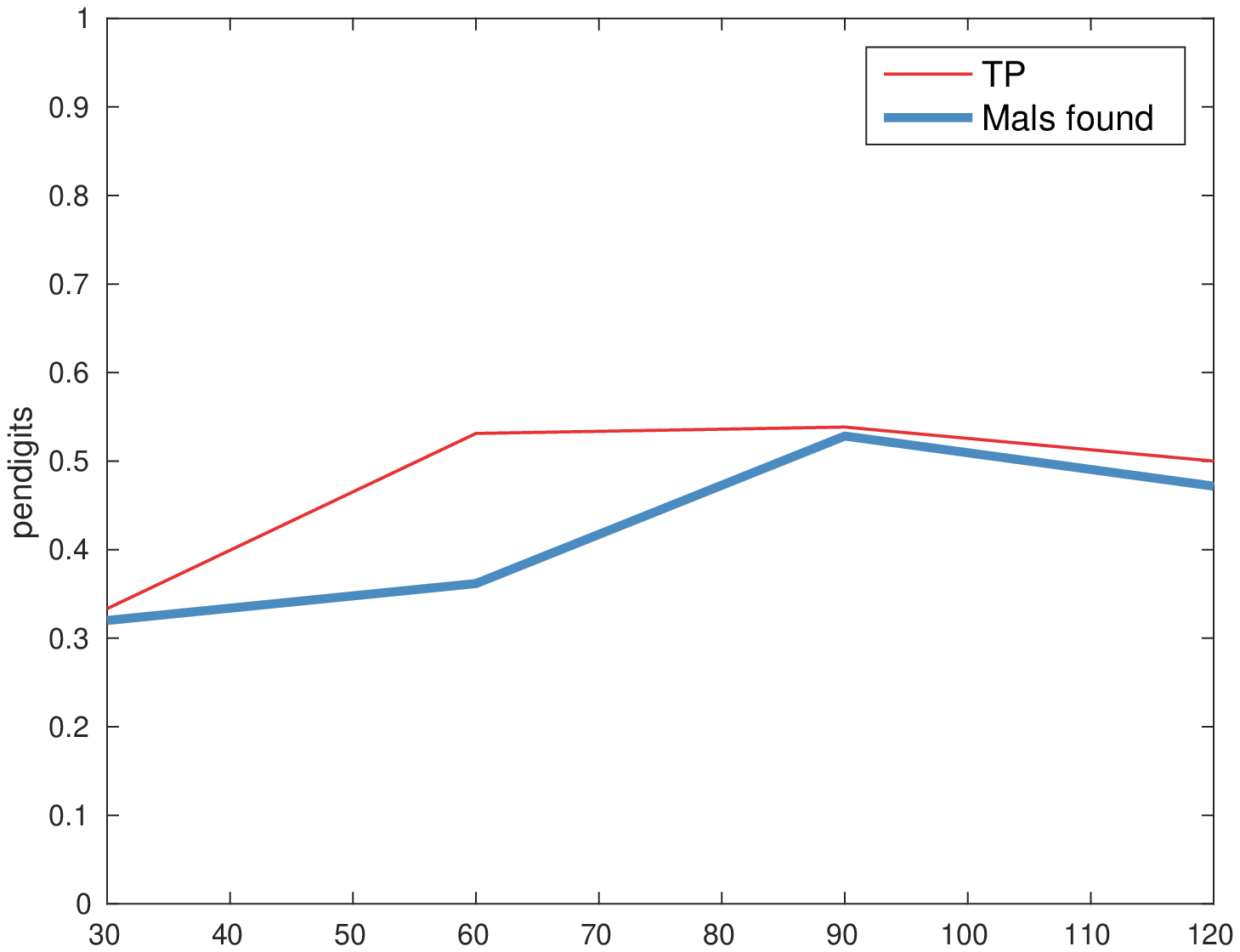}&
    \includegraphics*[width=0.5\linewidth]{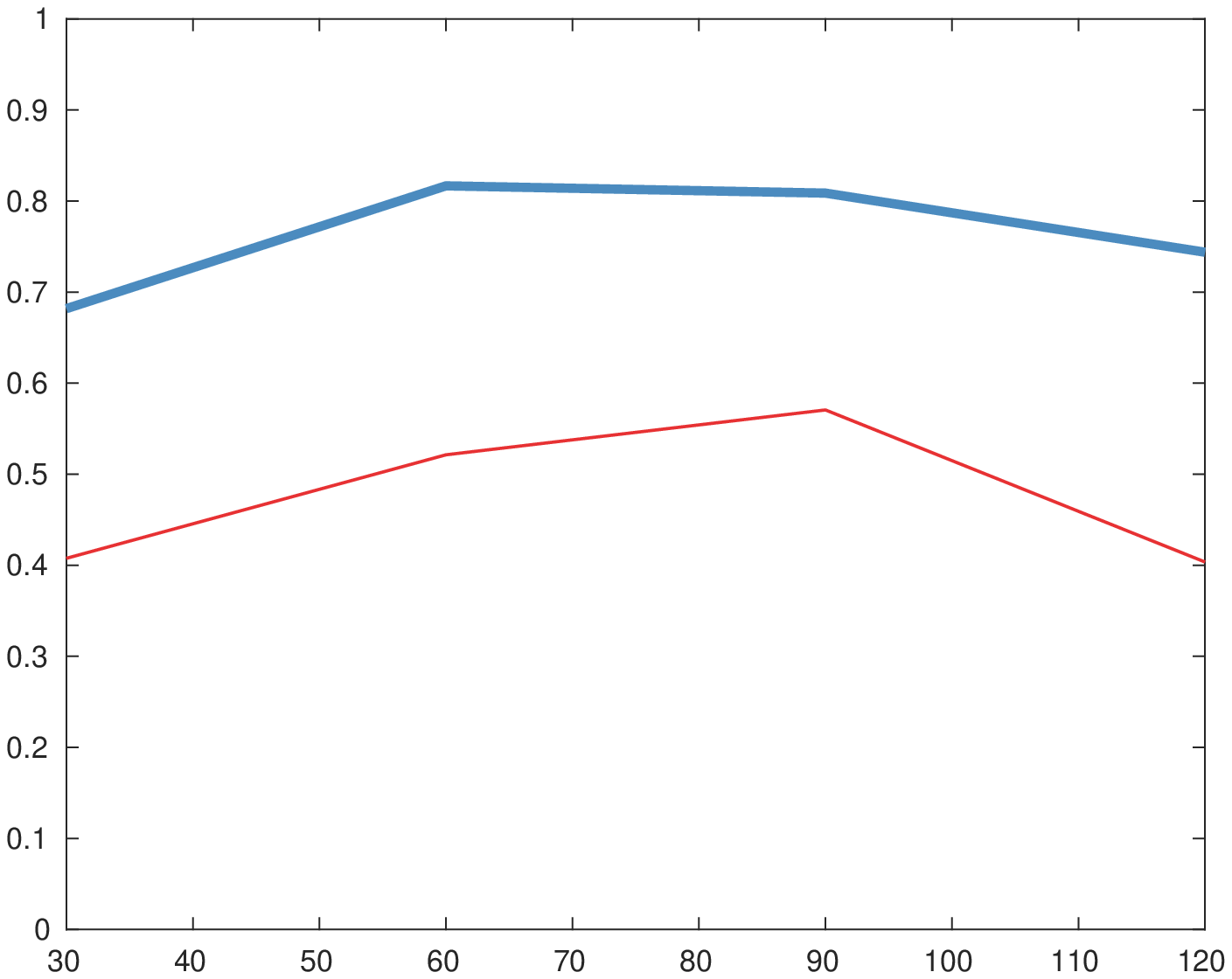}
    \\
    \includegraphics*[width=0.5\linewidth]{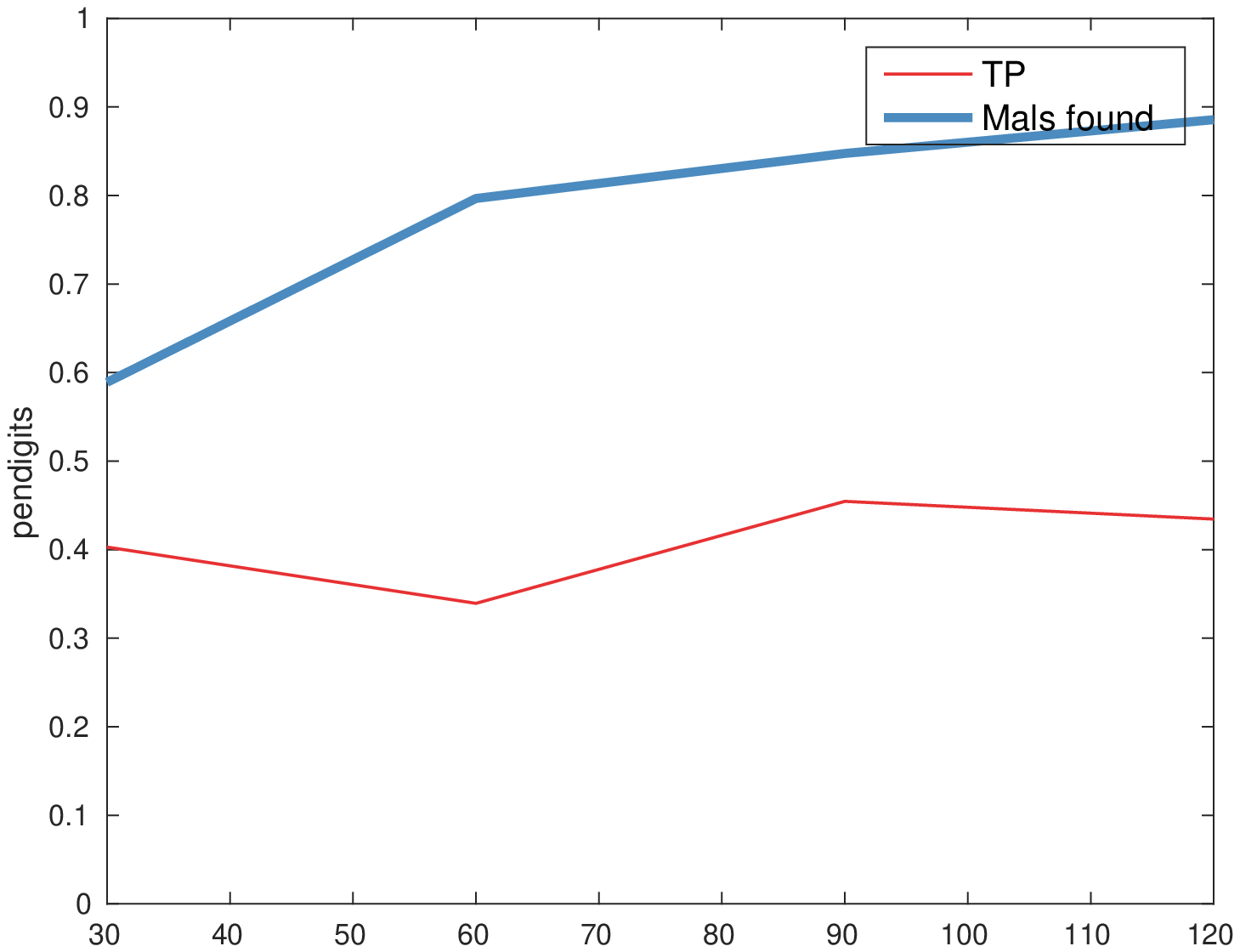}&
    \includegraphics*[width=0.5\linewidth]{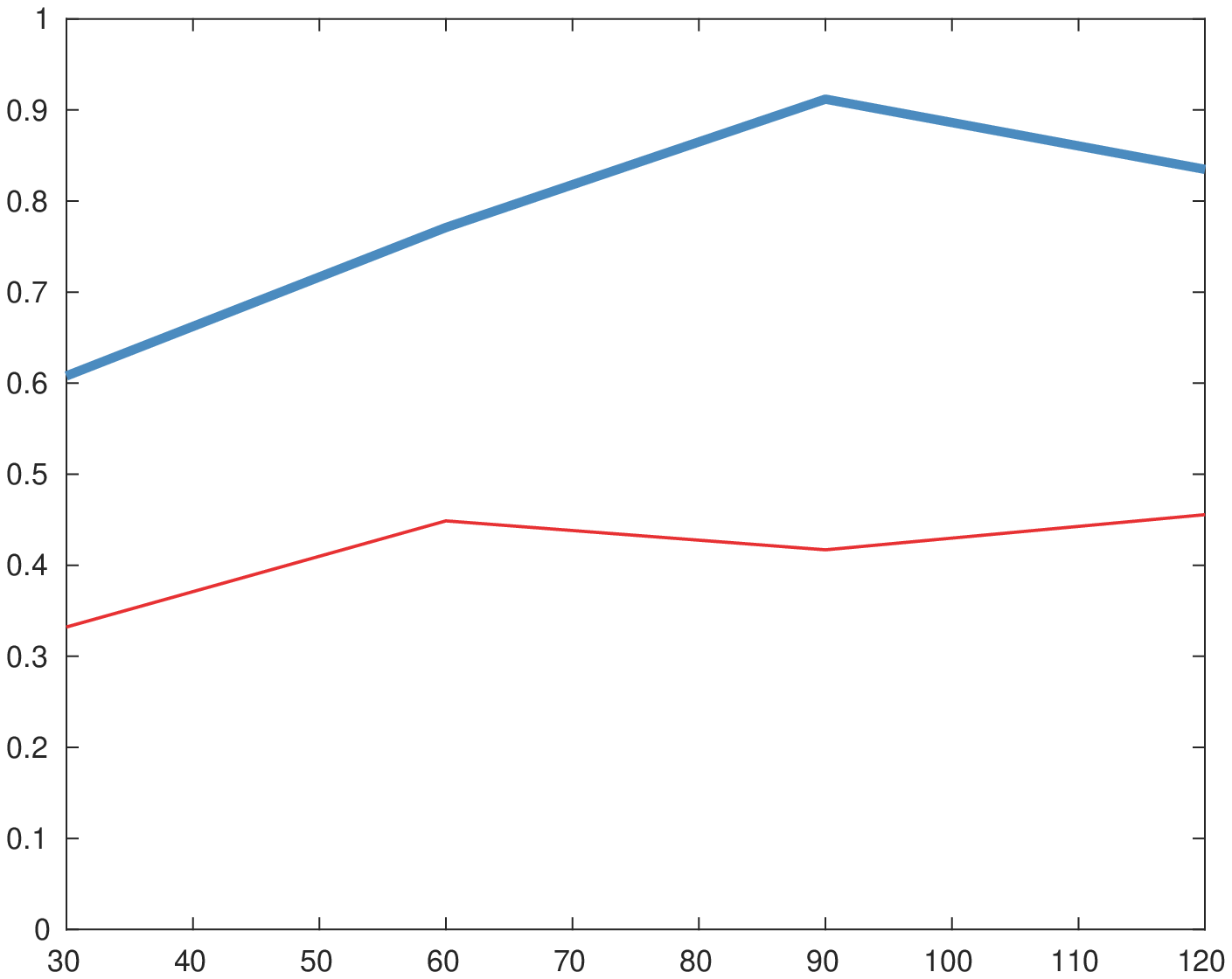}
    \\
  \end{tabular}
  \caption{The vertical axis represents the ratio. The horizontal axis presents the number of base models. The red line shows the TP and blue line shows the ratio of found malicious samples. The curves correspond to the same experiments as in figure 3 of the main paper.}
\label{fig:TP-20-30}
\end{figure} 

\begin{figure}[!ht]
\centering
    \begin{tabular}{c@{}c@{}c@{}}
    ratios with $5\%$ attack & ratios with $10\%$ attack
    \\
    \includegraphics*[width=0.5\linewidth]{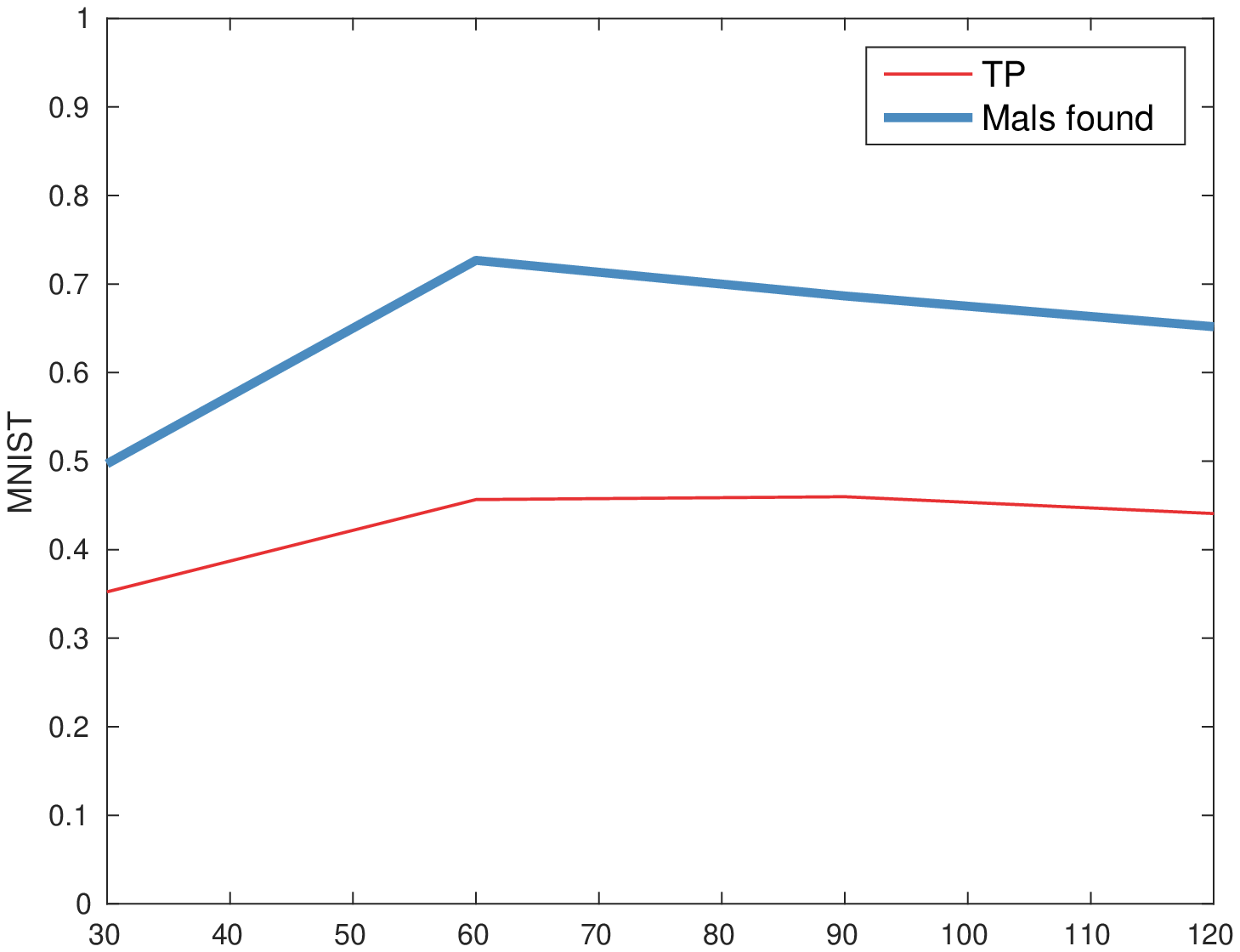}&
    \includegraphics*[width=0.5\linewidth]{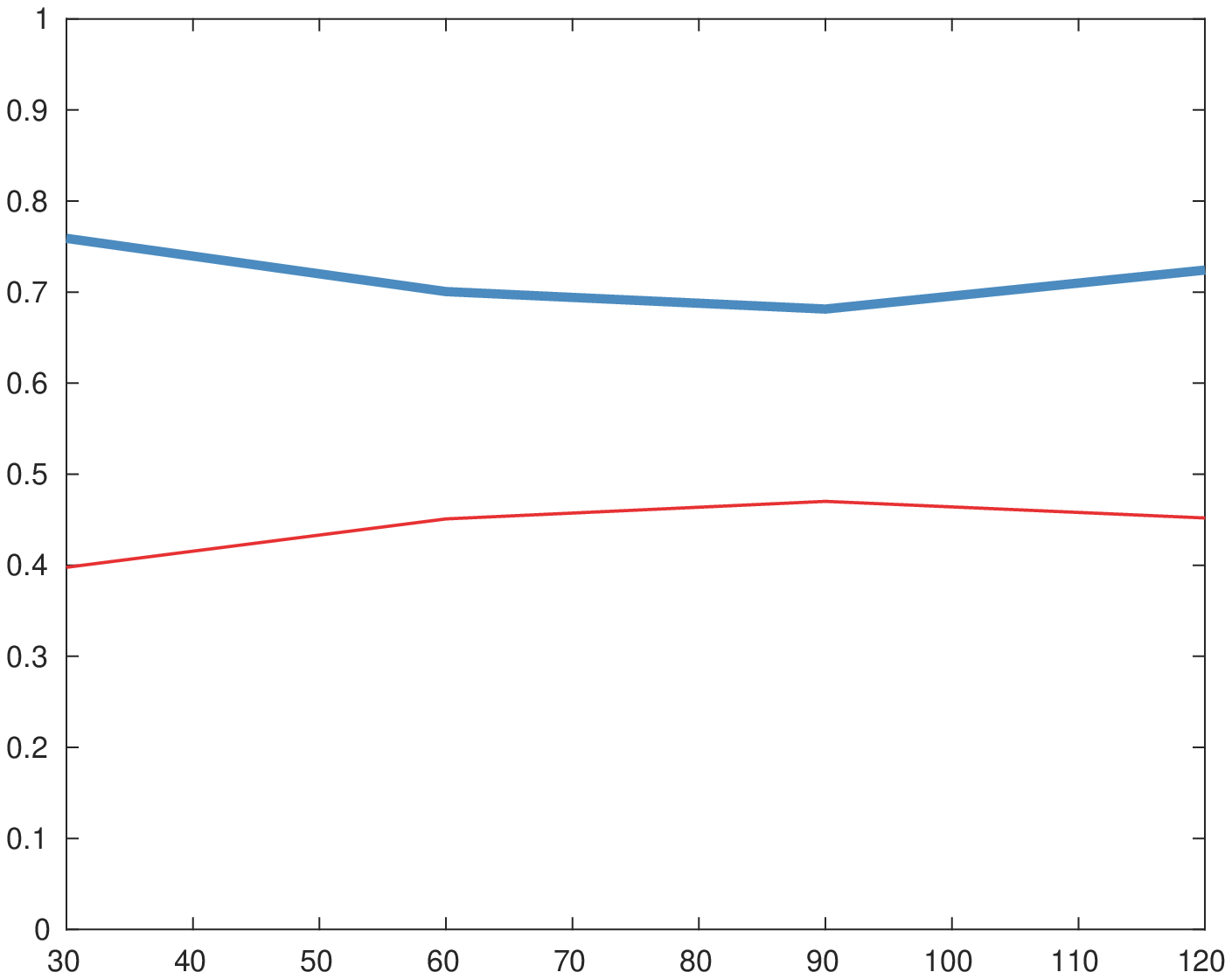}
    \\
    \includegraphics*[width=0.5\linewidth]{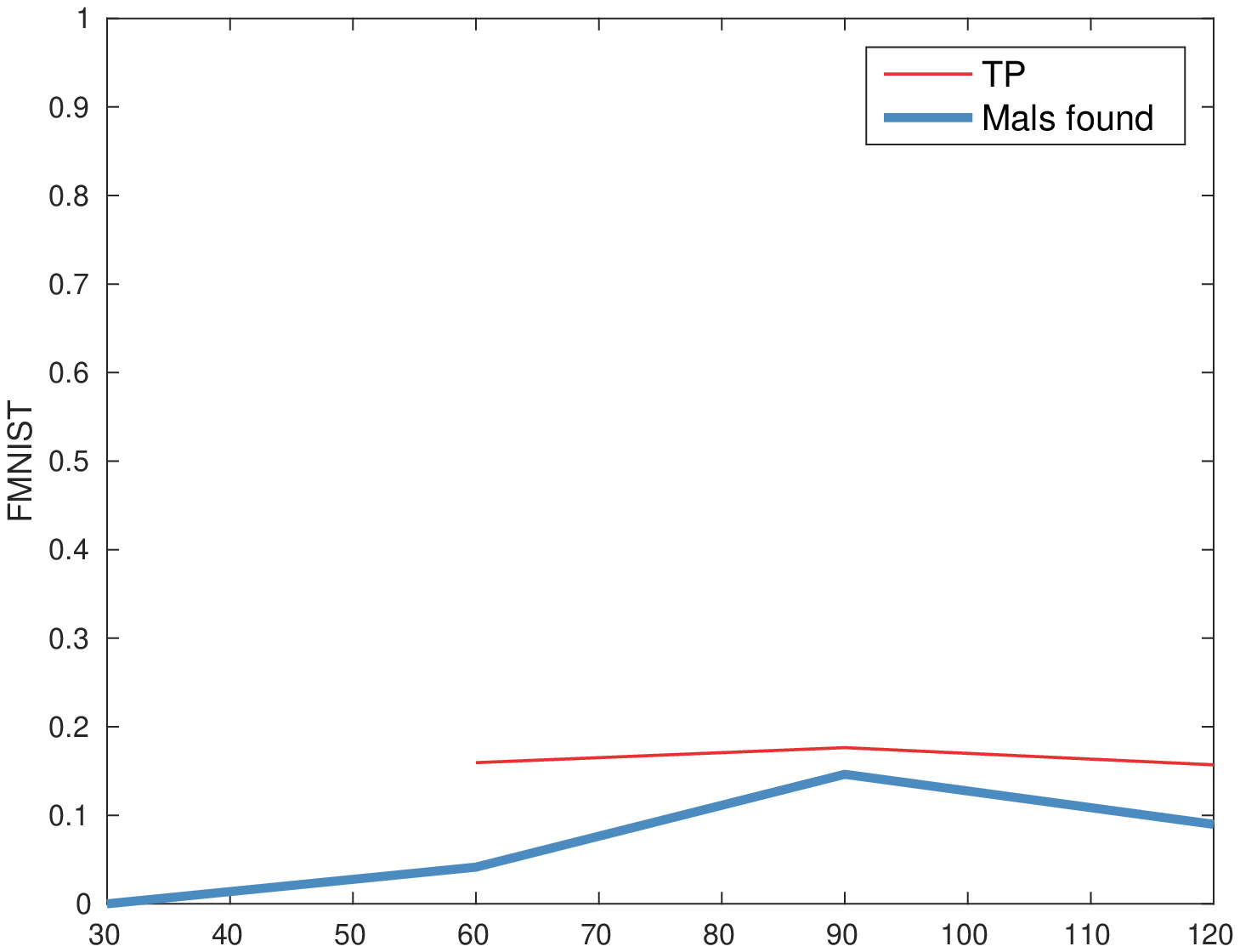}&
    \includegraphics*[width=0.5\linewidth]{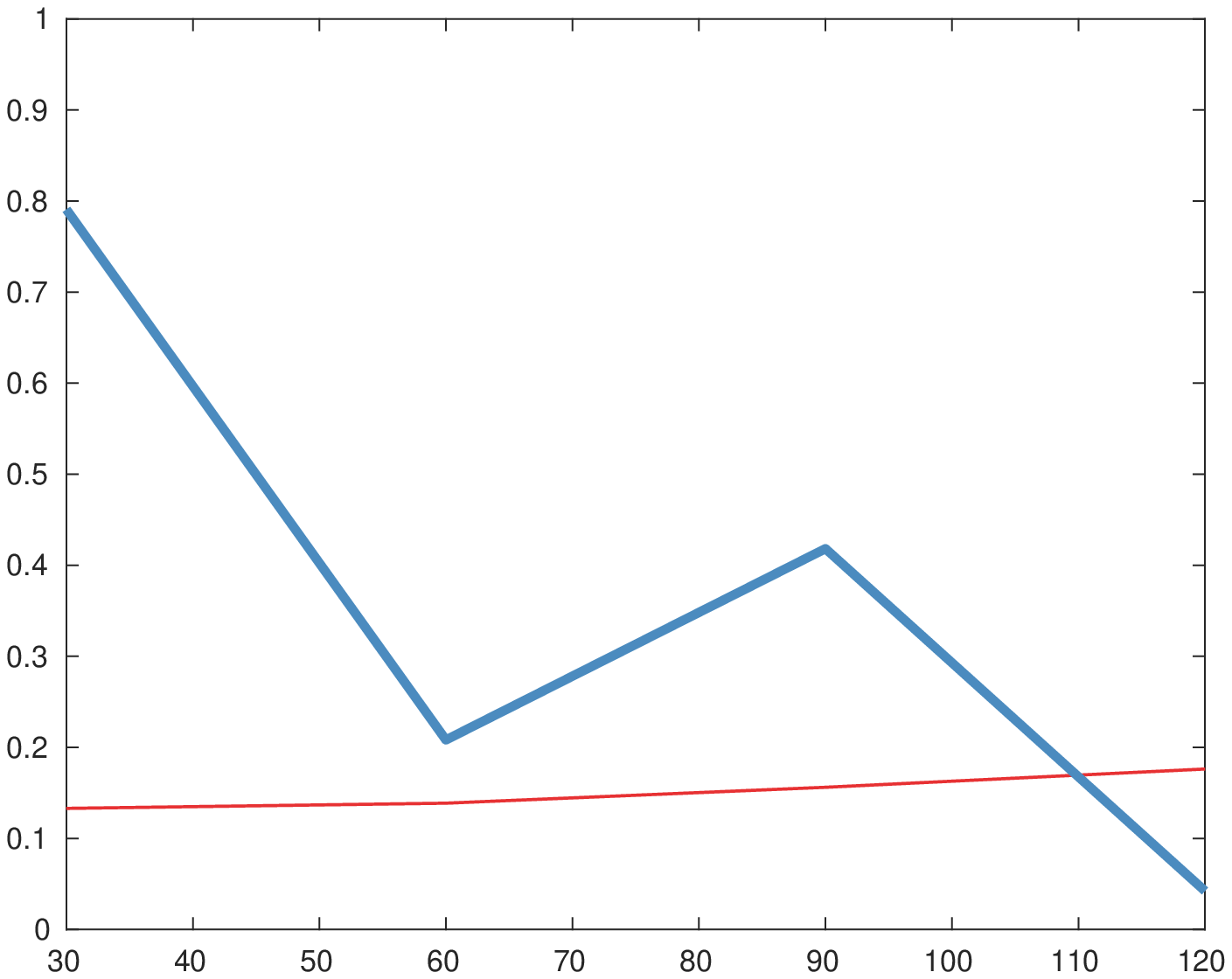}
    \\
    \includegraphics*[width=0.5\linewidth]{grf/pendigits/5_80_30_def_TP.eps}&
    \includegraphics*[width=0.5\linewidth]{grf/pendigits/10_80_30_def_TP.eps}
    \\
  \end{tabular}
  \caption{The vertical axis represents the ratio. The horizontal axis presents the number of base models. The red line shows the TP and blue line shows the ratio of found malicious samples. The curves correspond to the same experiments as in figure \ref{fig:defs-80-30} of the supplementary materials.}
\label{fig:TP-80-30}
\end{figure} 

\begin{figure}[!ht]
\centering
    \begin{tabular}{c@{}c@{}c@{}}
    ratios with $5\%$ attack & ratios with $10\%$ attack
    \\
    \includegraphics*[width=0.5\linewidth]{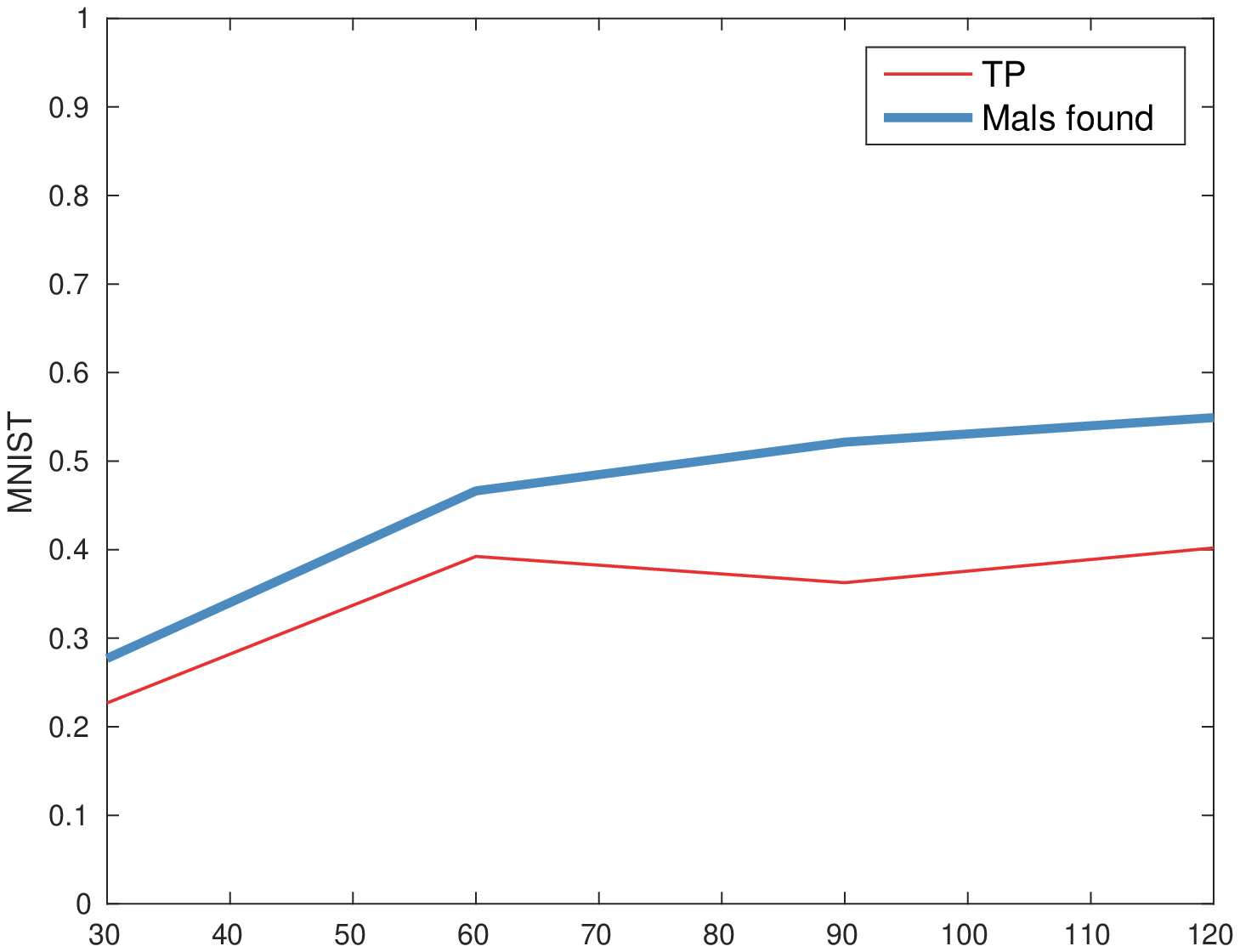}&
    \includegraphics*[width=0.5\linewidth]{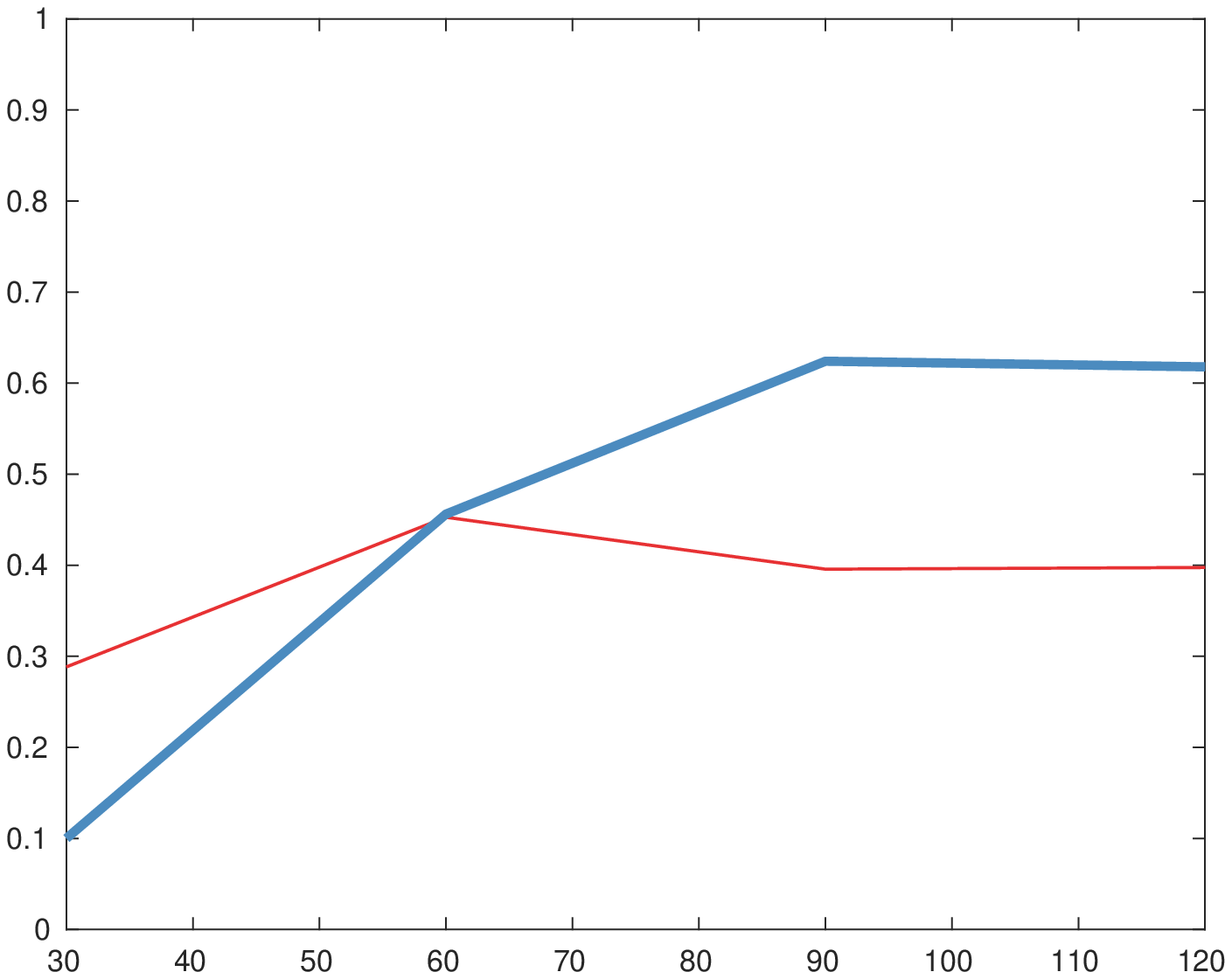}
    \\
    \includegraphics*[width=0.5\linewidth]{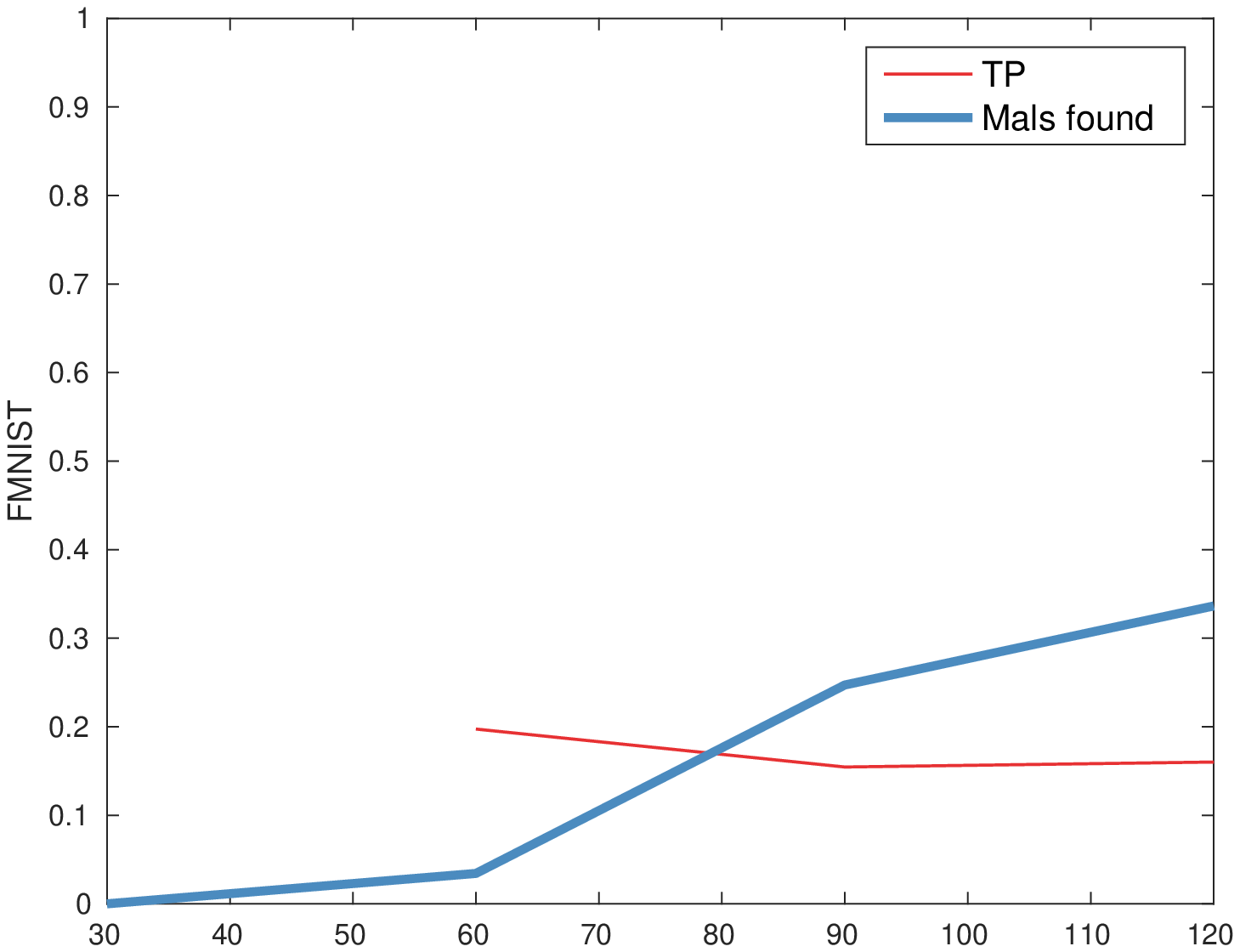}&
    \includegraphics*[width=0.5\linewidth]{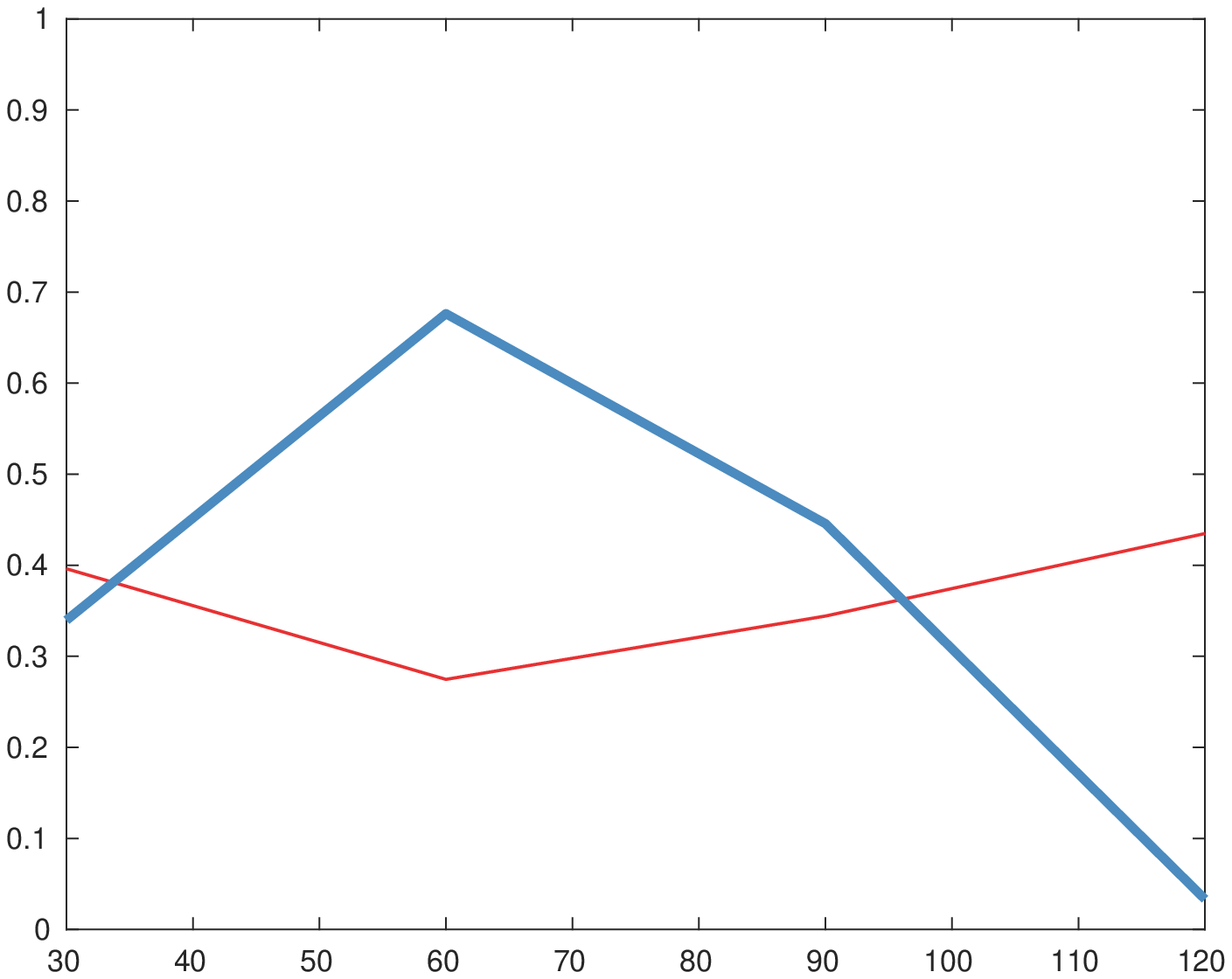}
    \\
    \includegraphics*[width=0.5\linewidth]{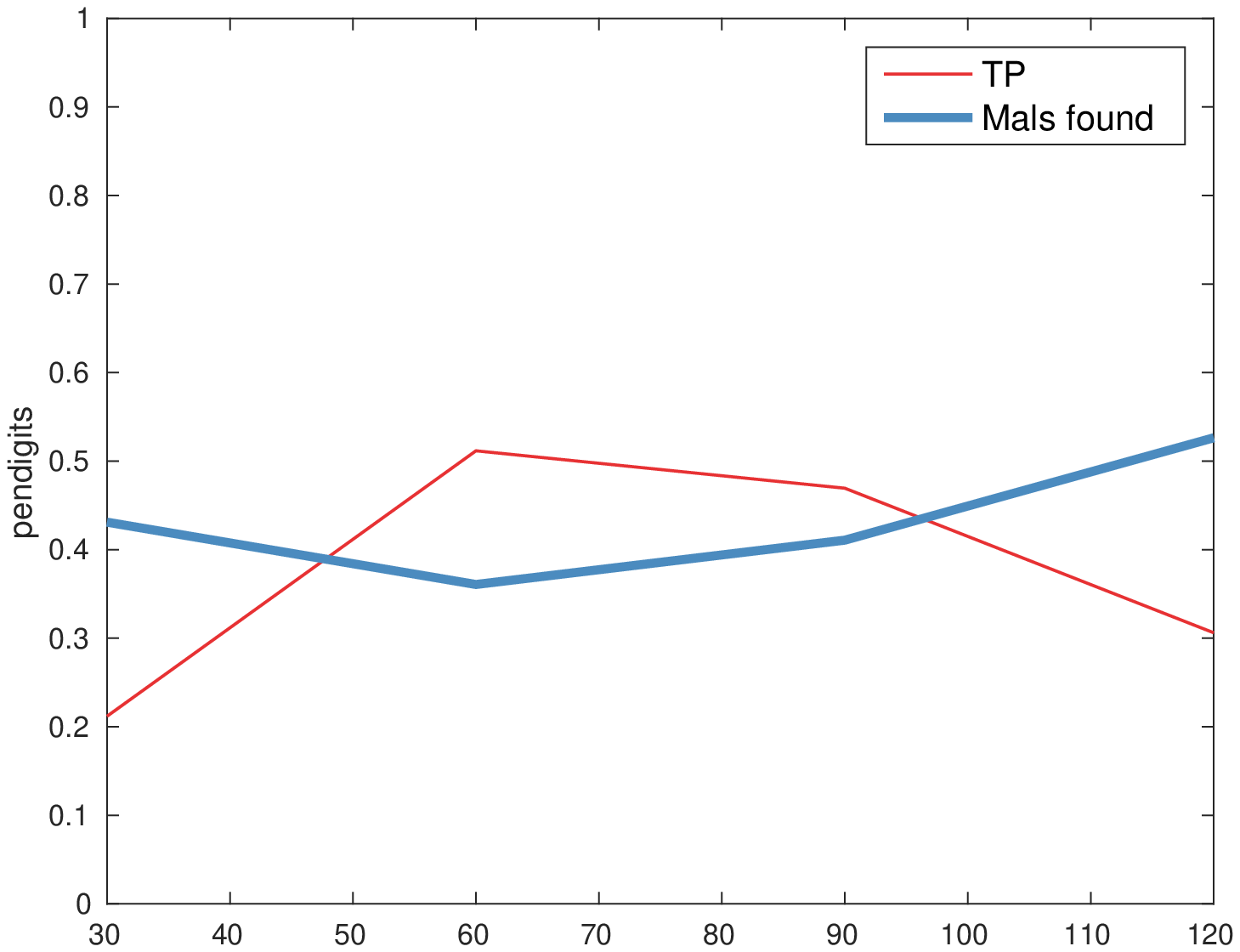}&
    \includegraphics*[width=0.5\linewidth]{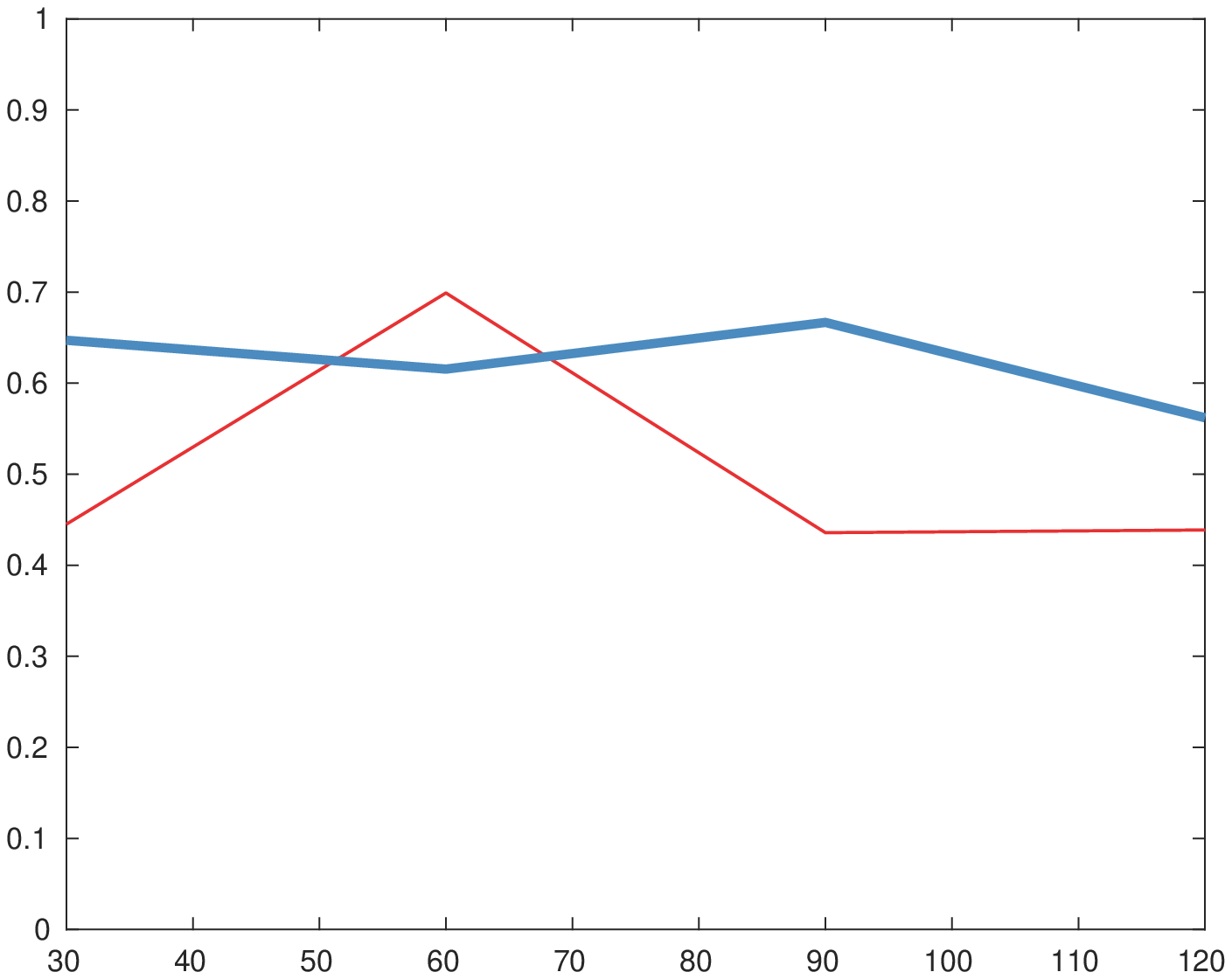}
    \\
    \includegraphics*[width=0.5\linewidth]{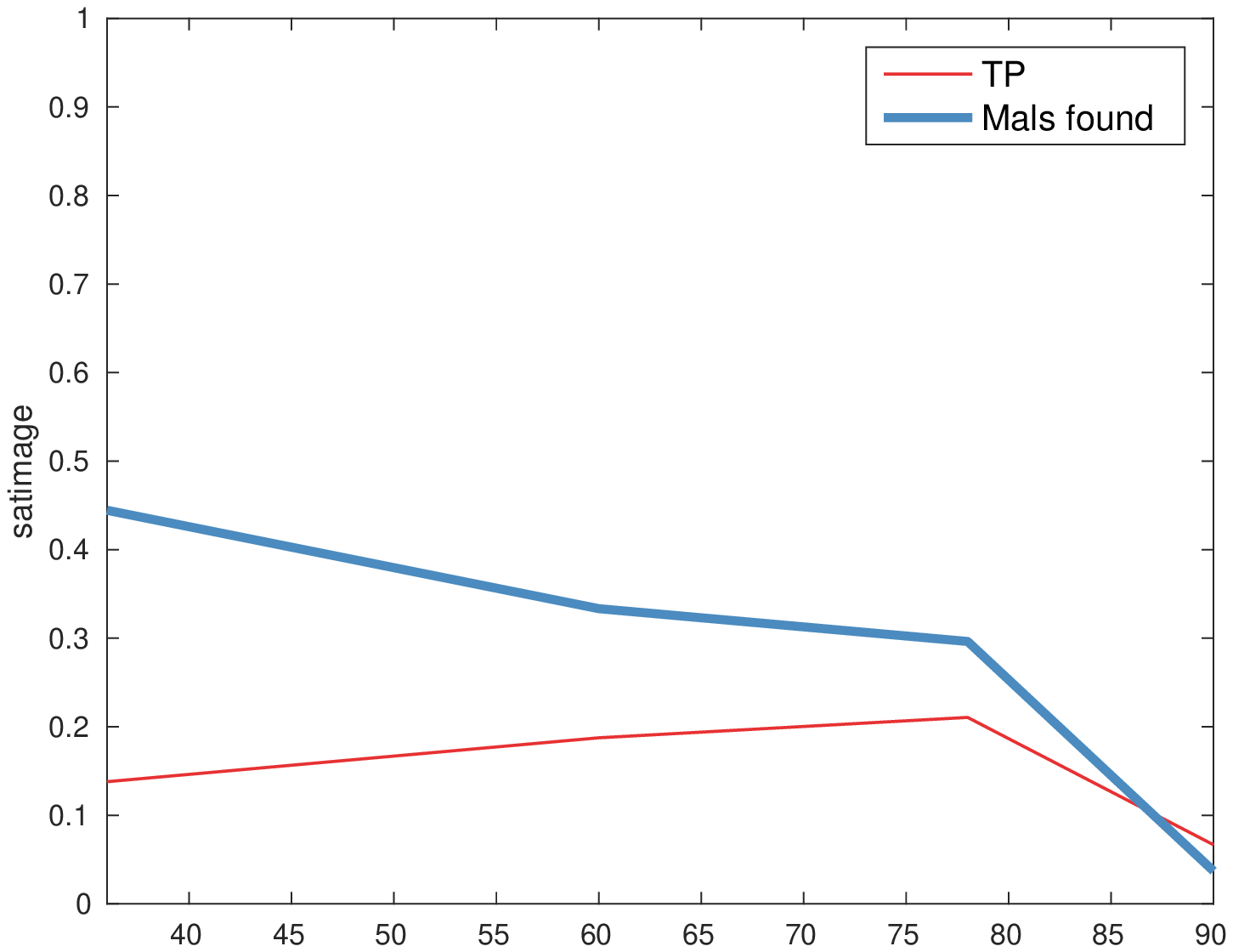}&
    \includegraphics*[width=0.5\linewidth]{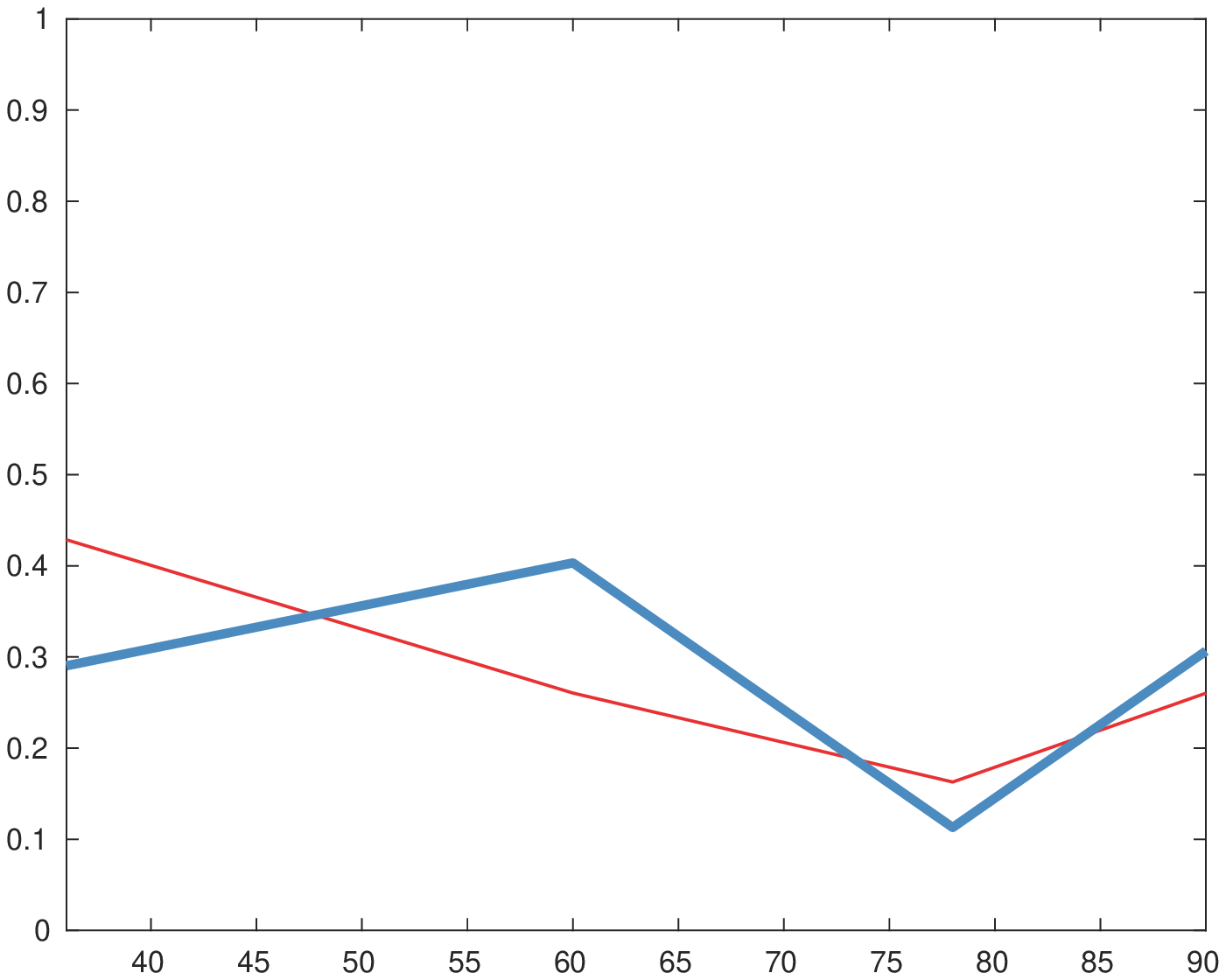}
    \\
  \end{tabular}
  \caption{The vertical axis represents the ratio. The horizontal axis presents the number of base models. The red line shows the TP and blue line shows the ratio of found malicious samples. The curves correspond to the same experiments as in figure \ref{fig:defs-20-60} of the supplementary materials.}
\label{fig:TP-20-60}
\end{figure} 

{\small
\bibliographystyle{ieee_fullname}
\bibliography{refs}
}

\end{document}